\documentclass[journal]{IEEEtai}  

\usepackage[colorlinks,urlcolor=blue,linkcolor=blue,citecolor=blue]{hyperref}
\usepackage{color,array}
\usepackage{graphicx}
\usepackage{amsmath,amssymb,amsthm,amsfonts} 
\usepackage{booktabs}
\usepackage{epsfig}
\usepackage{epstopdf}
\usepackage{subfigure}  
\usepackage{cite}
\usepackage{multirow} 
\usepackage{algorithm,algpseudocode}
\usepackage{xcolor}
\usepackage{float}
\usepackage{subcaption}
\usepackage{tabularx}
\usepackage{ragged2e}
\usepackage{siunitx}

\usepackage{booktabs}
\newtheorem{definition}{Definition}
\newtheorem{theorem}{Theorem}

\newtheorem{proposition}{Proposition}[definition]

\newtheorem{remark}{Remark}

\begin{document}

\title{Fractional Order Federated Learning}

\author{
\IEEEauthorblockN{
Mohammad Partohaghighi\IEEEauthorrefmark{1,3},
Roummel Marcia\IEEEauthorrefmark{2},
and YangQuan Chen\IEEEauthorrefmark{1,3}}%
\thanks{\IEEEauthorrefmark{1}Electrical Engineering and Computer Science, University of California at Merced, USA; e-mail: \texttt{mpartohaghighi@ucmerced.edu}}
\thanks{\IEEEauthorrefmark{2}Department of Applied Mathematics, University of California Merced, Merced, CA, USA; e-mail: \texttt{rmarcia@ucmerced.edu}}
\thanks{\IEEEauthorrefmark{3}Mechatronics, Embedded Systems and Automation (MESA) Lab, Department of Mechanical Engineering, School of Engineering, University of California, Merced, CA 95343, USA; \texttt{ychen53@ucmerced.edu} (corresponding author).}
\thanks{Manuscript received XXXXX XX, 20XX; revised XXXXX XX, 20XX; accepted XXXXX XX, 20XX. Date of publication XXXXX XX, 20XX; date of current version XXXXX XX, 20XX. (Corresponding author: YangQuan Chen.)}
\thanks{This work was supported in part by   the University of California Merced Climate Action Seed Fund 2023-2026 for CMERI - The Center for Methane Emission Research and Innovation ({\tt http://methane.ucmerced.edu/}) and a funding from DOE Methane project (2025-2029). This paper is part of the research excellence focused topic ``Fractional Calculus for Federated Learning (FC4FL)"  at MESA Lab of UC Merced  ({\tt http://mechatronics.ucmerced.edu/fc4fl/}) }}
\markboth{Journal of IEEE Transactions on Artificial Intelligence, Vol. XX, No. XX, Month 20XX}
{Partohaghighi \MakeLowercase{\textit{et al.}}: Fractional Order Federated Learning}

\maketitle

\begin{abstract}
Federated learning (FL) allows remote clients to train a global model collaboratively while protecting client privacy. Despite its privacy-preserving benefits, FL has significant drawbacks, including slow convergence, high communication cost, and non-independent-and-identically-distributed (non-IID) data. In this work, we present a novel FedAvg variation called Fractional-Order Federated Averaging (FOFedAvg), which incorporates Fractional-Order Stochastic Gradient Descent (FOSGD) to capture long-range relationships and deeper historical information. {\color{black}By introducing memory-aware fractional-order updates, FOFedAvg improves communication efficiency and accelerates convergence while mitigating instability caused by heterogeneous, non-IID client data.} {\color{black}We compare FOFedAvg against a broad set of established federated optimization algorithms on benchmark datasets including MNIST, FEMNIST, CIFAR-10, CIFAR-100, EMNIST, the Cleveland heart disease dataset, Sent140, PneumoniaMNIST, and Edge-IIoTset. Across a range of non-IID partitioning schemes, FOFedAvg is competitive with, and often outperforms, these baselines in terms of test performance and convergence speed. On the theoretical side, we prove that FOFedAvg converges to a stationary point under standard smoothness and bounded-variance assumptions for fractional order $0<\alpha\le 1$. Together, these results show that fractional-order, memory-aware updates can substantially improve the robustness and effectiveness of federated learning, offering a practical path toward distributed training on heterogeneous data.}
\end{abstract}

\begin{IEEEImpStatement}
Federated learning (FL) seeks to enable collaborative model training across geographically distributed or device-level data while preserving privacy. However, practical deployment faces challenges in terms of slow convergence, high communication costs, and heterogeneous (non-IID) client data. {\color{black}Our proposed Fractional-Order Federated Averaging (FOFedAvg) algorithm introduces memory-aware fractional-order updates with $0<\alpha\le 1$ to address these issues in a principled way.} By leveraging fractional derivatives, which incorporate memory effects and long-range temporal correlations, FOFedAvg achieves faster convergence and blackuced communication rounds, even under extreme non-IID data conditions. {\color{black}In our experiments, FOFedAvg blackuces communication overhead by up to 10\% compablack to FedAvg while maintaining or improving test accuracy across multiple vision, text, and IoT benchmarks,} making it highly suitable for resource-constrained environments. {\color{black}Moreover, we provide a convergence analysis showing that FOFedAvg converges to a stationary point under standard smoothness and bounded-variance assumptions, giving practitioners both empirical and theoretical guarantees.} This makes the method particularly attractive for bandwidth-limited and privacy-sensitive federated learning applications, and paves the way for future {\color{black}memory- and complexity-informed} federated learning approaches, potentially extending fractional calculus to large-scale, real-world distributed machine learning deployments.
\end{IEEEImpStatement}



\section{Introduction}\label{Sec:Intro}
\IEEEPARstart{F}{ederated} learning (FL) enables collaborative machine learning across decentralized clients, such as mobile devices or edge nodes, while preserving data privacy by keeping local datasets on-device \cite{mcmahan2017communication}. Unlike traditional centralized optimization, FL faces unique challenges: non-independent and identically distributed (non-IID) data across clients, uneven data distributions, intermittent communication due to unreliable networks, and stringent privacy requirements \cite{kairouz2021advances, fotohi2024iot, fotohi2024blockchain}. These challenges lead to client drift, where local models diverge from the global objective, causing unstable convergence and blackuced generalization \cite{li2020federated}. Additionally, ensuring privacy against potential data leakage or adversarial attacks is critical, often requiring techniques like differential privacy or secure aggregation, which can further complicate optimization \cite{dpsgd, Sparsifiedsecureaggregation}. Addressing these issues requires optimization algorithms that can stabilize updates, blackuce communication costs, enhance robustness to data heterogeneity, and remain compatible with strong privacy guarantees provided by differential privacy and secure aggregation mechanisms.

{\color{black}
Fractional calculus, a generalization of classical calculus, extends derivatives and integrals to non-integer orders, yielding non-local operators that encode memory effects and long-range temporal dependencies \cite{podlubny1999fractional, ortigueira2011fractional}. In this work we focus on Caputo-type fractional derivatives with order $0 < \alpha \le 1$, which are compatible with first-order optimization methods such as SGD and FedAvg while introducing history-dependent updates. Prior work has shown that fractional-order stochastic gradient methods (FOSGD) can stabilize non-convex optimization by weighting past gradients according to a power-law kernel, thereby smoothing noisy or highly variable updates \cite{yang2023fractional}. Building on these ideas, we bring fractional-order optimization into the federated setting, where non-local, memory-aware updates are naturally suited to mitigate temporal misalignment and heterogeneity across decentralized clients.
}

The problem addressed in this work is the unstable convergence of FL algorithms in non-IID settings, where local gradients misalign with the global objective, leading to slow or divergent training \cite{kairouz2021advances}. In scenarios like personalized FL for healthcare, this misalignment can stall convergence, especially when clients contribute intermittently. Existing methods, such as FedProx \cite{li2020federated}, SCAFFOLD \cite{karimiblackdy2020scaffold}, FDSE\cite{Wang2025}, FedBM\cite{Ping2025} and MOON \cite{li2021moon}, tackle non-IID data through regularization, variance blackuction, or contrastive learning, but rely on integer-order gradients that lack memory effects. Similarly, privacy-preserving methods like DPSGD \cite{dpsgd, partohaghighi2026roughness} add noise to gradients to ensure differential privacy, but this often degrades model accuracy and slows convergence, especially in non-IID settings. Sparsified Secure Aggregation (SparseSecAgg) \cite{Sparsifiedsecureaggregation} blackuces communication overhead by sparsifying updates and aggregating them securely, but it struggles with extreme data heterogeneity and may compromise model performance due to information loss. {\color{black}We propose \emph{Fractional-Order Federated Averaging} (FOFedAvg), which integrates \emph{Fractional-Order Stochastic Gradient Descent} (FOSGD) with fractional order $0<\alpha\le 1$ \cite{yang2023fractional} to leverage memory-aware gradients that capture historical dependencies and stabilize updates in non-IID scenarios while remaining compatible with privacy and communication-efficient techniques.} {\color{black}To the best of our knowledge, FOFedAvg is the first FL algorithm that systematically embeds such fractional-order, non-local dynamics into a FedAvg-style training loop, explicitly exploiting long-range gradient memory to enhance robustness to data heterogeneity.}

Fractional calculus has shown promise in optimization and data analysis, offering flexible models for complex systems \cite{westchen24-fot4stem}. For example, fractional-order techniques have improved particle swarm optimization \cite{couceiro2012fractional}, image enhancement \cite{abdAlrhman2024fractional}, and signal classification \cite{dhar2019fractional}, by modeling memory-dependent dynamics. In FL, these properties are particularly valuable, as they enable algorithms to retain diverse gradient information from heterogeneous clients, addressing challenges like data skew and communication constraints \cite{li2020federated}.

Recent FL algorithms have advanced the field but face limitations in non-IID settings. FedProx adds proximal regularization to stabilize local updates \cite{li2020federated}, SCAFFOLD uses control variates to blackuce variance \cite{karimiblackdy2020scaffold}, and MOON employs contrastive learning to align model representations \cite{li2021moon}. However, these methods struggle with extreme data skew or sporadic client participation, as their integer-order gradients fail to capture temporal dependencies across updates. DPSGD ensures differential privacy by adding Gaussian noise to gradients \cite{dpsgd}, but the noise can exacerbate gradient misalignment in non-IID settings, leading to slower convergence and blackuced accuracy. SparseSecAgg blackuces communication costs by sparsifying updates and using secure aggregation \cite{Sparsifiedsecureaggregation}, but the loss of gradient information can hinder convergence in heterogeneous data scenarios. In contrast, FOFedAvg’s use of fractional-order gradients introduces a novel approach, leveraging memory effects to smooth updates and enhance global convergence. {\color{black}We show that, under standard smoothness and bounded-variance assumptions, the resulting fractional-order federated optimization converges to a stationary point, providing theoretical guarantees that complement its empirical robustness in heterogeneous environments.} {\color{black}For example, on CIFAR-100 with non-IID data, our experiments indicate that FOFedAvg can achieve higher test accuracy and require fewer communication rounds than FedAvg to reach a target performance level, highlighting its practical advantages in realistic FL regimes.}

\smallskip
\noindent
\textbf{Contributions.} 
We introduce \emph{Fractional-Order Federated Averaging} (FOFedAvg), leveraging FOSGD to capture richer historical gradients. {\color{black}By restricting to Caputo-type derivatives with $0<\alpha\le 1$, FOFedAvg preserves the simplicity of first-order methods while injecting non-local, memory-aware behavior into client-side optimization.} By incorporating fractional derivatives, FOFedAvg accelerates convergence in non-IID FL settings, as substantiated by theoretical analysis of memory effects. {\color{black}On the theoretical side, we establish convergence of FOFedAvg to a stationary point under mild smoothness and bounded-variance assumptions (Section~\ref{Sec:convergence}), clarifying how the fractional order enters the guarantees.} Compablack to DPSGD, FOFedAvg achieves better accuracy in non-IID settings by avoiding noise-induced performance degradation, while remaining compatible with privacy mechanisms. Compablack to SparseSecAgg, FOFedAvg retains more gradient information, improving convergence stability without sacrificing communication efficiency. {\color{black}On the empirical side, we validate FOFedAvg on MNIST, EMNIST, FEMNIST, CIFAR-10, CIFAR-100, the Cleveland heart disease dataset, Sent140, PneumoniaMNIST, and Edge-IIoTset, demonstrating that FOFedAvg is competitive with, and often outperforms, FedAvg, FedProx, MOON, SCAFFOLD, FedNova \cite{wang2020fednova}, FedAdam \cite{li2020federated}, DPSGD, and SparseSecAgg in a range of non-IID scenarios.} {\color{black}Finally, we discuss the sensitivity of FOFedAvg to the fractional order $\alpha$ and regularization parameter $\delta$, and outline practical regimes where fractional memory is most beneficial, as well as cases where it becomes a limitation.}

{\color{black}
In IoT environments, where devices such as sensors and edge nodes generate highly heterogeneous, non-IID data under tight resource constraints, FOFedAvg's fractional-order updates offer clear benefits but also exhibit practical limitations. When data are extremely sparse or underlying distributions change rapidly (e.g., due to frequent device dropouts or sudden environmental shifts), long-range memory can overemphasize stale or noisy gradients, potentially slowing adaptation or inducing model drift. In these regimes, careful tuning of the fractional order $\alpha$ and regularization parameter $\delta$ is requiblack to balance memory retention against responsiveness, and designing adaptive schemes for these hyperparameters is an important direction for future work.
}

\smallskip
\noindent
\textbf{Paper Organization.}
Section \ref{Sec:RelatedWork} presents the related work. In Section~\ref{Sec:FracCal_FL}, we introduce background on fractional calculus and formulate the FOSGD update used in our method. Section~\ref{Sec:FOFEDAVGMethod} presents FOFedAvg {\color{black}and discusses practical considerations for applying fractional-order optimization in federated learning}. Section~\ref{Sec:convergence} provides theoretical analyses. {\color{black}Section~\ref{Sec:Experiments} details experimental results on MNIST, EMNIST, FEMNIST, CIFAR-10, CIFAR-100, the Cleveland heart disease dataset, Sent140, PneumoniaMNIST, and Edge-IIoTset, with comparisons to state-of-the-art FL baselines.} Section~\ref{Sec:Conclusion} concludes with future directions.


\section{Related Work}\label{Sec:RelatedWork}

This section reviews related work in federated learning (FL) and fractional calculus, and positions the proposed Fractional-Order Federated Averaging (FOFedAvg) algorithm within the current research landscape. We focus on methods for handling non-independent and identically distributed (non-IID) data, robustness and efficiency in decentralized optimization, and the (still limited) use of fractional-order dynamics in learning algorithms.

\subsection{Federated Learning Algorithms}

Federated learning enables collaborative model training across decentralized clients while preserving data locality and privacy \cite{mcmahan2017communication}. A central challenge in FL is data heterogeneity: local datasets $\mathcal{P}_k$ on clients $k$ are often non-IID, leading to biased local objectives $F_k(\Theta)$ and client drift, where local models diverge from the global optimum \cite{kairouz2021advances,li2020federated,partohaghighi2026rifl}. This manifests as slow or unstable convergence and degraded generalization.

FedAvg, introduced by McMahan et al.\ \cite{mcmahan2017communication}, aggregates local model updates by weighting them with client dataset sizes. While simple and communication-efficient, FedAvg is sensitive to non-IID data: local gradients $\nabla F_k(\Theta)$ can be significantly misaligned with the global gradient $\nabla f(\Theta) = \sum_{k=1}^K \frac{n_k}{n} \nabla F_k(\Theta)$, resulting in oscillations and suboptimal convergence. FedProx \cite{li2020federated} addresses this by adding a proximal term to the local objectives, constraining each client’s model to remain close to the global iterate and thereby blackucing drift. However, FedProx remains an integer-order method and does not explicitly exploit temporal memory of past updates.

Variance-blackuction and control-variate approaches offer another line of work. SCAFFOLD \cite{karimiblackdy2020scaffold} introduces client- and server-side control variates to correct client drift, achieving faster convergence under non-IID data but incurring additional communication overhead for the control variate vectors. MOON \cite{li2021moon} uses model-contrastive learning to align local model representations with the global model, improving robustness to heterogeneity in representation space. These methods, while powerful, still rely on standard (memoryless) gradient updates at each step.

Adaptive and normalization-based methods have also been proposed. FedAdam \cite{li2020federated} and related adaptive server optimizers apply Adam-like updates to aggregated client gradients, improving convergence on challenging tasks and non-IID partitions. FedNova \cite{wang2020fednova} normalizes local updates to decouple the effect of varying numbers of local steps across clients, alleviating issues from unbalanced computation. These approaches change the effective step sizes or normalization rules, but do not introduce non-local, fractional-order memory in the update dynamics.

Beyond optimization-centric methods, privacy- and security-oriented FL algorithms form a complementary line of work. DPSGD \cite{dpsgd} enforces differential privacy by clipping and noising client updates, while secure aggregation protocols, including sparsified variants such as SparseSecAgg \cite{Sparsifiedsecureaggregation}, cryptographically hide individual updates from the server. These approaches provide formal privacy guarantees but can exacerbate optimization difficulties in highly non-IID regimes due to added noise or information loss. FOFedAvg is orthogonal to these mechanisms: it targets optimization robustness via fractional, memory-aware updates and can in principle be combined with DP or secure aggregation, as discussed in Appendix~A.

More recent methods such as FDSE \cite{Wang2025} and FedBM \cite{Ping2025} further enrich the FL landscape by introducing sophisticated regularization, correction, or bias-mitigation mechanisms tailoblack to heterogeneous environments. These algorithms represent stronger and more modern baselines than classical FedAvg-type methods. In our experiments, we therefore compare FOFedAvg not only against FedAvg, FedProx, and SCAFFOLD, but also against FedNova, FedAdam, MOON, DPSGD, SparseSecAgg, FDSE, and FedBM to provide a contemporary and comprehensive evaluation.

FOFedAvg distinguishes itself from the above methods by using Fractional-Order Stochastic Gradient Descent (FOSGD) on each client (Eq.~\eqref{e9}). Instead of relying solely on the current gradient, FOSGD scales updates by a trajectory-dependent factor $(\lvert \Theta_{t+1}-\Theta_t\rvert + \delta)^{1-\alpha}/\Gamma(2-\alpha)$, which implements a power-law weighting of recent parameter changes. This introduces an explicit, tunable memory mechanism that can smooth local updates and blackuce drift in non-IID settings, while preserving the communication pattern and aggregation rule of FedAvg.

\subsection{Fractional Calculus in Optimization}

Fractional calculus generalizes classical differentiation and integration to non-integer orders, yielding non-local operators that naturally encode memory effects and long-range temporal dependencies \cite{podlubny1999fractional,ortigueira2011fractional}. In many dynamical systems, fractional models have been shown to more accurately capture heblackitary phenomena and anomalous diffusion than integer-order counterparts \cite{westchen24-fot4stem}. These properties have motivated applications of fractional calculus in optimization, control, and signal processing.

In optimization and metaheuristics, fractional-order variants of classical algorithms have been proposed. Fractional-order particle swarm optimization has been shown to improve exploration–exploitation trade-offs by incorporating memory of past velocities and positions \cite{couceiro2012fractional}. Fractional derivatives have been used to enhance image processing and enhancement tasks \cite{abdAlrhman2024fractional} and to design more flexible feature extractors in signal classification \cite{dhar2019fractional}. These works exploit the non-local nature of fractional operators to stabilize or enrich the dynamics of optimization trajectories.

In machine learning, fractional-order gradient descent has been primarily studied in centralized settings. Yang et al.\ \cite{yang2023fractional} proposed an improved stochastic fractional-order gradient descent algorithm and demonstrated accelerated convergence on certain convex and non-convex problems by approximating Caputo derivatives via truncated series expansions. Hapsari et al.\ \cite{hapsari2021fractional} applied fractional-order optimization to support vector machines, achieving improved classification performance. These methods adapt the local update rule by incorporating fractional derivatives of the loss, but do not consider decentralized or federated architectures.

To date, the application of fractional calculus in federated learning remains limited. Existing FL algorithms almost exclusively rely on integer-order gradients, which are local in time and lack explicit memory of past updates beyond what is captublack implicitly by learning-rate schedules or momentum. This gap motivates the development of FL algorithms that systematically integrate fractional-order dynamics into client-side optimization.

\subsection{Positioning FOFedAvg’s Contribution}

FOFedAvg bridges the above gap by integrating fractional-order optimization into a standard FedAvg-style training loop. On each client, the gradient-based update is replaced by a fractional-order step derived from a Caputo-type derivative with order $0<\alpha\le 1$, approximated via a local surrogate that depends on the current gradient and the most recent parameter difference (Section~\ref{Sec:FracCal_FL}). This design yields a memory-aware update that compresses the effect of the past trajectory into a power-law scaling factor, while keeping per-round time and memory complexity comparable to FedAvg (Appendix~A).

Compablack to FedProx, SCAFFOLD, MOON, FedAdam, FedNova, and more recent baselines such as FDSE and FedBM, FOFedAvg introduces an explicit temporal memory mechanism rather than additional regularizers, control variates, or server-side adaptivity. This mechanism smooths local updates and can blackuce client drift in non-IID regimes without requiring extra communication or the storage of full gradient histories. Theoretical convergence guarantees in Section~\ref{Sec:convergence} show that, under standard smoothness and bounded-variance assumptions, the resulting fractional-order federated optimization converges to a stationary point for $0<\alpha\le 1$.

Empirically, we validate FOFedAvg on a diverse set of benchmarks, including MNIST \cite{lecun1998mnist}, FEMNIST \cite{cohen2017emnist}, EMNIST \cite{Cohen2017}, Sent140 \cite{go2009sent140}, CIFAR-10 and CIFAR-100 \cite{krizhevsky2009cifar}, the Cleveland heart disease dataset \cite{Sharma2017}, PneumoniaMNIST \cite{yang2021medmnist}, and Edge-IIoTset \cite{ferrag2022edgeiiotset}. Across these datasets and a range of non-IID partitioning schemes, FOFedAvg is consistently competitive with or superior to FedAvg and strong contemporary baselines, particularly in more heterogeneous regimes. By introducing fractional calculus into federated learning, FOFedAvg opens a new avenue for memory-aware, complexity-informed optimization in decentralized environments.


\section{Fractional Calculus in Federated Learning}\label{Sec:FracCal_FL}

Fractional calculus generalizes classical differentiation and integration to non-integer orders, yielding nonlocal operators that naturally encode memory effects and long-range temporal dependencies \cite{podlubny1999fractional, ortigueira2011fractional, partohaghighi2025roughnessml}. In contrast to integer-order derivatives, which depend only on local behavior of a function, fractional derivatives aggregate information over its past trajectory, making them well suited for modeling systems with heblackitary dynamics and history-dependent responses.

\begin{definition}[Gamma Function]\label{def:gamma_function}\cite{podlubny1999fractional}
The Gamma function, denoted \(\Gamma(z)\), is a generalization of the factorial to real and complex numbers, defined for \(\mathrm{Re}(z) > 0\) as
\begin{equation}
\Gamma(z) = \int_0^\infty t^{z-1} e^{-t} \, dt.
\end{equation}
For positive integers \(n\), we have \(\Gamma(n) = (n-1)!\), and it satisfies the functional equation \(\Gamma(z+1) = z \Gamma(z)\). In fractional calculus, the Gamma function is essential for defining binomial coefficients and normalizing kernels in fractional derivatives.
\end{definition}

The following are two common types of fractional derivatives.

\begin{definition}[Gr\"{u}nwald--Letnikov Derivative]\label{def:GL-derivative}\cite{monje2010fractional}
The Gr\"{u}nwald--Letnikov derivative provides a series representation of fractional differentiation and is defined by
\begin{equation}
  \Delta_h^\alpha f(t)
  \;=\;
  \lim_{h \to 0}
  \frac{1}{h^\alpha}
  \sum_{k=0}^{\lfloor t/h \rfloor}
    (-1)^k
    \binom{\alpha}{k}\,
    f\bigl(t - kh\bigr),
\end{equation}
where 
\(
  \binom{\alpha}{k}
  =
  \frac{\Gamma(\alpha+1)}{\Gamma(k+1)\,\Gamma(\alpha-k+1)}
\)
and \(\Gamma(\cdot)\) denotes the Gamma function.
\end{definition}

\begin{definition}[Caputo Derivative]\label{def:Caputo-derivative}\cite{monje2010fractional}
The Caputo derivative is frequently used in applications because it handles classical initial conditions in a natural way. It is defined as
\begin{equation}
{}^C D_{t}^{\alpha} f(t)
\;=\;
\frac{1}{\Gamma(n-\alpha)}
\int_{0}^{t}
  \frac{
    f^{(n)}(\tau)
  }{
    (t-\tau)^{\alpha-n+1}
  }
\, d\tau,
\quad
n-1 < \alpha < n,
\end{equation}
where \(n\) is the smallest integer greater than \(\alpha\). In particular, the Caputo derivative of a constant is zero, which makes it especially convenient for initial value problems in engineering and physics.
\end{definition}

{\color{black}
In this work, we \emph{restrict} to the regime \(n = 1\), so that
\[
0 < \alpha \le 1,
\]
and the Caputo derivative depends only on the first derivative \(f'(\tau)\). This choice is aligned with first-order optimization methods such as SGD and FedAvg and forms the mathematical basis of our Fractional Order SGD (FOSGD) and FOFedAvg algorithms.
Higher-order cases \(n \ge 2\) (i.e., \(1 < \alpha \le 2\), etc.) are not consideblack in our theoretical development and are left as future work.
}

A fundamental feature of fractional derivatives is their \emph{nonlocality}, which means that the derivative at a specific point depends on a range of past values, thereby incorporating memory effects \cite{ortigueira2011fractional}. We now describe the derivation of FOSGD, which integrates historical information and long-range dependencies into the optimization dynamics.

\subsection{Fractional Order Stochastic Gradient Descent (FOSGD)}
\noindent
The traditional gradient descent algorithm for the objective function \( f(\Theta) \) is
\begin{equation}
\Theta_{t+1} = \Theta_t - \eta \nabla f(\Theta_t),
\label{e3}
\end{equation}
where \(\eta\) is the learning rate and \(\nabla f(\Theta_t)\) denotes the gradient at \(\Theta_t\). By replacing the standard gradient in \eqref{e3} with a fractional-order derivative, we obtain the iteration
\begin{equation}
\Theta_{t+1} = \Theta_t - \mu_{t} D_{t}^{\alpha} f(\Theta_t).
\label{e4}
\end{equation}

{\color{black}
Following \cite{yang2023fractional}, we specialize to \(0 < \alpha \le 1\) and approximate the Caputo derivative \(D_t^\alpha f(t)\) using a Taylor-series expansion around an initial time \(t_0\). This yields
}
\begin{equation}
D_{t}^{\alpha} f(t)
=
\sum_{i=1}^{\infty}
\frac{f^{(i)}(t_0)}{\Gamma(i+1-\alpha)}
(t - t_0)^{i-\alpha},
\quad 0 < \alpha \le 1.
\label{e2}
\end{equation}
{\color{black}
Retaining the leading term in this series and translating it into a discrete-time optimization setting leads to the practical FOSGD update
}
\begin{equation}
\Theta_{t+2}
=
\Theta_{t+1}
-
\mu_t \,
\frac{\nabla f(\Theta_{t+1})}{\Gamma(2-\alpha)}
\,
\big\|\Theta_{t+1} - \Theta_t\big\|^{1-\alpha},
\quad 0 < \alpha \le 1,
\label{e8}
\end{equation}
where we adopt the decaying learning-rate schedule
\[
\mu_t = \frac{\mu_0}{\sqrt{t+1}}
\]
to stabilize updates and avoid division by zero at early iterations. To prevent the degenerate case where \(\Theta_{t+1} \approx \Theta_t\) leads to vanishing step sizes, we introduce a small regularization parameter \(\delta>0\), obtaining the regularized update
\begin{equation}
\Theta_{t+2}
=
\Theta_{t+1}
-
\mu_t \,
\frac{\nabla f(\Theta_{t+1})}{\Gamma(2-\alpha)}
\,
\big(\|\Theta_{t+1} - \Theta_t\|+\delta\big)^{1-\alpha},
\quad 0 < \alpha \le 1.
\label{e9}
\end{equation}
{\color{black}
In this form, the factor \(\big(\|\Theta_{t+1} - \Theta_t\|+\delta\big)^{1-\alpha}\) compresses the contribution of the past trajectory into a single-step memory term, implementing a fractional, power-law weighting of recent parameter changes.
}

{\color{black}
\paragraph{Computational viewpoint.}
A naive implementation of fractional differentiation based directly on the Caputo or Gr\"unwald--Letnikov definitions in Definitions~\ref{def:GL-derivative}--\ref{def:Caputo-derivative} would require storing and reweighting the \emph{entire} history of past iterates or gradients, leading to \(\mathcal{O}(T d)\) time and memory costs over \(T\) iterations for a \(d\)-dimensional model. This is prohibitive in large-scale federated learning. In our implementation, we deliberately avoid such full-history schemes and instead use the local surrogate in \eqref{e9}: the deep-learning framework computes the standard backpropagation gradient \(\nabla f(\Theta_{t+1})\), and the fractional effect is realized only through the scalar factor \(\big(\|\Theta_{t+1} - \Theta_t\|+\delta\big)^{1-\alpha} / \Gamma(2-\alpha)\). As a result, FOSGD and FOFedAvg retain the same \(\mathcal{O}(E n_k d)\) asymptotic complexity as FedAvg on each client, with only a lightweight \(\mathcal{O}(d)\) overhead per mini-batch for computing the parameter difference and applying the scalar scaling.
}

\subsection{The Importance of Fractional Calculus in Federated Learning}

Federated learning faces significant challenges due to heterogeneous and non-IID data, as well as intermittent communication between clients \cite{mcmahan2017communication, kairouz2021advances}. Conventional gradient-based federated optimization algorithms such as FedAvg may exhibit slow or unstable convergence under these conditions. Incorporating fractional derivatives offers a principled way to address these issues.

Fractional derivatives naturally retain more historical gradient information than integer-order methods, resulting in memory-aware updates that help stabilize training when client data distributions diverge. The optimizer can exploit this \emph{fractional memory} to adjust updates based on broader historical trends rather than relying solely on the most recent gradients. This increased adaptability is particularly advantageous in environments with infrequent communication or asynchronous client participation. Moreover, fractional-order methods can improve robustness to non-IID data, where local gradients often reflect different portions of the data distribution. By implementing an extended memory mechanism, these methods can blackuce the variance induced by mismatched updates and thereby enhance global convergence \cite{sheng2020convolutional}. In our FOSGD formulation and its use within FOFedAvg, this memory manifests concretely through the factor \(\big(\|\Theta_{t+1} - \Theta_t\|+\delta\big)^{1-\alpha}\) in \eqref{e9}, which implements a power-law weighting of recent parameter changes and stabilizes updates under client heterogeneity.

\begin{algorithm}
\caption{Fractional Order SGD (FOSGD) \cite{yang2023fractional}}
\label{alg:FOSGD}
\textbf{Input:} $\alpha, \mu_0, t_{\max}, \delta$
\begin{itemize}
    \item Initialize $t = 0$ and $\Theta_0$;
    \item Compute $\Theta_1$ using ordinary SGD:
    \[
    \Theta_1 = \Theta_0 - \mu_0 \nabla f(\Theta_0);
    \]
    \item For $t = 0,1,\dots,t_{\max}-2$ update
    \[
    \Theta_{t+2}
    =
    \Theta_{t+1}
    -
    \mu_t \frac{\nabla f(\Theta_{t+1})}{\Gamma(2-\alpha)}
    \big(\|\Theta_{t+1} - \Theta_t\| + \delta\big)^{1-\alpha},
    \]
    where $\mu_t = \frac{\mu_0}{\sqrt{t+1}}$;
    \item Return $\Theta_{t_{\max}}$.
\end{itemize}
\end{algorithm}

\section{Fractional Order Federated Averaging Algorithm}\label{Sec:FOFEDAVGMethod}

McMahan et al.\ introduced Federated Averaging (FedAvg) in \cite{mcmahan2017communication}. FedAvg initializes a global model on a central server and then updates it through repeated communication rounds. In each round, a subset of clients performs local training on their private data using SGD and sends their updated model parameters back to the server. The server then aggregates these updates using a weighted average based on the local dataset sizes. This procedure blackuces communication frequency by allowing multiple local updates before synchronization and has therefore become a workhorse for federated learning research and applications in heterogeneous, privacy-sensitive environments.

A significant portion of practical interest focuses on nonconvex neural network objectives, but the FedAvg algorithm is applicable to any finite-sum objective of the form
\[
\min_{\Theta \in \mathbb{R}^d} f(\Theta) = \frac{1}{n} \sum_{i=1}^{n} f_i(\Theta).
\]
In typical machine-learning settings, \( f_i(\Theta) \) represents the loss function \(\ell(x_i, y_i; \Theta)\) for example \((x_i,y_i)\) with parameters \(\Theta\). The data are partitioned across \(K\) clients, where \(\mathcal{P}_k\) denotes the index set of client \(k\) and \(n_k = |\mathcal{P}_k|\). The global objective can be rewritten as
\[
f(\Theta)
=
\sum_{k=1}^{K} \frac{n_k}{n} F_k(\Theta),
\quad
\text{where}
\quad
F_k(\Theta)
=
\frac{1}{n_k} \sum_{i \in \mathcal{P}_k} f_i(\Theta),
\]
and \(n = \sum_{k=1}^K n_k\). When data are IID across clients, we have
\(\mathbb{E}_{\mathcal{P}_k}[F_k(\Theta)] = f(\Theta)\); in non-IID regimes, however, each \(F_k\) may be biased relative to \(f\), leading to client drift.

FedAvg proceeds as follows. The server initializes a global model \(\Theta_0\) and updates it iteratively over \(T\) communication rounds. At round \(t\), the server selects a subset of clients \(S_t\) of size
\[
m = \max(C \cdot K, 1),
\]
where \(C \in (0,1]\) is the client participation fraction and \(K\) is the total number of clients. Each selected client \(k \in S_t\) then performs local training on its dataset \(\mathcal{P}_k\) via the following steps:
\begin{enumerate}
    \item Partition \(\mathcal{P}_k\) into mini-batches of size \(B\).
    \item For \(E\) local epochs, loop over each mini-batch \(b\) and update the local model using stochastic gradient descent:
    \[
    \Theta \leftarrow \Theta - \eta \nabla \ell(\Theta; b),
    \]
\end{enumerate}
where \(\ell(\Theta; b)\) is the loss on mini-batch \(b\) and \(\eta\) is the learning rate. After completing its local updates at round \(t\), client \(k\) returns its updated parameters \(\Theta_{t+1}^{(k)}\) to the server. The server then forms the new global model by a weighted average:
\[
\Theta_{t+1}
\leftarrow
\sum_{k \in S_t} \frac{n_k}{n} \Theta_{t+1}^{(k)},
\]
where \(n_k\) is the number of data points on client \(k\) and \(n = \sum_{k \in S_t} n_k\) is the total number of samples held by the participating clients in round \(t\). This iterative procedure blackuces communication overhead while maintaining data locality and is the baseline upon which we build our fractional-order extension.

\subsection{Fractional-Order Extension}

As described in Algorithm~\ref{alg:fedavg_fsgd}, we propose a variation of FedAvg called \emph{Fractional-Order Federated Averaging} (FOFedAvg). The key idea is to replace the standard local SGD updates with the Fractional Order SGD (FOSGD) step derived in Section~\ref{Sec:FracCal_FL}, thereby injecting memory and long-range temporal dependencies into the client-side optimization.

FOFedAvg targets one of the central challenges in FL: non-IID data distributions across clients, which induce biased local gradients and client drift. To address this, we adopt a Caputo-type fractional derivative of order \(\alpha \in (0,1]\), denoted \(D_t^\alpha\), and instantiate it via the regularized FOSGD update in \eqref{e9}. Conceptually, for a mini-batch \(b\) on client \(k\) at round \(t \ge 1\), a local update step takes the form
\begin{equation}
\Theta_t^{(k)}
\;\leftarrow\;
\Theta_t^{(k)}
-
\frac{\mu_t}{\Gamma(2-\alpha)}
\big(\|\Theta_t^{(k)} - \Theta_{t-1}^{(k)}\| + \delta\big)^{1-\alpha}
\nabla \ell\big(\Theta_t^{(k)}; b\big),
\label{eq:fofedavg_local_update}
\end{equation}
where \(\mu_t = \mu_0 / \sqrt{t+1}\) is a decaying learning-rate schedule, \(\delta > 0\) is a small regularization constant, and \(\Theta_{t-1}^{(k)}\) denotes the previous-round parameters for client \(k\). For \(t = 0\), we recover standard SGD updates without the fractional memory term. In this formulation, the scalar factor
\[
\big(\|\Theta_t^{(k)} - \Theta_{t-1}^{(k)}\| + \delta\big)^{1-\alpha} / \Gamma(2-\alpha)
\]
acts as a memory-aware modulation of the step size: larger changes between rounds (e.g., due to highly skewed local data) alter the effective learning rate, while smaller changes emphasize accumulated history. This implements a power-law weighting of recent parameter differences, consistent with the fractional-order dynamics derived in Section~\ref{Sec:FracCal_FL}, but without storing the full trajectory. Intuitively, this fractional update aggregates information from past rounds into a single-step memory term that smooths abrupt changes in the local optimization path. In non-IID scenarios where some classes are over-represented on certain clients, the long-range memory encoded by \eqref{eq:fofedavg_local_update} helps temper the impact of highly biased local steps by tying them to the previous-round state. Although FOFedAvg does not semantically distinguish between gradients originating from outliers and those from underrepresented data, its power-law temporal averaging can blackuce the variance of the effective update direction and mitigate client drift at the level of time-smoothed dynamics. This effect is particularly useful in healthcare or IoT deployments, where local datasets can be highly skewed and client participation intermittent; by stabilizing local updates, FOFedAvg can blackuce the number of communication rounds requiblack to reach a target performance while maintaining robustness under heterogeneous conditions.

\begin{algorithm}[btph]
\caption{Fractional Order Federated Averaging Algorithm (FOFedAvg)}
\label{alg:fedavg_fsgd}

\textbf{Server executes:}
\begin{algorithmic}[1]
\State Initialize global model $\Theta_0$, fractional order $\alpha \in (0,1]$, initial learning rate $\mu_0$, regularization constant $\delta$.
\For{each round $t = 0, 1, 2, \dots$}
    \State Randomly select a subset $S_t$ of $m = \max(C \cdot K, 1)$ clients, where $C$ is the client participation fraction and $K$ is the total number of clients.
    \State Send the current global model $\Theta_t$ (and, if needed, $\Theta_{t-1}$ for $t \ge 1$) to all clients in $S_t$.
    \For{each client $k \in S_t$ \textbf{in parallel}}
        \State $\Theta_{t+1}^{(k)} \gets$ \textbf{ClientUpdate}($k, \Theta_t, t, \alpha, \mu_0, \delta$)
    \EndFor
    \State Aggregate updates:
    \[
    \Theta_{t+1} = \sum_{k \in S_t} \frac{n_k}{n} \Theta_{t+1}^{(k)},
    \]
    where $n_k$ is the number of data points on client $k$ and $n = \sum_{k \in S_t} n_k$.
\EndFor
\end{algorithmic}

\textbf{ClientUpdate($k, \Theta_t, t, \alpha, \mu_0, \delta$):}
\begin{algorithmic}[1]
\State Initialize local model $\Theta_t^{(k)} \gets \Theta_t$.
\State (If $t \ge 1$) retain previous-round parameters $\Theta_{t-1}^{(k)}$ for the memory term; for $t = 0$ use standard SGD.
\State Split local data $\mathcal{P}_k$ into mini-batches of size $B$.
\For{local epoch $e = 1, 2, \dots, E$}
    \For{each mini-batch $b$}
        \If{$t = 0$}
            \State \textbf{Standard SGD update:}
            \[
            \Theta_t^{(k)} \leftarrow \Theta_t^{(k)} - \mu_0 \nabla \ell\big(\Theta_t^{(k)}; b\big)
            \]
        \Else
            \State \textbf{Fractional SGD update:}
            \[
            \Theta_t^{(k)}
            \leftarrow
            \Theta_t^{(k)}
            -
            \frac{\mu_0}{\sqrt{t+1}}
            \cdot
            \frac{\nabla \ell\big(\Theta_t^{(k)}; b\big)}{\Gamma(2-\alpha)}
            \cdot
            \big(\|\Theta_t^{(k)} - \Theta_{t-1}^{(k)}\| + \delta\big)^{1-\alpha}
            \]
        \EndIf
    \EndFor
\EndFor
\State \textbf{return} $\Theta_{t+1}^{(k)} \gets \Theta_t^{(k)}$ to the server.
\end{algorithmic}
\end{algorithm}

\section{Convergence analysis}\label{Sec:convergence}
In this section, we examine the convergence characteristics of the FOSGD algorithm and provide theoretical validation as previously mentioned. The assumptions that will be consideblack are as follows:\\

\noindent \textbf{Notation.} 
Given $d$, we denote the Euclidean norm on $\mathbb{R}^d$ by $\|\cdot\|$. Suppose we have a function $f:\mathbb{R}^d\to\mathbb{R}$ with the following properties:

\noindent \textbf{Assumption 1: Smoothness.}  
The function \(f : \mathbb{R}^d \to \mathbb{R}\) is assumed to have a \emph{Lipschitz continuous gradient}, characterized by the existence of a constant \(L > 0\) such that, for all \(x, y \in \mathbb{R}^d\),
\[
\|\nabla f(x) - \nabla f(y)\| \leq L \|x - y\|.
\]
\noindent \textbf{Assumption 2: Lower Boundedness.}  
The function \(f : \mathbb{R}^d \to \mathbb{R}\) is also assumed to be bounded from below, meaning there exists a constant \(f_{\inf} \in \mathbb{R}\) such that
\[
f(x) \geq f_{\inf} \quad \text{for all } x \in \mathbb{R}^d.
\]

\begin{theorem}[Convergence to Stationary Points]
\label{thm:main}
Let $f:\mathbb{R}^d \to \mathbb{R}$ be an $L$-smooth (potentially non-convex) function with a lower bound $f_{\inf}$. Suppose $0 < \alpha \le 1$ (fractional order), and consider the sequence \(\{\Theta_t\}\) generated by
\begin{equation}
\label{eq:Theta_update}
\Theta_{t+2}
\;=\; 
\Theta_{t+1}
\;-\;
\alpha_t\,\nabla f(\Theta_{t+1}),
\end{equation}
where 
\begin{equation}
\label{eq:alpha_t}
\alpha_t
\;=\;
\frac{\mu_t}{\Gamma(2-\alpha)}
\Bigl(\|\Theta_{t+1} - \Theta_t\| + \delta\Bigr)^{1-\alpha}
\;\le\;
\bar{\alpha}
\;\le\;
\frac{2}{L}.
\end{equation}
Here, $\bar{\alpha} > 0$ is an upper bound ensuring $\alpha_t \le \bar{\alpha} \le \tfrac{2}{L}$ for all $t$. Then,
\[
\liminf_{t \to \infty}\|\nabla f(\Theta_t)\|
\;=\;
0.
\]
In particular, if the sequence $\{\Theta_t\}$ is bounded, every limit point of $\{\Theta_t\}$ is a stationary point of $f$.
\end{theorem}

\begin{proof}
\noindent
Define a Lyapunov function $V_t = f(\Theta_t)$. We aim to show $V_t$ cannot keep decreasing indefinitely without forcing the gradient norms to vanish. Because $f$ is $L$-smooth, for any $x,y \in \mathbb{R}^d$, we have:
\begin{equation}
\label{eq:Lsmooth}
f(y)
\;\le\;
f(x) 
\;+\;
\nabla f(x)^\top (y - x)
\;+\;
\frac{L}{2}\,\|y - x\|^2.
\end{equation}
Taking $x = \Theta_{t+1}$ and $y = \Theta_{t+2}$, we substitute into \eqref{eq:Lsmooth}:

\begin{equation}
\begin{split}
f(\Theta_{t+2})
&\le
f(\Theta_{t+1}) 
\;+\;
\nabla f(\Theta_{t+1})^\top\! (\Theta_{t+2} - \Theta_{t+1})
\\[5pt]
&\quad+\;
\frac{L}{2}\,\|\Theta_{t+2} - \Theta_{t+1}\|^2.
\end{split}
\end{equation}

\medskip
\noindent
By \eqref{eq:Theta_update}, we have 
\[
\Theta_{t+2} - \Theta_{t+1}
\;=\;
-\,\alpha_t \,\nabla f(\Theta_{t+1}).
\]
Plugging this back in, we get:
\begin{equation}
\begin{split}
f(\Theta_{t+2})
&\le 
f(\Theta_{t+1}) 
\;+\;
\nabla f(\Theta_{t+1})^\top\!\Bigl(-\,\alpha_t\,\nabla f(\Theta_{t+1})\Bigr)
\\[5pt]
&\quad+\;
\frac{L}{2}\,\Bigl\| -\,\alpha_t\,\nabla f(\Theta_{t+1})\Bigr\|^2
\\[5pt]
&=\;
f(\Theta_{t+1}) 
\;-\;
\alpha_t\,\|\nabla f(\Theta_{t+1})\|^2
\;+\;
\frac{L\,\alpha_t^2}{2}\,\|\nabla f(\Theta_{t+1})\|^2.
\end{split}
\end{equation}

Rearranging yields the sufficient decrease inequality:
\begin{equation}
\label{eq:suffdec}
f(\Theta_{t+2})
\;\le\;
f(\Theta_{t+1})
\;-\;
\Bigl(\alpha_t - \tfrac{L}{2}\,\alpha_t^2\Bigr)\,
\|\nabla f(\Theta_{t+1})\|^2.
\end{equation}
Define 
\[
\kappa_t 
\;=\; 
\alpha_t - \tfrac{L}{2}\,\alpha_t^2
\;=\;
\alpha_t\,\Bigl(1 - \tfrac{L}{2}\,\alpha_t\Bigr).
\]
Then \eqref{eq:suffdec} becomes
\[
f(\Theta_{t+2})
\;\le\;
f(\Theta_{t+1})
\;-\;
\kappa_t\,\|\nabla f(\Theta_{t+1})\|^2.
\tag{$*$}
\]

\medskip
\noindent
For $\kappa_t$ to be strictly positive, we need 
\[
\alpha_t 
\;\le\;
\frac{2}{L}.
\]
From \eqref{eq:alpha_t}, this means
\[
\frac{\mu_t}{\Gamma(2 - \alpha)} 
\bigl(\|\Theta_{t+1} - \Theta_t\| + \delta\bigr)^{1-\alpha}
\;\le\;
\bar{\alpha}
\;\le\;
\frac{2}{L}.
\]
Thus, provided we select $\mu_t$ (and $\delta$) such that the above expression remains bounded by $\tfrac{2}{L}$ (e.g., by choosing a small enough constant $\mu_t$ or a diminishing sequence in $t$), we guarantee $\kappa_t \ge 0$.  

\smallskip
\noindent
Under $L$-smoothness and typical gradient-descent assumptions, the iterates often remain in a bounded region. In particular, if $\|\nabla f(\Theta_t)\|$ is not excessively large or if $\mu_t$ diminishes over time, then $\|\Theta_{t+1} - \Theta_t\|$ will not explode. This ensures the factor $(\|\Theta_{t+1} - \Theta_t\| + \delta)^{1-\alpha}$ remains finite, helping keep $\alpha_t \le \frac{2}{L}$.

\medskip
\noindent
Let $V_{t+1} = f(\Theta_{t+1})$. Summing $(*)$ from $t=0$ to $t=T-1$, we obtain:
\[
\begin{aligned}
f(\Theta_{T+1})
&\;=\;
V_{T+1}
\;\le\;
V_{1} 
\;-\;
\sum_{t=0}^{T-1}
\kappa_t\,\|\nabla f(\Theta_{t+1})\|^2.
\end{aligned}
\]
Since $f(\Theta_{T+1}) \ge f_{\inf}$ for all $T$, it follows that
\[
f_{\inf}
\;\le\;
V_{1}
\;-\;
\sum_{t=0}^{T-1}
\kappa_t\,\|\nabla f(\Theta_{t+1})\|^2,
\]
so
\[
\sum_{t=0}^{T-1}
\kappa_t\,\|\nabla f(\Theta_{t+1})\|^2
\;\le\;
V_{1} 
\;-\; 
f_{\inf}.
\]
Taking $T \to \infty$ shows
\[
\sum_{t=0}^{\infty}
\kappa_t\,\|\nabla f(\Theta_{t+1})\|^2
\;<\;
\infty.
\]
Since each term is nonnegative, we must have 
\[
\kappa_t\,\|\nabla f(\Theta_{t+1})\|^2 \;\to\; 0 
\quad\text{as } t\to\infty.
\]
\noindent
If $\{\kappa_t\}$ does not rapidly decay to zero (ensublack by keeping $\alpha_t$ in a reasonable range and not letting $\mu_t$ shrink too fast), then
\[
\liminf_{t\to\infty}\|\nabla f(\Theta_{t+1})\|\;=\;0.
\]
Relabeling the index gives
\[
\liminf_{t\to\infty}\|\nabla f(\Theta_t)\|\;=\;0.
\]
If, in addition, the sequence $\{\Theta_t\}$ is bounded, any cluster point $\Theta^*$ of $\{\Theta_t\}$ satisfies 
\[
\nabla f(\Theta^*) = 0,
\]
meaning \(\Theta^*\) is a \emph{stationary point} of \(f\). 

We note that $\alpha \in (0,1]$ and 
\[
\alpha_t 
\;=\;
\frac{\mu_t}{\Gamma(2-\alpha)}
\Bigl(\|\Theta_{t+1} - \Theta_t\| + \delta\Bigr)^{1-\alpha}
\]
captures the “fractional” aspect in the step size. However, the sufficient condition $\alpha_t \le \frac{2}{L}$ aligns with conventional non-convex optimization theory. By combining fractional-order terms with a properly bounded (or diminishing) $\mu_t$, we preserve the standard guarantee that the iterates asymptotically approach the set of stationary points of $f$.
\end{proof}

\begin{remark}[Role of $\delta>0$]
Incorporating $\delta$ into the expression $\bigl(\|\Theta_{t+1} - \Theta_t\| + \delta\bigr)^{1-\alpha}$ mitigates problematic scenarios in which $\|\Theta_{t+1} - \Theta_t\|\approx 0$ might lead to an unstable or excessively large effective step size. 
\end{remark}
\textcolor{black}{To see a theoretical foundation for the robustness and contributions of the Fractional-Order Federated Averaging (FOFedAvg) algorithm, including its handling of non-IID data, smoothness properties, memory effects, bias analysis, impact of fractional order $\alpha$, justification of local epochs, and privacy considerations, refer to Appendix A.}

\section{Experiments}\label{Sec:Experiments}
\textcolor{black}{To investigate the efficacy of the proposed algorithm we conduct a comprehensive experimental study. Our experiments evaluate a comprehensive suite of algorithms, including FedAvg, FOFedAvg, MOON, SCAFFOLD, FedProx, FedNova, FedAdam, FDSE, FedBM, and SparseSecAgg (SSecAgg) across diverse datasets under both independent and identically distributed (IID) and non-IID data distributions. The primary objectives are to assess convergence behavior, test accuracy, communication efficiency, and robustness to heterogeneous data. We consider nine benchmark datasets: MNIST, EMNIST, Cleveland heart disease, FEMNIST, CIFAR-10, CIFAR-100, Sent140, PneumoniaMNIST, and Edge-IIoTset, each partitioned across clients using Dirichlet or pathological non-IID schemes to simulate real-world federated settings. These experiments provide deep insights into each algorithm's scalability, stability, and performance in challenging federated learning scenarios. We summarize the key mathematical symbols in Table~\ref{tab:symbols}.}

\begin{table}[H]
\fontsize{7}{9}\selectfont
\centering
\caption{Key Mathematical Symbols Used in FOFedAvg}
\label{tab:symbols}
\begin{tabularx}{0.7\columnwidth}{c>{\RaggedRight}X}
\toprule
\textbf{Symbol} & \textbf{Description} \\
\midrule
$\Gamma(z)$ & Gamma function, generalizing the factorial \\
$\alpha$ & Fractional order, $0 < \alpha \leq 1$ \\
$\mu_t$ & Learning rate, $\mu_t = \mu_0 / \sqrt{t+1}$ \\
$\mu_0$ & Initial learning rate (FOSGD algorithm) \\
$\delta$ & Regularization constant, $\delta > 0$ \\
$\Theta_t$ & Model parameters at iteration $t$ \\
$\nabla f(\Theta_t)$ & Gradient of the objective function at $\Theta_t$ \\
$\mathrm{D}_t^\alpha$ & Fractional derivative of order $\alpha$ \\
$\Delta_h^\alpha$ & Grünwald-Letnikov fractional derivative \\
${}^C \mathrm{D}_t^\alpha$ & Caputo fractional derivative \\
$\binom{\alpha}{k}$ & Generalized binomial coefficient \\
$f(\Theta)$ & Global objective function \\
$F_k(\Theta)$ & Local objective function for client $k$ \\
$\mathcal{P}_k$ & Local dataset of client $k$ \\
$n_k$ & Number of data points in client $k$'s dataset \\
$n$ & Total number of data points across selected clients \\
$K$ & Total number of clients \\
$C$ & Fraction of participating clients \\
$S_t$ & Subset of clients selected at round $t$ \\
$B$ & Mini-batch size (FOFedAvg algorithm) \\
$E$ & Number of local epochs (FOFedAvg algorithm) \\
$L$ & Lipschitz constant for gradient smoothness \\
$\kappa_t$ & Decrease factor in convergence analysis \\
$f_{\inf}$ & Lower bound of the objective function \\
\bottomrule
\end{tabularx}
\end{table}

\begin{figure}[H]
\centering
\includegraphics[height=5cm,width=7cm]{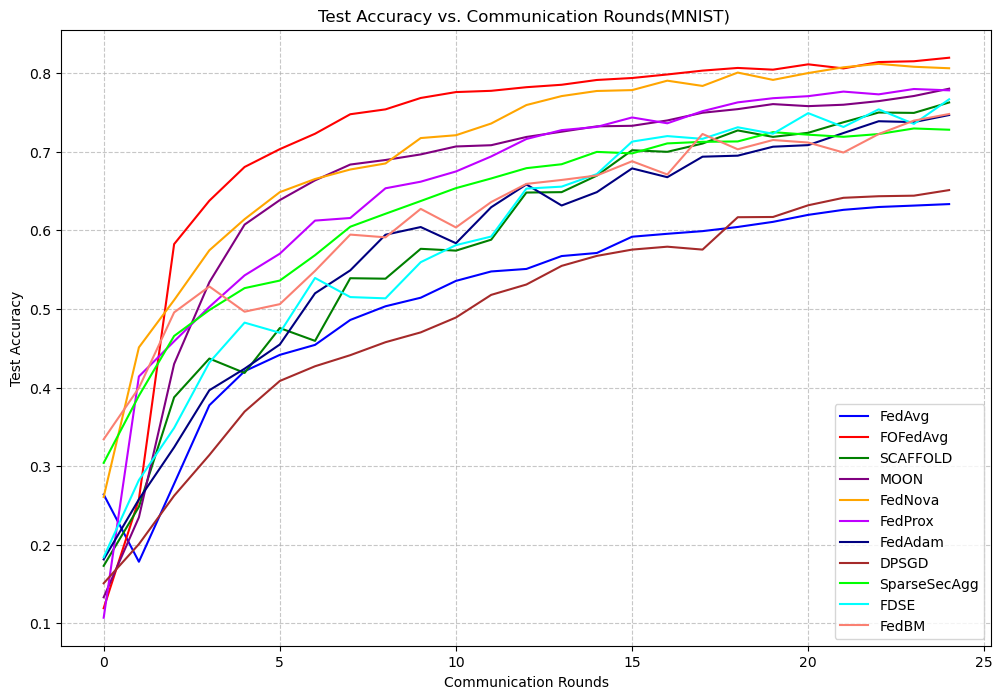}
\caption{Comparison of federated algorithms on the MNIST dataset under a non-IID setting with 10 clients.}
\label{fig:federated_learning_plotMNIST}
\end{figure}

The performance of ten federated learning algorithms --- FedAvg, FOFedAvg, SCAFFOLD, MOON, FedNova, FedProx, FedAdam, FDSE, FedBM, and SparseSecAgg --- is compablack based on their test accuracy over 25 communication rounds on the MNIST dataset. Figure~\ref{fig:federated_learning_plotMNIST} illustrates the accuracy trends. FOFedAvg and FedNova consistently outperform the other methods, achieving accuracies around 0.81–0.82, with FOFedAvg showing a steep initial increase and sustained high performance. FedNova starts strong and maintains competitive accuracy. MOON and FedProx also perform well, reaching approximately 0.78–0.80, followed by SCAFFOLD, FedAdam, FDSE, FedBM, and SparseSecAgg. FedAvg, the baseline, exhibits the slowest convergence and lowest final accuracy, highlighting the limitations of simple averaging in heterogeneous settings.

\begin{figure}[H]
\centering
\includegraphics[height=5cm,width=7cm]{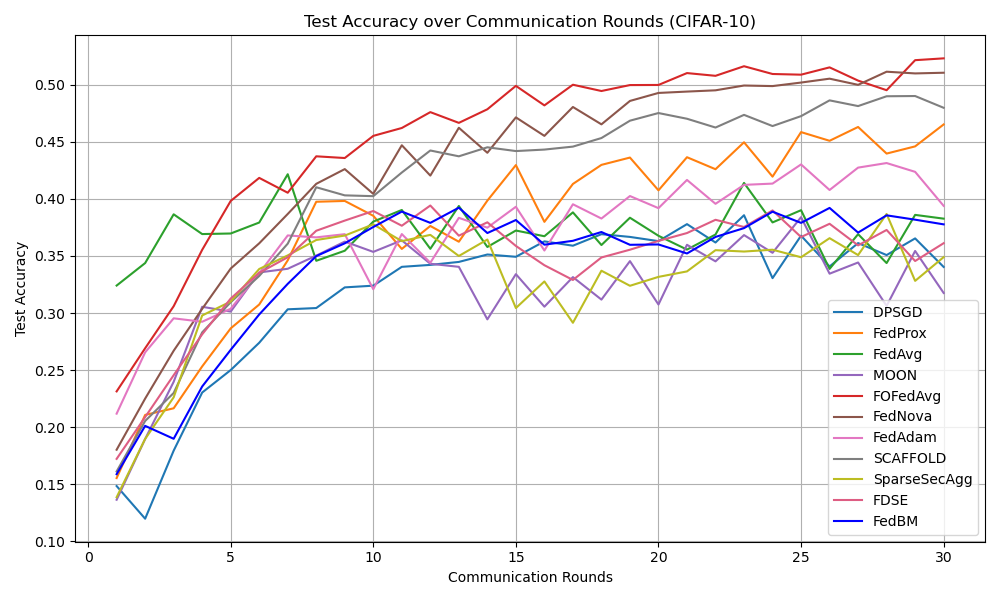}
\caption{Comparison of federated algorithms on the CIFAR-10 dataset under a non-IID setting.}
\label{fig:federated_learning_plotCIFAR10}
\end{figure}

Federated learning on complex datasets like CIFAR-10 poses unique challenges, demanding robust algorithms to achieve high accuracy under decentralized constraints. Figure~\ref{fig:federated_learning_plotCIFAR10} visualizes the test accuracy of ten methods over 30 communication rounds. FOFedAvg leads with a final accuracy of 0.5231, surging early (0.2314 at round 1) and maintaining a strong trajectory. FedNova follows closely, reaching 0.5105 with consistent gains, showcasing its ability to handle heterogeneity compablack to other algorithms.

\begin{figure}[H]
\centering
\includegraphics[height=5cm,width=7cm]{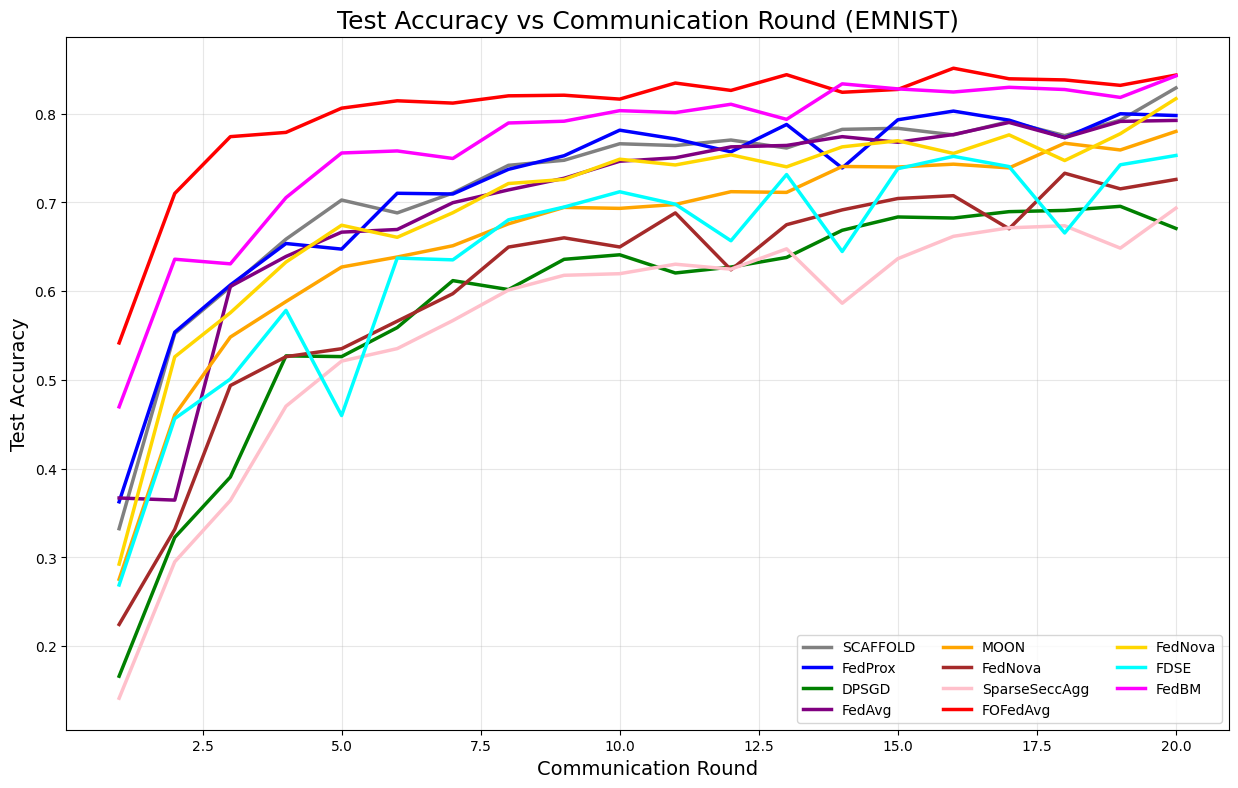}
\caption{Comparison of federated learning algorithms on the EMNIST dataset.}
\label{fig:EMNIST3_comparison}
\end{figure}

Figure~\ref{fig:EMNIST3_comparison} illustrates the test accuracy of ten federated learning algorithms over 20 communication rounds on the EMNIST dataset. FOFedAvg stands out as the top performer, showing rapid convergence and strong final results. FedNova, MOON, and FDSE exhibit moderate progress, with steady but slower advancement throughout. FedAvg, along with SparseSecAgg and some of the other baselines, lags behind, reflecting challenges in this heterogeneous setting and underscoring the benefit of memory-aware updates.

\begin{figure}[H]
\centering
\includegraphics[height=5cm,width=7cm]{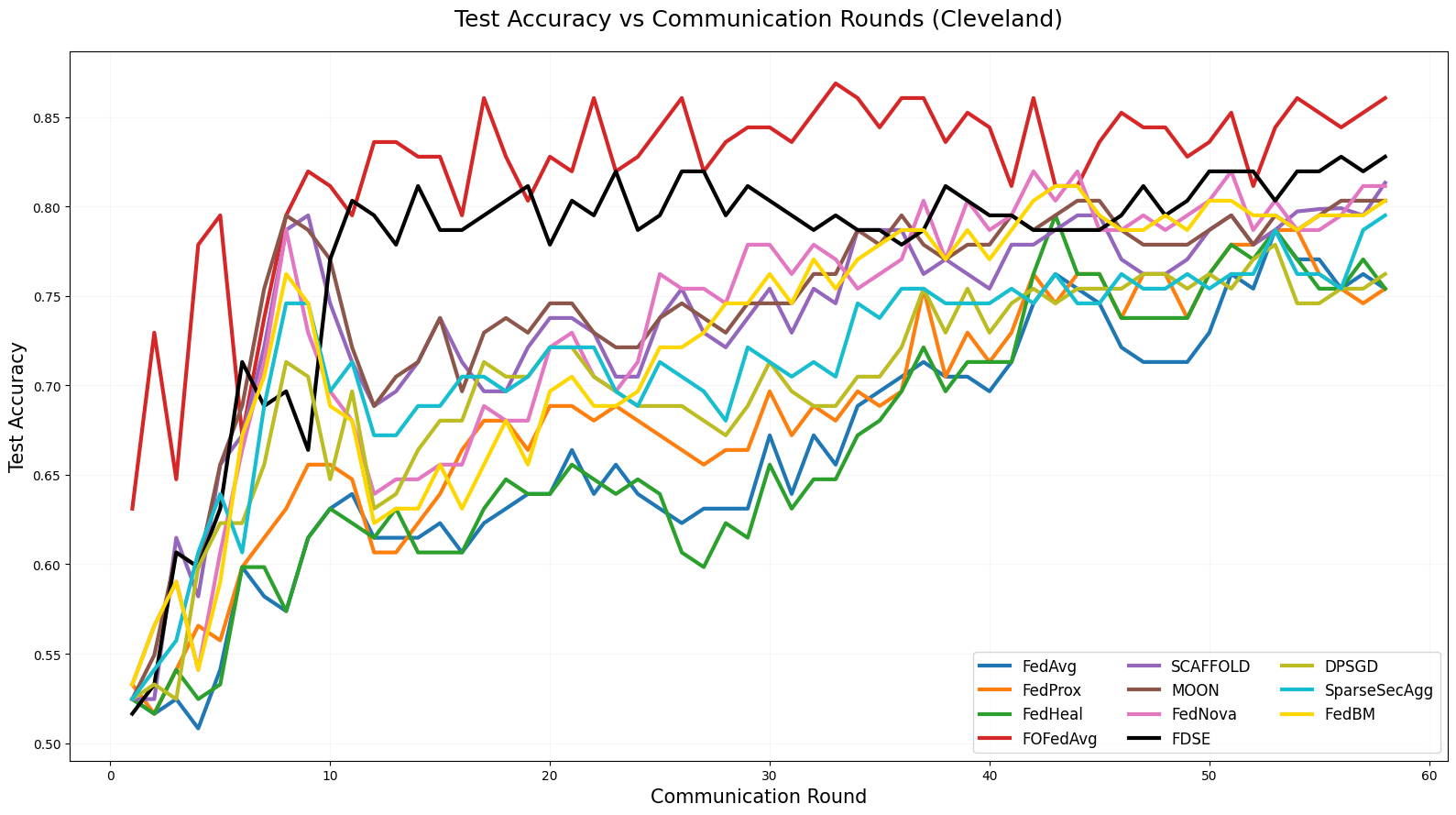}
\caption{Comparison of federated learning algorithms on the Cleveland heart disease dataset.}
\label{fig:Cleveland3}
\end{figure}

Figure~\ref{fig:Cleveland3} shows the test accuracy of ten federated learning algorithms on the Cleveland heart disease dataset. FOFedAvg clearly outperforms all other methods, achieving the fastest convergence and the highest final accuracy. Standard FedAvg shows comparatively slower convergence, while SparseSecAgg attains the lowest overall accuracy, indicating a noticeable accuracy–communication/privacy trade-off in this setting.

\begin{figure}[H]
\centering
\includegraphics[height=5cm,width=7cm]{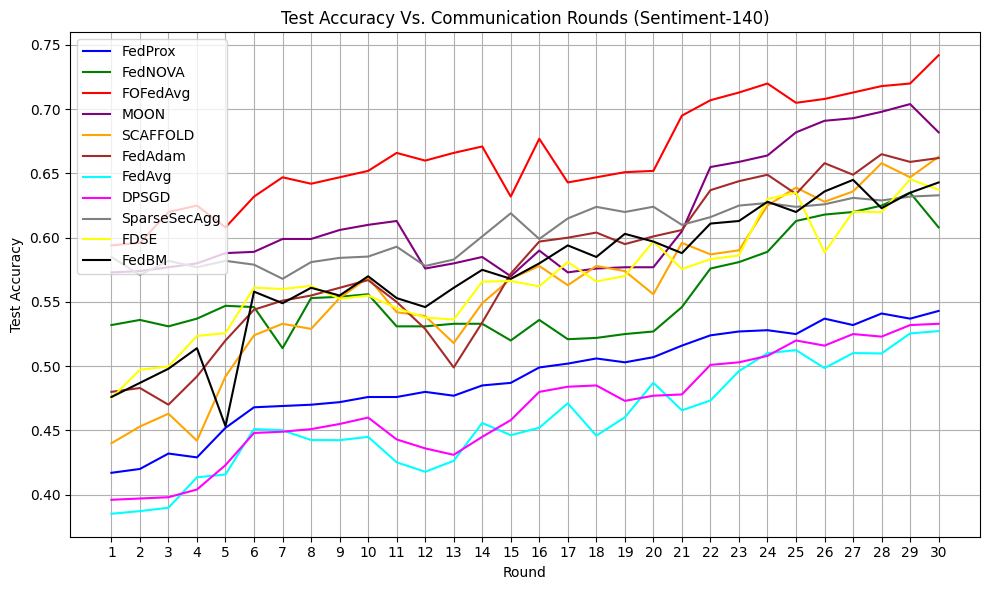}
\caption{Test accuracy of federated learning algorithms on the Sent140 dataset.}
\label{fig:test_accuracy_sent140}
\end{figure}

\noindent Figure~\ref{fig:test_accuracy_sent140} shows the evolution of test accuracy for different federated learning algorithms on the Sent140 dataset. Across nearly all rounds, the fractional-order method FOFedAvg yields the highest accuracy and exhibits a steady upward trend, indicating both fast adaptation and strong final performance. A second group of methods, including SCAFFOLD, MOON, FedAdam, and FedBM, closely track one another and consistently outperform the classical FedAvg baseline, suggesting that variance-blackuction, proximal, or adaptive mechanisms are beneficial on this highly non-IID text task. In contrast, FedAvg and, to a lesser extent, FedProx, FedNova, SparseSecAgg, and FDSE lag behind the leading methods, reflecting slower convergence and lower asymptotic accuracy under the same communication budget.

\begin{figure}[H]
\centering
\includegraphics[height=6cm,width=8cm]{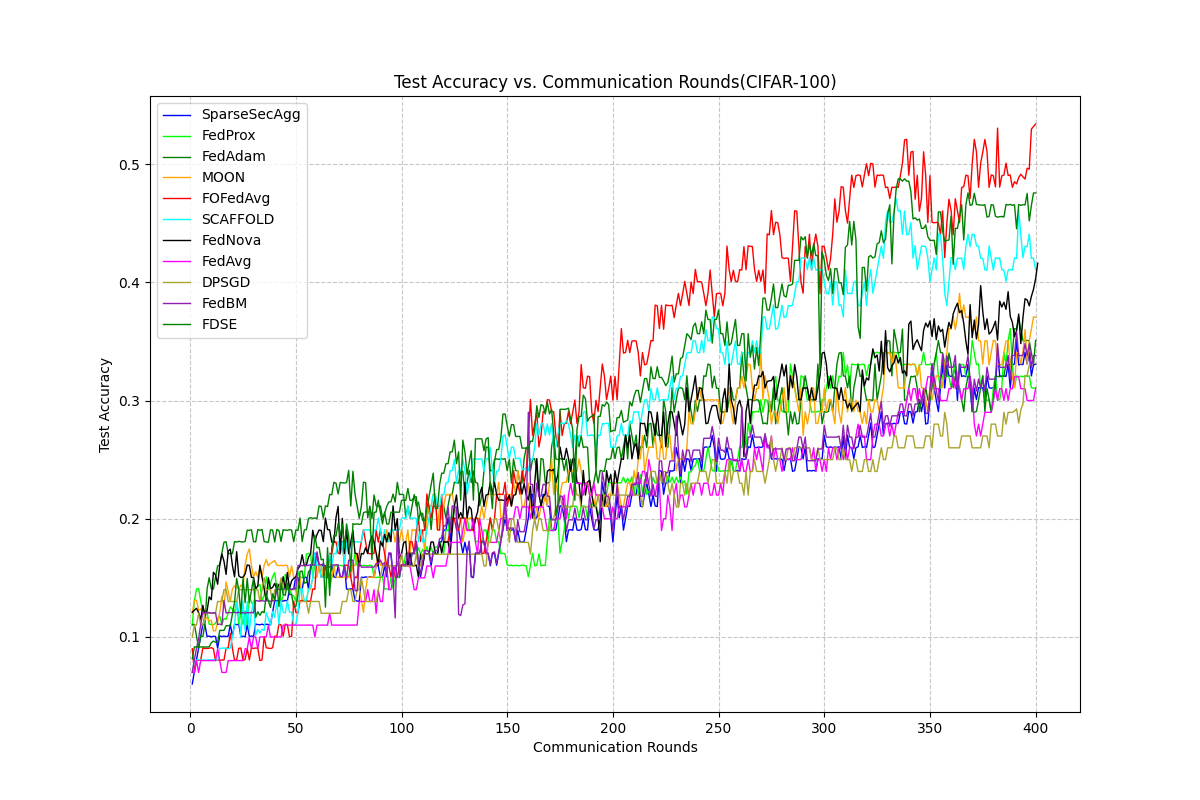}
\caption{Test accuracy of federated learning algorithms on the CIFAR-100 dataset.}
\label{fig:test_accuracy_cifar100}
\end{figure}

The test-accuracy trajectory in Figure~\ref{fig:test_accuracy_cifar100} reveals a clear trend: although FedAdam leads during the early rounds, its advantage fades around round 170, where our proposed method begins to dominate in the highly heterogeneous setting. Beyond this point, FOFedAvg consistently outperforms all baselines, followed by FedAdam and SCAFFOLD.

\begin{figure}[H] 
\centering
\includegraphics[height=5cm,width=7cm]{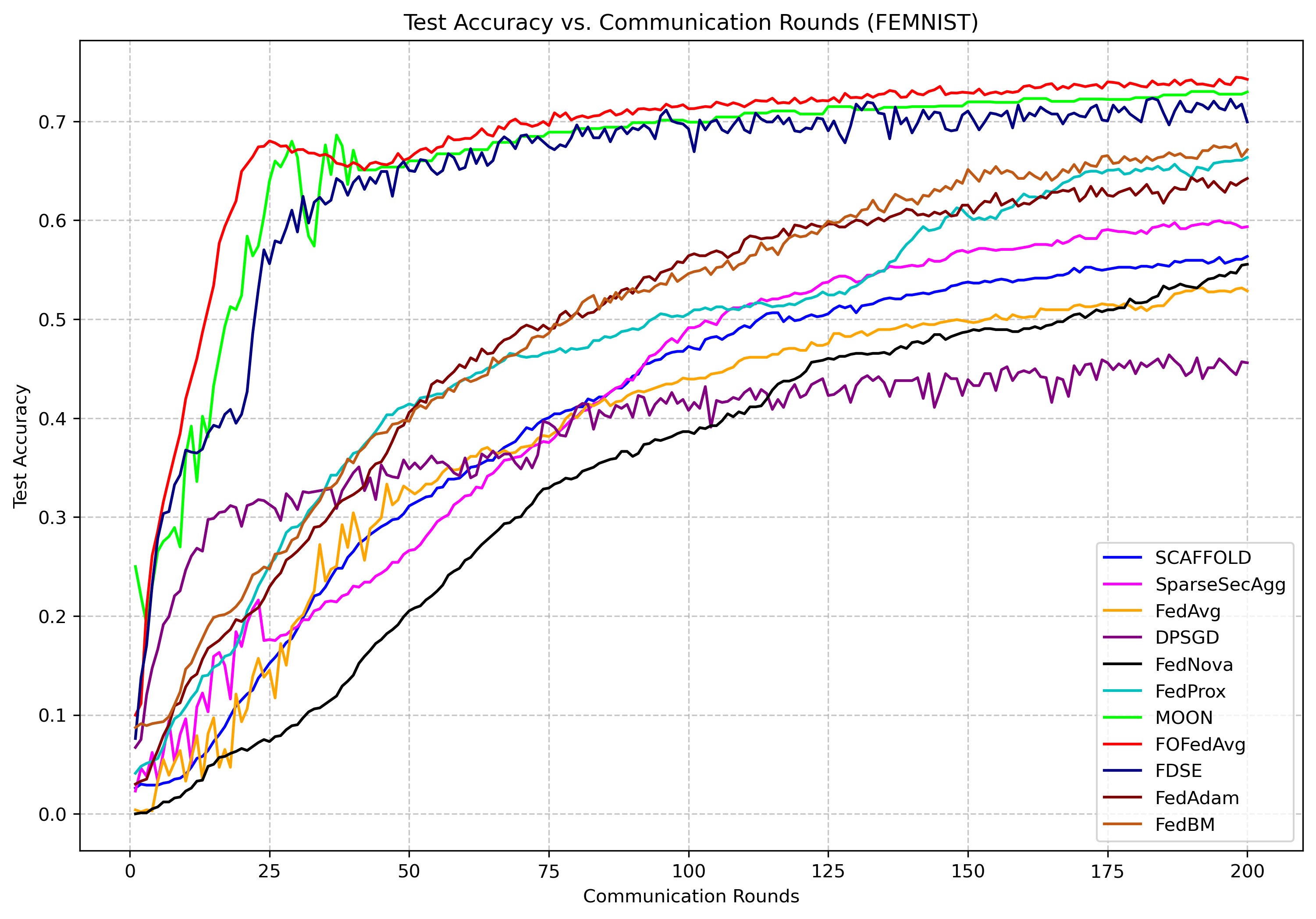} 
\caption{Test accuracy of federated learning algorithms on the FEMNIST dataset.} 
\label{fig:test_accuracy_efmnist} 
\end{figure}

Figure~\ref{fig:test_accuracy_efmnist} compares the test accuracy of several federated learning baselines on FEMNIST. The fractional-order method FOFedAvg exhibits both the fastest initial gain and the highest final accuracy, maintaining a clear margin over all competitors after roughly \(50\) rounds. MOON, FedProx, SCAFFOLD, FDSE, FedAdam, and FedBM form a second tier: they converge substantially faster and to higher accuracies than the vanilla FedAvg baseline, but remain consistently below FOFedAvg in the steady state. In contrast, SparseSecAgg, FedAvg, and FedNova lag behind throughout training, reflecting slower optimization and lower asymptotic performance under the same communication budget.


\begin{figure}[H]
\centering
\includegraphics[height=5cm,width=7cm]{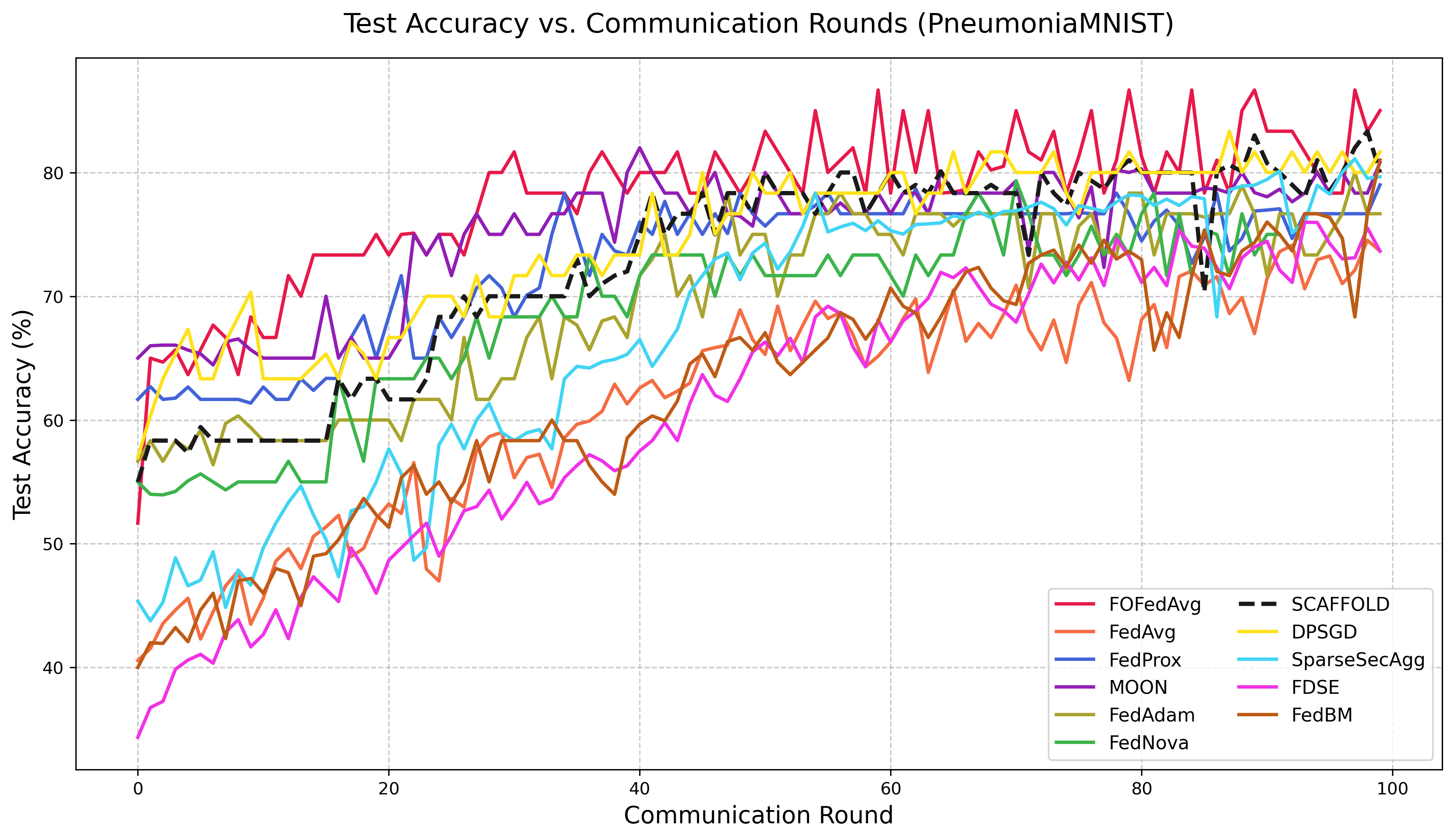}
\caption{Test accuracy comparison of FL algorithms on the PneumoniaMNIST dataset.}
\label{fig:PneumoniaMNIST}
\end{figure}

\begin{figure}[H]
\centering
\includegraphics[height=5cm,width=7cm]{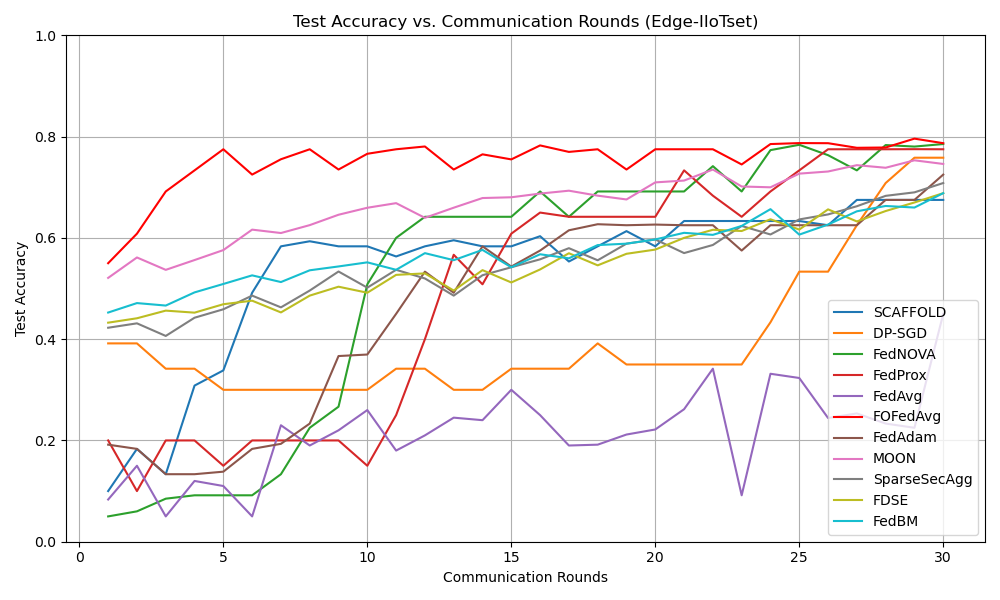}
\caption{Test accuracy comparison of FL algorithms on the Edge-IIoTset dataset.}
\label{fig:Edge}
\end{figure}

\noindent Figures~\ref{fig:Edge} and \ref{fig:PneumoniaMNIST} compare the test accuracy of FL algorithms on the Edge-IIoTset and PneumoniaMNIST datasets, respectively. Across both tasks, FOFedAvg consistently attains the highest accuracy and converges rapidly, highlighting the benefit of memory-aware updates in challenging non-IID regimes. On Edge-IIoTset, MOON, FedAdam, and SparseSecAgg form a competitive second tier, clearly outperforming the classical FedAvg baseline. On PneumoniaMNIST, FOFedAvg again leads, closely followed by SCAFFOLD and FDSE, while FedAvg, FedProx, FedAdam, FedNova, and FedBM cluster in the middle range and SparseSecAgg typically lags behind, underscoring a consistent trade-off between robustness, communication efficiency, and final accuracy across both medical and industrial IoT workloads.

\begin{figure}[H]
\centering
\includegraphics[height=5cm,width=9cm]{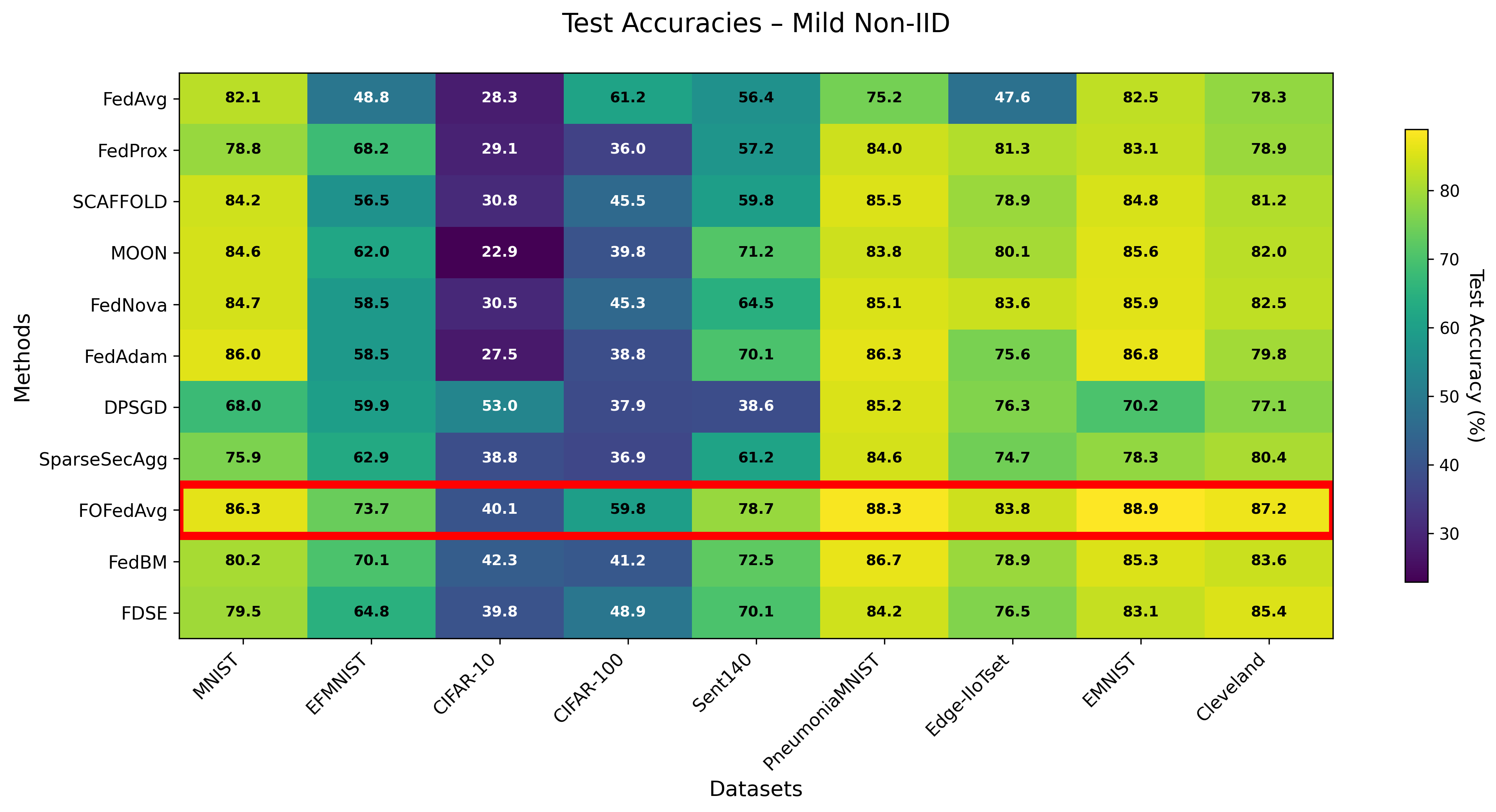}
\caption{Test accuracies of federated learning methods under mild non-IID conditions.}
\label{fig:mild_non_iid}
\end{figure}

As shown in Figure~\ref{fig:mild_non_iid}, in the mild non-IID scenario, most federated learning methods achieve relatively high and stable test accuracies across all nine datasets. FOFedAvg consistently ranks among the top-performing approaches, reaching 86.3\% on average and outperforming baselines such as FedAvg and FedProx by 4–6 percentage points in many cases. Communication-efficient or regularized methods such as FedNova, MOON, and FedAdam also perform well. SparseSecAgg shows moderate degradation, reflecting the cost of aggressive sparsification and secure aggregation, while the overall accuracy variation between methods stays limited, indicating good robustness when data heterogeneity is mild.

\begin{figure}[t]
\centering
\includegraphics[height=4cm,width=4cm]{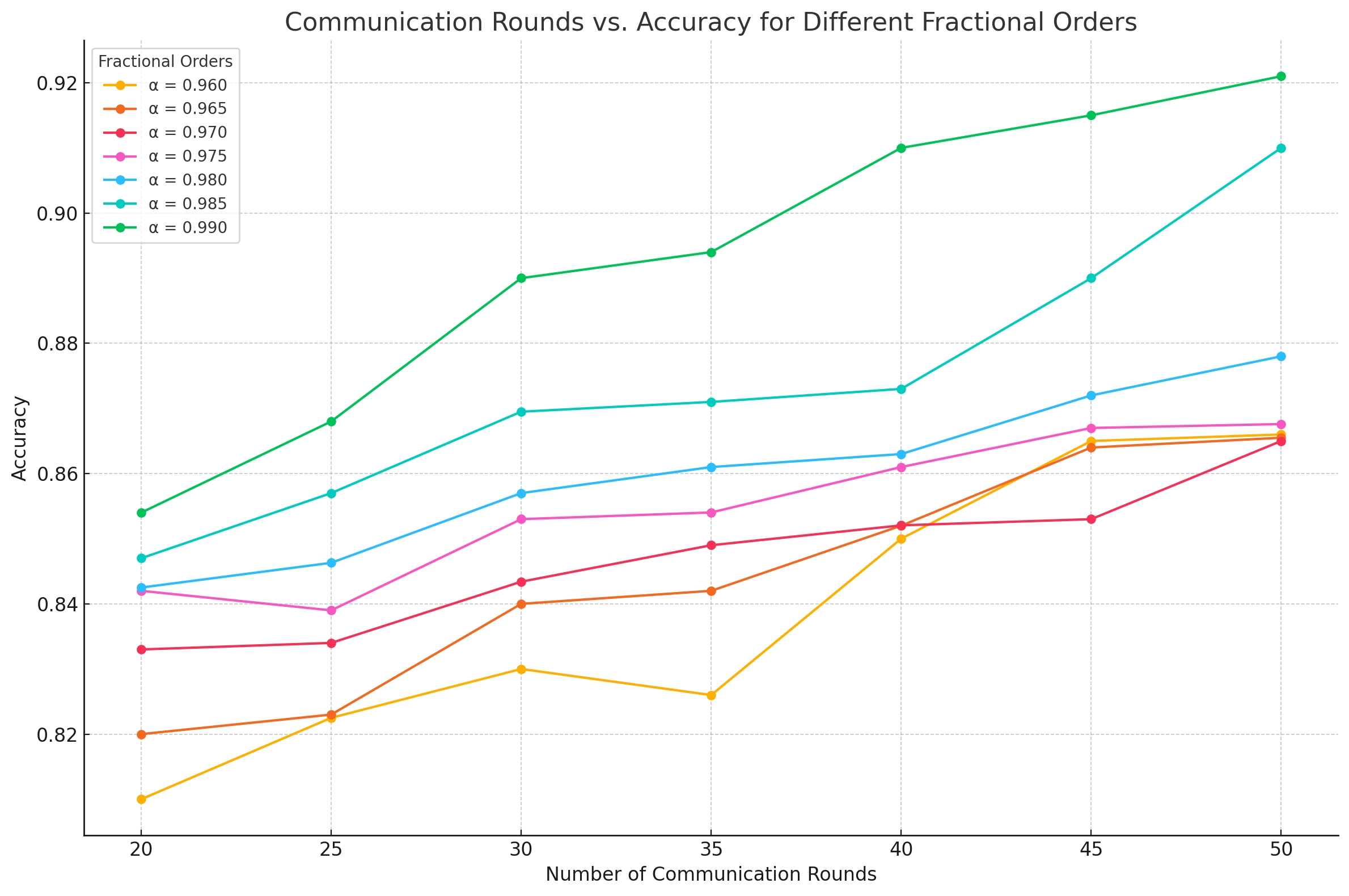}
\includegraphics[height=4cm,width=4cm]{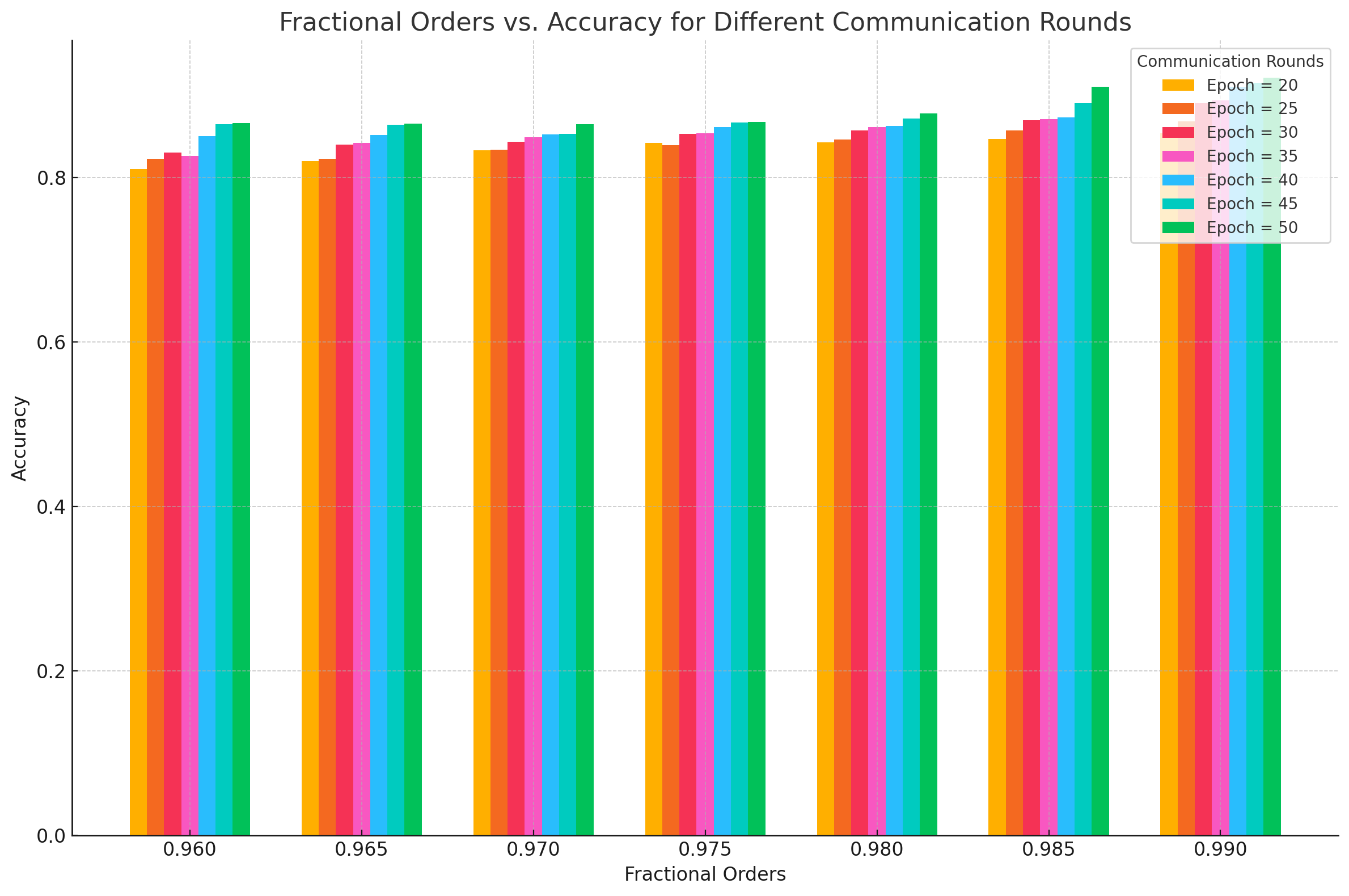}
\caption{Test accuracy visualization for the MNIST dataset ($B = 50$, $E = 20$) across fractional orders and communication rounds in a non-IID setting.}
\label{mn2}
\end{figure}

Figure~\ref{mn2} illustrates the relationship between fractional orders ($\alpha$), communication rounds, and test accuracy in a non-IID federated learning setting. The line plot on the left shows that as fractional orders increase, accuracy consistently improves, with a particularly notable jump for $\alpha > 0.97$. Additionally, higher communication rounds further boost accuracy, highlighting the combined benefits of both factors. The bar plot on the right reinforces this trend, showing that higher communication rounds consistently result in better performance across all fractional orders. These results demonstrate the synergistic effect of tuning fractional orders and increasing communication rounds in achieving superior accuracy in federated learning, especially under non-IID conditions.

\begin{figure}[H]
\centering
\includegraphics[height=4.5cm,width=4.5cm]{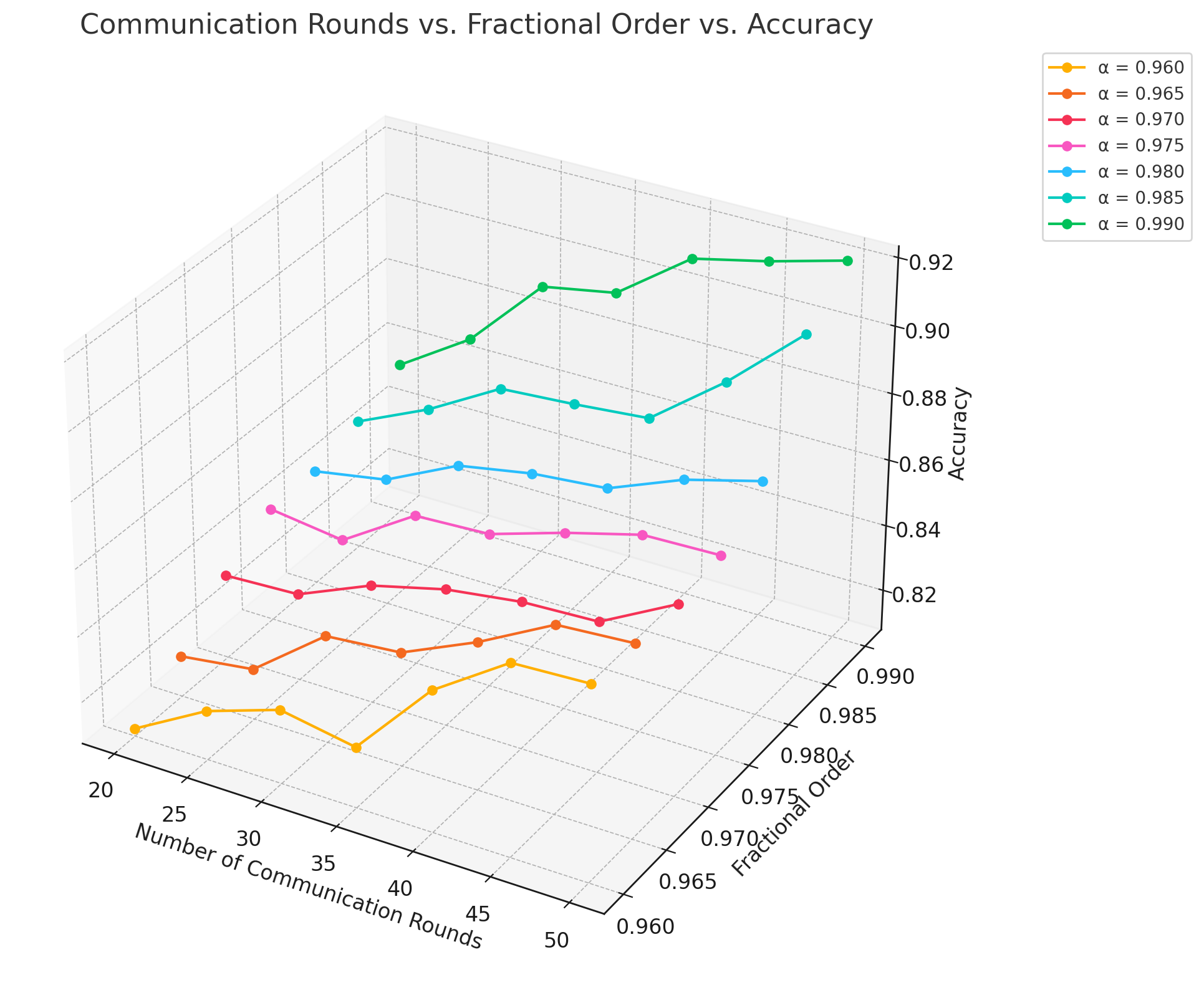}
\includegraphics[height=4cm,width=4cm]{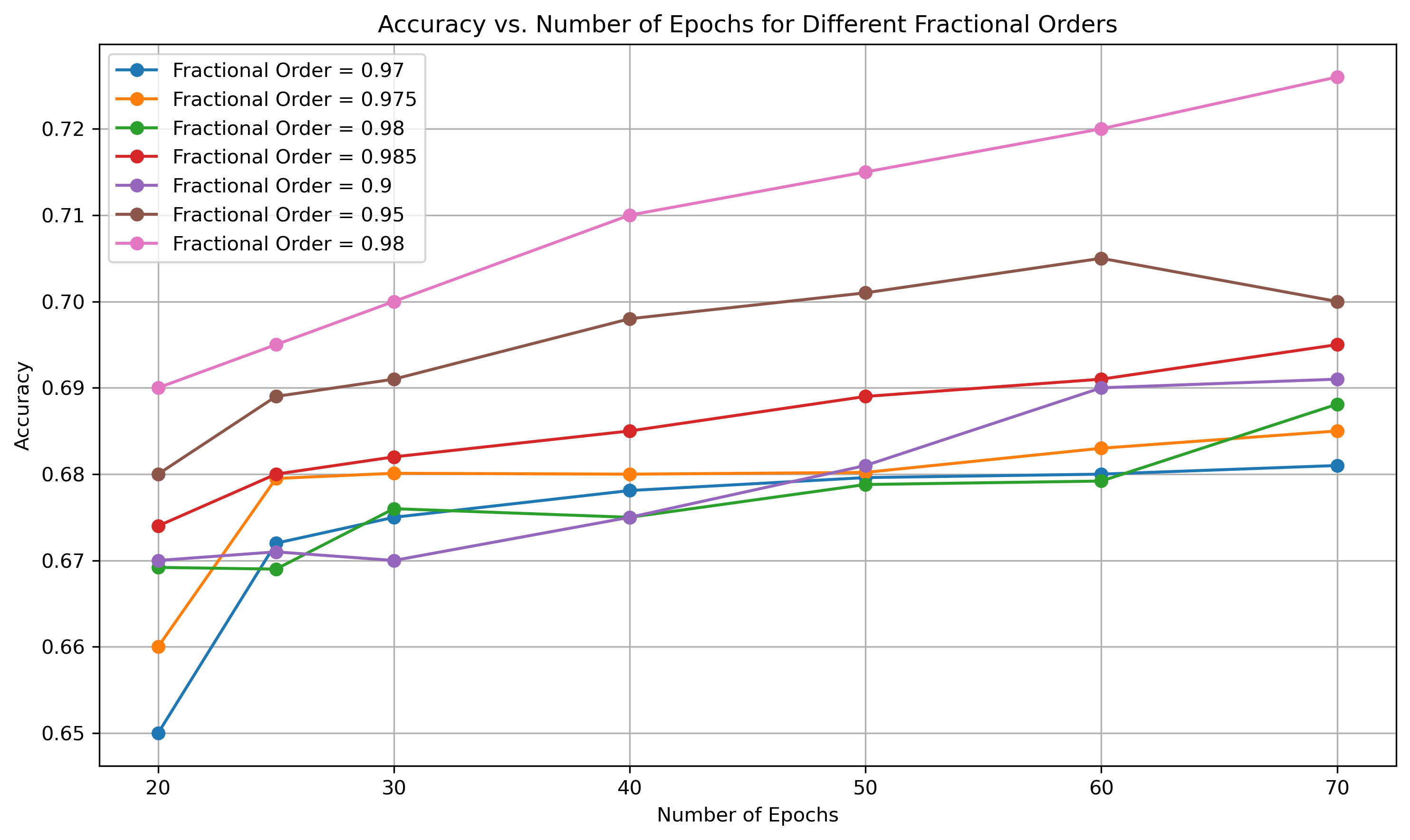}
\caption{Left: 3D line visualization of test accuracy for MNIST ($B = 50$, $E = 20$) across fractional orders and communication rounds in a non-IID setting. Right: test accuracy for CIFAR-10 across fractional orders and communication rounds in an IID setting.}
\label{mnistg}
\end{figure}

Figure~\ref{mnistg} (left) provides a 3D line visualization illustrating the relationship between fractional orders ($\alpha$), communication rounds, and accuracy in non-IID federated learning settings using the FOFedAvg algorithm. The left plot shows that increasing fractional orders and communication rounds leads to consistent improvements in accuracy, with the most significant gains observed for $\alpha > 0.97$ and higher communication rounds. Fractional-order dynamics effectively balance exploration and exploitation, enabling FOFedAvg to converge faster even in heterogeneous data environments. This highlights the synergistic effect of fractional-order optimization and communication frequency in achieving better performance in non-IID scenarios. The choice of $E = 20$ was inspiblack by \cite{mcmahan2017communication}.

Figure~\ref{mnistg} (right) highlights the effect of fractional orders ($\alpha$) and communication rounds on test accuracy for the CIFAR-10 dataset in an IID setting using a CNN architecture. The line plot on the right demonstrates a clear trend where accuracy consistently improves as fractional orders increase, particularly for higher communication rounds. For instance, at $\alpha = 0.99$, the accuracy is significantly higher compablack to lower fractional orders such as $\alpha = 0.96$ or $\alpha = 0.97$. Additionally, higher communication rounds, such as 70 rounds, amplify the benefits of higher fractional orders, leading to further improvements in accuracy. This trend highlights the synergistic effect of fractional-order dynamics and communication frequency in federated learning, where both factors contribute to achieving better performance.

\begin{figure}[H]
\centering
\includegraphics[height=4cm,width=4cm]{3Dlineplotmnistnoniidcnn.png}
\includegraphics[height=4cm,width=4cm]{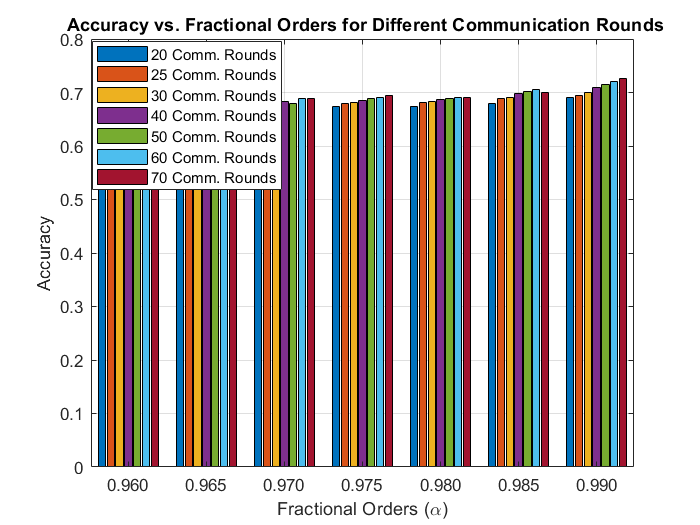}
\caption{3D and 2D visualizations of test accuracy trends on CIFAR-10 across fractional orders and communication rounds.}
\label{cifar10mmm}
\end{figure}

Figure~\ref{cifar10mmm} presents 3D and 2D visualizations of test accuracy trends for CIFAR-10 across fractional orders ($\alpha$) and communication rounds, highlighting the impact of these parameters on model performance. The left 3D plot illustrates the combined effect of fractional orders and communication rounds on test accuracy, while the right 2D bar plot provides a comparative view of accuracy for specific communication rounds across different fractional orders.

The 3D surface plot on the left demonstrates a clear trend: test accuracy improves with both increasing fractional orders and communication rounds. Fractional orders higher than $\alpha = 0.97$ consistently yield better accuracy, particularly when paiblack with a higher number of communication rounds (e.g., 70 rounds). The surface rises more steeply for higher fractional orders, reflecting their ability to enhance the learning process by balancing local updates and global aggregation more effectively. This visualization underscores the synergistic relationship between fractional orders and communication rounds in achieving superior model performance.

The 2D bar plot on the right complements the 3D visualization by showing accuracy trends for specific communication rounds. For all communication rounds, higher fractional orders ($\alpha = 0.985$ and $\alpha = 0.990$) result in the highest accuracy, confirming their effectiveness across various communication budgets. The differences between fractional orders become more pronounced as the communication rounds increase, further emphasizing the role of fractional orders in improving generalization.

Figure~\ref{cifar10mmm} highlights the importance of tuning fractional orders and communication rounds to optimize accuracy in federated learning. The combined visualizations show that fractional orders greater than $\alpha = 0.97$, coupled with sufficient communication rounds, are crucial for achieving reliable and consistent performance on non-trivial datasets like CIFAR-10. These results validate the potential of fractional-order optimization in blackucing communication costs while maintaining high accuracy, making FOFedAvg a practical solution for federated learning tasks. 

\textcolor{black}{For a comprehensive evaluation of the Fractional-Order Federated Averaging (FOFedAvg) algorithm, including detailed experimental settings, datasets, model architectures, hyperparameters, statistical analyses, and performance comparisons across various configurations and datasets, refer to Appendix B.}


\section{Conclusion and Future Work}\label{Sec:Conclusion}
\textcolor{black}{ In this work, we introduced Fractional-Order Federated Averaging (FOFedAvg), a novel enhancement to the integer-order FedAvg algorithm that leverages Fractional-Order Stochastic Gradient Descent (FOSGD) to address key challenges in federated learning. By incorporating fractional derivatives, FOFedAvg effectively captures long-range dependencies and deeper historical information, leading to improved convergence rates, enhanced communication efficiency, and greater robustness against non-IID data distributions.  Our extensive empirical evaluations on a diverse set of benchmark datasets, including MNIST, EMNIST, FEMNIST, the Cleveland Heart Disease dataset, CIFAR-10, CIFAR-100, Sent140, PneumoniaMNIST, and Edge-IIoTset, show that FOFedAvg is competitive with, and often outperforms, a wide range of established federated learning algorithms such as FedProx, FedNova, FedAdam, SCAFFOLD, MOON, FDSE, FedBM, FedAvg, DPSGD, and SparseSecAgg, particularly in highly challenging non-IID scenarios. Furthermore, our theoretical analysis confirms FOFedAvg's convergence to stationary points under standard smoothness assumptions, providing a solid foundation for its practical applicability. Future work will explore the optimization of FOFedAvg by addressing the challenge of identifying the optimal fractional order, which is critical for maximizing FOFedAvg's performance across diverse federated learning scenarios.}

\section*{Conflict of Interest Statement}
The authors have no conflicts of interest to declare.


\section{Appendix}

\section{Theoretical Analysis and Experiments}
We present a comprehensive evaluation of Fractional-Order Federated Averaging (FOFedAvg) through theoretical analyses and additional experiments, detailed in Appendices~A and~B, respectively, to demonstrate its robustness and efficacy in federated learning settings.

\section{Appendix A: Theoretical Analysis}
\label{sec:appendix_theoretical}

This section consolidates additional theoretical analyses to support the claims and robustness of FOFedAvg, focusing on the properties of Fractional-Order Stochastic Gradient Descent (FOSGD) and its application in federated learning (FL). Each subsection addresses a specific theoretical aspect, ensuring a comprehensive understanding of FOFedAvg’s contributions without blackundancy.

\subsection{Smoothness Assumption in FOSGD}
\label{subsec:smoothness_fosgd}

Here, we provide a formal analysis, including a proposition and proof, to clarify how the standard smoothness assumption interacts with FOSGD.

The smoothness assumption is that the objective function \(f: \mathbb{R}^d \to \mathbb{R}\) is \(L\)-smooth, meaning it has a Lipschitz continuous gradient with constant \(L > 0\):
\begin{equation}
\|\nabla f(x) - \nabla f(y)\| \leq L \|x - y\| \quad \text{for all } x, y \in \mathbb{R}^d,
\label{eq:grad_lipschitz}
\end{equation}
which implies:
\begin{equation}
f(y) \leq f(x) + \nabla f(x)^\top (y - x) + \frac{L}{2} \|y - x\|^2,
\label{eq:smoothness_implication}
\end{equation}
a standard condition for ensuring convergence in SGD and related methods. In standard SGD, the update rule is:
\begin{equation}
\Theta_{t+1} = \Theta_t - \eta \nabla \ell(\Theta_t; b),
\label{eq:theta_update_simple}
\end{equation}
where \(\eta\) is the learning rate, and \(\nabla \ell(\Theta_t; b)\) is the stochastic gradient on a mini-batch \(b\). The smoothness assumption is a property of the objective function \(f\) (or its local components), not of the algorithm itself, and ensures that gradient changes are bounded, facilitating convergence analysis.

In FOFedAvg, FOSGD modifies the update rule for client \(k\) at iteration \(t\), where the effective step size is:
\begin{equation}
\alpha_t^{(k)} = \frac{\mu_0}{\sqrt{t+1} \cdot \Gamma(2-\alpha)} \left( \|\Theta_t^{(k)} - \Theta_{t-1}^{(k)}\| + \delta \right)^{1-\alpha},
\label{eq:alpha_t2}
\end{equation}
for \(0 < \alpha \le 1\). The fractional term
\[
\frac{1}{\Gamma(2-\alpha)} \left( \|\Theta_t^{(k)} - \Theta_{t-1}^{(k)}\| + \delta \right)^{1-\alpha}
\]
scales the standard gradient \(\nabla f(\Theta_t^{(k)}; b)\), introducing a dependency on past parameter changes \(\|\Theta_t^{(k)} - \Theta_{t-1}^{(k)}\|\). We emphasize that the smoothness assumption pertains to the objective function \(f\), not the optimization algorithm’s update mechanism; the update still relies on the standard gradient \(\nabla f(\Theta_t^{(k)}; b)\), which inherits the smoothness properties of \(f\).

\begin{proposition}
\label{prop:smoothness_fosgd}
Let \(f: \mathbb{R}^d \to \mathbb{R}\) be an \(L\)-smooth function, i.e., satisfying \eqref{eq:grad_lipschitz}. Assume that the stochastic gradient \(\nabla f(\Theta_t^{(k)}; b)\) is computed for a mini-batch \(b\) drawn from the local dataset \(\mathcal{P}_k\), where \(f\) is the loss function contributing to the global objective. Then, the FOSGD update operates under the same smoothness assumption: the fractional term in \(\alpha_t^{(k)}\) does not alter the smoothness properties of \(f\), but only rescales the step size.
\end{proposition}

\begin{proof}
The FOSGD update for client \(k\) at iteration \(t\) is:
\begin{equation}
\Theta_{t+1}^{(k)} = \Theta_t^{(k)} - \alpha_t^{(k)} \nabla f(\Theta_t^{(k)}; b),
\label{eq:theta_update}
\end{equation}
where the effective step size \(\alpha_t^{(k)}\) is defined in \eqref{eq:alpha_t2}. The gradient \(\nabla f(\Theta_t^{(k)}; b)\) is the standard stochastic gradient of the loss function, computed at \(\Theta_t^{(k)}\) for mini-batch \(b\). Since
\[
f(\Theta) = \sum_{k=1}^K \frac{n_k}{n} F_k(\Theta), 
\quad
F_k(\Theta) = \frac{1}{n_k} \sum_{i \in \mathcal{P}_k} \ell(\Theta; x_i, y_i),
\]
the smoothness of \(f\) implies that each \(F_k\) (and the corresponding mini-batch losses) is also \(L\)-smooth (assuming a common Lipschitz constant for simplicity). Thus:
\begin{equation}
\|\nabla f(\Theta; b) - \nabla f(\Theta'; b)\| \leq L \|\Theta - \Theta'\|,
\label{eq:lipschitz}
\end{equation}
for any \(\Theta, \Theta' \in \mathbb{R}^d\) and mini-batch \(b\), since the stochastic gradient is an average of individual gradients, each satisfying the smoothness condition.

The fractional term in \eqref{eq:alpha_t2} is a scalar function of the parameter difference \(\Delta_t = \|\Theta_t^{(k)} - \Theta_{t-1}^{(k)}\|\). For \(0 < \alpha \leq 1\):
\begin{itemize}
\item \(\Gamma(2-\alpha)\) is a positive constant, finite for \(\alpha \in (0,1]\);
\item \(\left( \Delta_t + \delta \right)^{1-\alpha}\) is positive and continuous, as \(\delta > 0\) ensures \(\Delta_t + \delta > 0\) and \(1-\alpha \ge 0\).
\end{itemize}
Thus, \(\alpha_t^{(k)}\) is a positive, bounded scalar (with \(\alpha_t^{(k)} \le 2/L\) enforced by an appropriate choice of \(\mu_0\) and \(\delta\)), which scales the gradient but does not modify its functional form or smoothness properties. The smoothness condition \eqref{eq:grad_lipschitz} remains unchanged, as it depends solely on \(f\)’s gradient, not on the step size or the algorithm.

To connect this with the fractional interpretation, consider the approximation:
\begin{equation}
D_t^\alpha f(\Theta_t^{(k)}) \approx \frac{\nabla f(\Theta_t^{(k)}; b)}{\Gamma(2-\alpha)} \left( \|\Theta_t^{(k)} - \Theta_{t-1}^{(k)}\| + \delta \right)^{1-\alpha}.
\label{eq:frac_deriv_approx}
\end{equation}
Here, \(\nabla f(\Theta_t^{(k)}; b)\) is the standard gradient, and the fractional scaling approximates the nonlocal behavior of the Caputo derivative. While the true fractional derivative \(D_t^\alpha f\) involves an integral over past states, the FOSGD implementation uses the current gradient scaled by a history-dependent factor. This scaling does not affect the smoothness of \(f\), as it operates on the output of \(\nabla f\), not on its definition. Therefore, the smoothness assumption holds for FOSGD, and the convergence proof in Theorem~1 remains valid for the fractional updates.
\end{proof}

In FL, the objective \(f(\Theta) = \sum_{k=1}^K \frac{n_k}{n} F_k(\Theta)\) aggregates local objectives, each assumed \(L\)-smooth, as is standard in FL literature. The smoothness assumption is critical for ensuring that local gradients \(\nabla F_k(\Theta)\) are well-behaved, even in non-IID settings. FOFedAvg’s use of FOSGD does not alter this property, as shown in Proposition~\ref{prop:smoothness_fosgd}. The fractional updates enhance stability in non-IID scenarios by incorporating memory effects, but these effects operate within the framework of the original smoothness assumption.

{\color{black}
It is important to emphasize that the fractional scaling in $\alpha_t^{(k)}$ affects only the \emph{magnitude} of the update step, not the direction of the expected gradient. In particular, under the usual unbiasedness assumption for the stochastic gradient, we still have
\[
\mathbb{E}_b\bigl[\nabla \ell(\Theta_t^{(k)}; b)\bigr] = \nabla F_k(\Theta_t^{(k)}),
\]
so the structural deviation between a local client and the global objective remains
\[
\nabla F_k(\Theta_t^{(k)}) - \nabla f(\Theta_t^{(k)}),
\]
exactly as in standard local-SGD. The fractional term modifies how aggressively this direction is followed over time (via a history-dependent step size), but it does not by itself remove or re-weight the underlying non-IID bias. This clarifies that our use of fractional calculus builds a memory-aware step-size schedule on top of the standard smoothness framework rather than changing the fundamental gradient structure.
}

\subsection{Drawbacks and Bias Analysis of FOSGD}
\label{subsec:fosgd_drawbacks}

We provide an analysis to identify FOSGD’s limitations, evaluate the potential for historical gradients to introduce bias, and discuss their implications for federated learning (FL), supported by theoretical derivations and references to fractional calculus and FL literature. The FOSGD algorithm modifies the standard SGD update \eqref{eq:theta_update_simple} to incorporate a fractional-order gradient approximation. Compablack to SGD, FOSGD introduces additional hyperparameters: the fractional order \(\alpha\), initial learning rate \(\mu_0\), and regularization constant \(\delta\). Tuning these is critical: small \(\alpha\) amplifies historical effects and can slow adaptation, while large \(\alpha\) approaches standard SGD but loses most of the memory benefits. Despite these drawbacks, FOSGD’s memory effects and long-range dependencies (Section~\ref{subsec:memory_analysis}) provide advantages in non-IID FL settings, stabilizing updates and improving robustness to data heterogeneity, as discussed below.

\subsubsection{Potential Bias from Historical Gradients}

This section revises the analysis of potential bias introduced by historical gradients in Fractional-Order Stochastic Gradient Descent (FOSGD) within the Fractional-Order Federated Averaging (FOFedAvg) algorithm, ensuring a rigorous and clear evaluation for federated learning (FL). FOSGD modifies the standard SGD update \eqref{eq:theta_update_simple} with the fractional-order update \eqref{eq:theta_update}. FOSGD’s reliance on historical gradients via \(\|\Theta_t^{(k)} - \Theta_{t-1}^{(k)}\|\) introduces memory effects that enhance stability in non-IID FL settings (Section~\ref{subsec:memory_analysis}), but raises questions about potential bias.

In this context, “bias” refers to the deviation of the expected update direction from the true global gradient of the objective
\[
f(\Theta) = \sum_{k=1}^K \frac{n_k}{n} F_k(\Theta).
\]
In FL, non-IID data causes local gradients \(\nabla F_k(\Theta)\) to differ from \(\nabla f(\Theta)\), introducing a structural bias that can lead to client drift. We analyze whether FOSGD’s historical gradients exacerbate or alleviate this bias.

The FOSGD update’s effective direction for client \(k\) at time \(t\) is \(\alpha_t^{(k)} \nabla \ell(\Theta_t^{(k)}; b_t)\), where \(\alpha_t^{(k)}\) is the history-dependent step size and \(\nabla \ell(\Theta_t^{(k)}; b_t)\) is the mini-batch gradient. The stochastic gradient is an unbiased estimator of the local gradient:
\begin{equation}
\mathbb{E}_{b_t} \left[ \nabla \ell(\Theta_t^{(k)}; b_t) \right] = \nabla F_k(\Theta_t^{(k)}).
\label{eq:unbiased_estimator_fosgd}
\end{equation}
Define the bias of the expected update direction relative to the global gradient as
\begin{equation}
\text{Bias}_t^{(k)} 
= \mathbb{E}_{b_t}\bigl[\alpha_t^{(k)} \nabla \ell(\Theta_t^{(k)}; b_t)\bigr] - \nabla f(\Theta_t^{(k)}).
\label{eq:bias_definition_correct}
\end{equation}
Using \eqref{eq:unbiased_estimator_fosgd}, we obtain
\begin{equation}
\begin{aligned}
\text{Bias}_t^{(k)}
&= \alpha_t^{(k)} \nabla F_k(\Theta_t^{(k)}) - \nabla f(\Theta_t^{(k)}) \\
&= \alpha_t^{(k)}\bigl(\nabla F_k(\Theta_t^{(k)}) - \nabla f(\Theta_t^{(k)})\bigr)
  + \bigl(\alpha_t^{(k)} - 1\bigr)\nabla f(\Theta_t^{(k)}).
\end{aligned}
\label{eq:bias_decomposition}
\end{equation}
The first term in \eqref{eq:bias_decomposition} is the familiar non-IID bias \(\nabla F_k - \nabla f\), scaled by \(\alpha_t^{(k)}\); the second term comes from the mismatch between the scalar \(\alpha_t^{(k)}\) and \(1\). As the learning-rate schedule \(\mu_t = \frac{\mu_0}{\sqrt{t+1}}\) decays and \(\alpha_t^{(k)}\) remains bounded, the factor \(|\alpha_t^{(k)} - 1|\) becomes small in the regime where Theorem~1 applies, and the dominant structural bias remains the usual non-IID gap \(\nabla F_k - \nabla f\).

Historical gradients influence \(\alpha_t^{(k)}\) through the trajectory term:
\begin{equation}
\|\Theta_t^{(k)} - \Theta_{t-1}^{(k)}\| = \bigl\| -\alpha_{t-1}^{(k)} \nabla \ell(\Theta_{t-1}^{(k)}; b_{t-1}) \bigr\|,
\label{eq:parameter_diff}
\end{equation}
so that \(\alpha_t^{(k)}\) depends recursively on past gradients. Conceptually, as discussed in Appendix~\ref{subsec:memory_analysis}, the fractional update can be interpreted as implementing a power-law–weighted moving average of past gradients,
\begin{equation}
D_t^\alpha \ell(\Theta_t^{(k)}) \approx \sum_{j=0}^t w_j \nabla \ell(\Theta_{t-j}^{(k)}; b_{t-j}),
\qquad
w_j \propto j^{-(1+\alpha)},
\label{eq:frac_deriv_approx_powerlaw}
\end{equation}
so that temporally persistent components in the gradient sequence are reinforced while rapidly fluctuating components are smoothed.

\begin{proposition}
\label{prop:bias_fosgd}
Let \(f: \mathbb{R}^d \to \mathbb{R}\) be an \(L\)-smooth, lower-bounded (potentially non-convex) objective, and consider the FOSGD update applied to a client’s local objective \(F_k\) in FL with step sizes \(\alpha_t^{(k)}\) satisfying the conditions of Theorem~1. Then:
\begin{itemize}
    \item The fractional history does not introduce additional structural non-IID bias beyond the usual term \(\nabla F_k(\Theta_t^{(k)}) - \nabla f(\Theta_t^{(k)})\); it merely rescales it via \(\alpha_t^{(k)}\) and adds a vanishing global term \((\alpha_t^{(k)} - 1)\nabla f(\Theta_t^{(k)})\).
    \item Under standard bounded-variance assumptions and a decaying learning-rate schedule \(\mu_t = \frac{\mu_0}{\sqrt{t+1}}\), the power-law temporal averaging implied by \eqref{eq:frac_deriv_approx_powerlaw} can blackuce the variance and drift of the effective update direction, without preventing convergence to stationary points.
\end{itemize}
\end{proposition}

\begin{proof}
The decomposition \eqref{eq:bias_decomposition} shows explicitly that the non-IID structural deviation \(\nabla F_k - \nabla f\) remains the same as in standard local-SGD and is simply scaled by \(\alpha_t^{(k)}\). The additional term \((\alpha_t^{(k)} - 1)\nabla f(\Theta_t^{(k)})\) is controlled by keeping \(\alpha_t^{(k)}\) bounded and using the decaying schedule \(\mu_t = \frac{\mu_0}{\sqrt{t+1}}\), which implies \(\alpha_t^{(k)} \to 0\) asymptotically. Thus, this term does not accumulate unboundedly over time.
\\
The sufficient decrease condition
\[
f(\Theta_{t+2}) \le f(\Theta_{t+1}) - \kappa_t \|\nabla f(\Theta_{t+1})\|^2,
\quad
\kappa_t = \alpha_t \bigl(1 - \tfrac{L}{2}\alpha_t\bigr),
\]
holds whenever \(\alpha_t \le 2/L\). This condition is enforced in FOSGD by choosing \(\mu_t\) and \(\delta\) so that \(\alpha_t^{(k)}\) remains bounded for all \(t\). Under the usual bounded-variance assumptions on the stochastic gradient noise, the power-law weights in \eqref{eq:frac_deriv_approx_powerlaw} act as a temporal smoother, averaging out high-frequency noise and blackucing drift, while the structural non-IID term \(\nabla F_k - \nabla f\) persists but is not amplified in an unbounded way. The convergence result of Theorem~\ref{thm:main} then implies that
\[
\liminf_{t \to \infty} \|\nabla f(\Theta_t)\| = 0,
\]
so every limit point is stationary, and the presence of fractional history does not prevent convergence.
\end{proof}

{\color{black}
\subsubsection{Outlier Gradients vs Underrepresented Data: Limitations and Assumptions}
\label{subsec:outlier_underrepresented_limitations}

There is an important question: can FOFedAvg distinguish between gradients arising from underrepresented (minority) data and gradients arising from outliers or corrupted samples? In the current design, the answer is \emph{no}. Algorithmically, both types of gradients enter the update solely through the stochastic mini-batch gradient and the trajectory term $\|\Theta_t^{(k)} - \Theta_{t-1}^{(k)}\|$ that defines $\alpha_t^{(k)}$.
\\
To make this explicit, let $g_t^{(k)} = \nabla \ell(\Theta_t^{(k)}; b_t)$ denote the mini-batch gradient for client $k$ at local step $t$. We may decompose it as
\[
g_t^{(k)} = \bar{G}_t + \Delta_k(\Theta_t^{(k)}) + \xi_t^{(k)},
\]
where $\bar{G}_t$ is the global gradient, $\Delta_k$ captures the systematic non-IID bias of client $k$, and $\xi_t^{(k)}$ collects stochastic noise, including outliers. As discussed in Appendix~\ref{subsec:memory_analysis}, the fractional update implicitly uses a power-law weighted average of past gradients,
\[
D_t^\alpha \ell(\Theta_t^{(k)}) \;\approx\; \sum_{j=0}^t w_j\, g_{t-j}^{(k)}, \qquad w_j \propto j^{-(1+\alpha)},
\]
so that temporally persistent components in $\bar{G}_t + \Delta_k(\Theta_t^{(k)})$ are reinforced, while uncorrelated, sporadic components in $\xi_t^{(k)}$ are averaged out under the usual bounded-variance assumptions.
\\
Under these assumptions, gradients from underrepresented but consistently observed data can benefit from the long-range memory, because their contribution persists over time, whereas isolated outliers are damped. However, this mechanism is \emph{agnostic} to the semantic origin of the gradients: if harmful or adversarial gradients are themselves temporally persistent and aligned with a minority direction, the power-law memory can propagate their influence just as it does for genuinely informative minority samples. Therefore, our theoretical claims should be interpreted as follows:
\begin{itemize}
    \item FOFedAvg does not introduce additional structural bias beyond local-SGD and can blackuce the variance and drift of updates by heavy-tailed temporal averaging under standard stochastic assumptions (bounded-variance stochastic gradients and $L$-smooth local objectives).
    \item It does not perform explicit outlier detection or semantic discrimination between “good” and “bad” gradients; it is a temporal filter whose effect depends on the statistics of the gradient sequence.
\end{itemize}
Designing a robust fractional-order FL scheme that explicitly separates outliers from minority gradients (e.g., via robust aggregation, adaptive clipping, or influence functions in fractional space) is an interesting direction for future work that we now highlight explicitly.
}

\subsection{Memory Effects and Long-Range Dependencies}
\label{subsec:memory_analysis}

The Fractional-Order Federated Averaging (FOFedAvg) algorithm leverages fractional-order stochastic gradient descent (FOSGD) to address challenges in federated learning (FL), particularly with non-independent and identically distributed (non-IID) data. A key claim of this work is that FOFedAvg effectively captures \emph{memory effects} and \emph{long-range dependencies}, enhancing convergence and stability in heterogeneous settings. We provide a formal theoretical discussion to substantiate these claims, focusing on the mathematical properties of fractional-order gradients and their implications for optimization in FL.
\\
In the context of optimization, \emph{memory effects} refer to the ability of an algorithm to incorporate information from past iterates or gradients into its current update, enabling it to retain historical trends. \emph{Long-range dependencies} describe the influence of distant past states (e.g., gradients from earlier iterations) on the current optimization step, allowing the algorithm to model temporal correlations over extended periods. In FL, these properties are critical for stabilizing training across clients with non-IID data, where local gradients may diverge significantly from the global objective.
\\
Standard stochastic gradient descent (SGD), as used in FedAvg, updates parameters as in \eqref{eq:theta_update_simple}, where \(\eta\) is the learning rate and \(\nabla \ell(\Theta_t; b)\) is the gradient on a mini-batch \(b\). This update depends solely on the current gradient, lacking memory of past states. In contrast, FOSGD introduces a fractional-order gradient update:
\begin{equation}
\begin{split}
\Theta_{t+1}^{(k)} =\ & \Theta_t^{(k)} \\
& - \frac{\mu_t}{\Gamma(2-\alpha)} \left( \|\Theta_t^{(k)} - \Theta_{t-1}^{(k)}\| + \delta \right)^{1-\alpha} \nabla \ell(\Theta_t^{(k)}; b_t),
\end{split}
\label{eq:theta_update_combined}
\end{equation}
where \(0 < \alpha \leq 1\), \(\mu_t = \frac{\mu_0}{\sqrt{t+1}}\) is the learning rate, \(\delta > 0\) is a regularization constant, and \(\Gamma(\cdot)\) is the Gamma function. The fractional term \(\left( \|\Theta_t^{(k)} - \Theta_{t-1}^{(k)}\| + \delta \right)^{1-\alpha}\) introduces a dependence on past parameter changes, which we analyze below to clarify the notion of memory retention and long-range dependency modeling.
\\
The fractional-order gradient in FOSGD is motivated by the Caputo fractional derivative \(D_t^\alpha \ell(\Theta_t)\), which, unlike integer-order derivatives, is \emph{nonlocal}. For a scalar function \(f(t)\), the Caputo fractional derivative is
\begin{equation}
{}^C D_t^\alpha f(t) = \frac{1}{\Gamma(n-\alpha)} \int_0^t \frac{f^{(n)}(\tau)}{(t - \tau)^{\alpha - n + 1}} \, d\tau, \quad n-1 < \alpha < n.
\label{eq:caputo_derivative_memory}
\end{equation}
For \(0 < \alpha < 1\), this simplifies to
\begin{equation}
{}^C D_t^\alpha f(t) = \frac{1}{\Gamma(1-\alpha)} \int_0^t \frac{f'(\tau)}{(t - \tau)^\alpha} \, d\tau.
\label{eq:caputo_simplified_memory}
\end{equation}
The integral form shows that the fractional derivative at time \(t\) depends on the function’s first derivative \(f'(\tau)\) for all \(\tau \in [0, t]\), weighted by the kernel \((t - \tau)^{-\alpha}\). This kernel assigns greater weight to recent states (small \(t - \tau\)) but retains contributions from earlier states, encoding a memory effect.

In FOSGD, the fractional gradient approximation is derived using a Taylor series expansion:
\begin{equation}
D_t^\alpha \ell(\Theta_t) = \sum_{i=1}^\infty \frac{\ell^{(i)}(\Theta_0)}{\Gamma(i+1-\alpha)} (t - t_0)^{i-\alpha}, \quad 0 < \alpha < 1.
\label{eq:taylor_approx_memory}
\end{equation}
For practical implementation, FOFedAvg does \emph{not} evaluate the full series or the full history integral. Instead, it uses the local surrogate
\begin{equation}
D_t^\alpha \ell(\Theta_t^{(k)}) \approx \frac{\nabla \ell(\Theta_t^{(k)}; b_t)}{\Gamma(2-\alpha)} \left( \|\Theta_t^{(k)} - \Theta_{t-1}^{(k)}\| + \delta \right)^{1-\alpha},
\label{eq:frac_deriv_approx_memory}
\end{equation}
which compresses the effect of past iterates into the single trajectory term \(\|\Theta_t^{(k)} - \Theta_{t-1}^{(k)}\|\).

The term \(\|\Theta_t^{(k)} - \Theta_{t-1}^{(k)}\|\) represents the magnitude of the previous update, which is itself a function of past gradients:
\begin{equation}
\Theta_t^{(k)} - \Theta_{t-1}^{(k)} 
= - \frac{\mu_{t-1}}{\Gamma(2-\alpha)} 
  \left( \|\Theta_{t-1}^{(k)} - \Theta_{t-2}^{(k)}\| + \delta \right)^{1-\alpha} 
  \mathcal{Q}.
\label{eq:theta_diff_memory}
\end{equation}
where $ = \nabla \ell(\Theta_{t-1}^{(k)}; b_{t-1})$. Recursively, \(\Theta_t^{(k)} - \Theta_{t-1}^{(k)}\) depends on \(\Theta_{t-1}^{(k)} - \Theta_{t-2}^{(k)}\), and so on, forming a chain of historical dependencies. The fractional exponent \(1-\alpha\) modulates the influence of these past updates, with smaller \(\alpha\) increasing the weight of earlier iterates. This recursive structure ensures that the update at time \(t+1\) indirectly incorporates information from multiple prior gradients, formalizing a form of memory effect even though only a single previous iterate is stoblack explicitly.

To quantify this, consider the effective step size in FOSGD:
\begin{equation}
\alpha_t^{(k)} = \frac{\mu_t}{\Gamma(2-\alpha)} \left( \|\Theta_t^{(k)} - \Theta_{t-1}^{(k)}\| + \delta \right)^{1-\alpha}.
\end{equation}
The factor \(\left( \|\Theta_t^{(k)} - \Theta_{t-1}^{(k)}\| + \delta \right)^{1-\alpha}\) scales the gradient based on the magnitude of recent parameter changes, which themselves encode past gradient information. For \(0 < \alpha < 1\), the exponent \(1-\alpha > 0\) amplifies the impact of small parameter changes, ensuring that historical trends (e.g., consistent gradient directions) persist in the update. As \(\alpha \to 1\), the term approaches $1$, and FOSGD blackuces to a (decaying-step-size) first-order method, losing most of its additional memory effect. Thus, the fractional order \(\alpha \in (0,1)\) directly controls the strength of memory retention.

Long-range dependencies can be understood by analogy with the discrete Grünwald–Letnikov derivative:
\begin{equation}
\Delta_h^\alpha f(t) = \lim_{h \to 0} \frac{1}{h^\alpha} \sum_{j=0}^{\lfloor t/h \rfloor} (-1)^j \binom{\alpha}{j} f(t - jh),
\label{eq:grunwald_letnikov_memory}
\end{equation}
where the binomial coefficients decay as \(j\) increases, but at a slower-than-exponential rate. In a similar spirit, the recursive dependence in \eqref{eq:theta_diff_memory} allows the fractional update to be \emph{interpreted} as implementing a power-law–weighted moving average of past gradients:
\begin{equation}
D_t^\alpha \ell(\Theta_t^{(k)}) \;\approx\; \sum_{j=0}^{t} w_j \,\nabla \ell(\Theta_{t-j}^{(k)}; b_{t-j}), 
\quad
w_j \propto (j+1)^{-(1+\alpha)},
\label{eq:weighted_gradients_memory}
\end{equation}
where the weights \(w_j\) decay with \(j\) in a heavy-tailed fashion for \(0 < \alpha \le 1\). This power-law decay ensures that gradients from early iterations (\(j \gg 1\)) retain a non-negligible influence on the current update, modeling long-range dependencies. In contrast, standard SGD corresponds to the degenerate case \(w_0 = 1, w_j = 0\) for \(j>0\), which ignores past gradients.

These memory effects and long-range dependencies are particularly relevant in FL. In non-IID settings, local objective functions \(F_k(\Theta) = \frac{1}{n_k} \sum_{i \in \mathcal{P}_k} \ell(\Theta; x_i, y_i)\) differ across clients \(k\), leading to biased local gradients \(\nabla F_k(\Theta)\) and potential gradient drift. FOFedAvg mitigates this at the algorithmic level via the following mechanisms:

\begin{itemize}
    \item \emph{Stabilizing updates via temporal averaging.} The power-law weights $w_j \propto (j+1)^{-(1+\alpha)}$ in \eqref{eq:weighted_gradients_memory} implement a heavy-tailed temporal average of past gradients. Under the standard bounded-variance assumption on stochastic gradients, this averaging blackuces the variance of the effective update direction and makes it less likely that any single aberrant local step dominates the trajectory. In non-IID FL, this helps stabilize training by smoothing out high-frequency fluctuations induced by heterogeneous local data.
    \item \emph{Preserving long-range minority signals.} Because the weights $w_j$ decay slowly in $j$, contributions from early rounds can remain influential for many steps. When minority patterns (e.g., rare classes or rare clinical sub-populations) appear consistently over time on some clients, their gradients contribute persistently to the fractional update, rather than being overwritten immediately by recent majority data. This provides a mechanism for preserving diversity across clients under non-IID partitions.
    \item \emph{Adapting to heterogeneity without semantic classification.} The memory effect in FOFedAvg, arising from the recursive dependence of $\|\Theta_t^{(k)} - \Theta_{t-1}^{(k)}\|$ on past updates, adapts the effective step size to the local dynamics of each client. Importantly, this mechanism is \emph{semantic-agnostic}: it does not attempt to classify gradients as outliers or minority; instead, it biases the dynamics toward directions that are temporally persistent. As discussed in Subsection~\ref{subsec:outlier_underrepresented_limitations}, this can be beneficial when outliers are rare and uncorrelated, but it is not a substitute for robust or adversarial defenses.
\end{itemize}

The nonlocal dependence on $\|\Theta_t^{(k)} - \Theta_{t-1}^{(k)}\|$ also provides a simple mechanism to handle intermittent communication. When a client performs many local steps between two server synchronizations, the cumulative change in its parameters tends to increase, which in turn modulates the trajectory factor $(\|\Theta_t^{(k)} - \Theta_{t-1}^{(k)}\| + \delta)^{1-\alpha}$ in \eqref{eq:theta_update_combined}. Together with the decaying learning rate $\mu_t = \mu_0 / \sqrt{t+1}$, this induces an implicit adaptation of the effective step size on such clients, preventing excessively large global jumps caused by long unsynchronized local trajectories. In this sense, the nonlocal “memory” of FOFedAvg differs from plain local-SGD, which uses a fixed learning rate regardless of how many local steps have been taken since the last aggregation.

\begin{remark}
The fractional order \(\alpha \in (0,1]\) provides a flexible mechanism to balance memory effects and computational simplicity. For \(\alpha = 1\), FOFedAvg blackuces to a standard first-order method (FedAvg with a decaying step size), losing most of the additional memory effects. The range \(0 < \alpha < 1\) is where FOFedAvg’s ability to model long-range temporal dependencies through heavy-tailed temporal averaging is most pronounced, making it particularly suited for non-IID FL scenarios.
\end{remark}

\subsection{Impact of the Fractional Order \texorpdfstring{$\alpha$}{alpha}}
\label{subsec:higher_alpha}

In this subsection, we analyze how the fractional order $\alpha$ (restricted to $0 < \alpha \leq 1$) influences the behavior of FOSGD and, consequently, FOFedAvg, with a focus on convergence speed and stability in federated learning (FL) scenarios with varying degrees of data heterogeneity.

\subsubsection{Role of Fractional Order \texorpdfstring{$\alpha$}{alpha} in FOSGD}

FOFedAvg employs FOSGD with the local update
\begin{equation}
\begin{split}
\Theta_{t+1}^{(k)} =\ & \Theta_t^{(k)} \\
& - \frac{\mu_t}{\Gamma(2-\alpha)} 
   \left( \|\Theta_t^{(k)} - \Theta_{t-1}^{(k)}\| + \delta \right)^{1-\alpha} 
   \nabla \ell(\Theta_t^{(k)}; b_t),
\end{split}
\end{equation}
where $\mu_t = \mu_0 / \sqrt{t+1}$ is the learning rate schedule, $\delta > 0$ is a small regularization constant, and $0 < \alpha \leq 1$ is the fractional order. The scalar
\[
\alpha_t^{(k)} 
  = \frac{\mu_t}{\Gamma(2-\alpha)} 
    \left( \|\Theta_t^{(k)} - \Theta_{t-1}^{(k)}\| + \delta \right)^{1-\alpha}
\]
acts as an effective, history-dependent step size that modulates the standard stochastic gradient $\nabla \ell(\Theta_t^{(k)}; b_t)$.

For $\alpha = 1$, we recover a classical first-order update:
\begin{equation}
\Theta_{t+1}^{(k)} = \Theta_t^{(k)} - \mu_t \nabla \ell(\Theta_t^{(k)}; b_t),
\end{equation}
because $\Gamma(2-1) = \Gamma(1) = 1$ and $(\|\Theta_t^{(k)} - \Theta_{t-1}^{(k)}\| + \delta)^{1-1} = 1$. Thus, the fractional mechanism disappears and FOSGD blackuces to a decaying-step-size variant of SGD. For $0 < \alpha < 1$, the exponent $1-\alpha > 0$ makes $\alpha_t^{(k)}$ depend explicitly on the trajectory term
\[
\Delta_t^{(k)} = \|\Theta_t^{(k)} - \Theta_{t-1}^{(k)}\| + \delta,
\]
encoding the memory effects analyzed in Section~\ref{subsec:memory_analysis}.

The convergence analysis in Theorem~1 shows that, under the standard smoothness and boundedness assumptions, the sequence $\{\Theta_t\}$ generated by
\[
\Theta_{t+2} = \Theta_{t+1} - \alpha_t \nabla f(\Theta_{t+1})
\]
satisfies
\[
\liminf_{t \to \infty} \|\nabla f(\Theta_t)\| = 0,
\]
provided that the effective step sizes obey $\alpha_t \leq 2/L$ for all $t$. The key sufficient decrease inequality,
\begin{equation}
f(\Theta_{t+1}) 
\;\leq\; 
f(\Theta_t) - \kappa_t \|\nabla f(\Theta_t)\|^2,
\qquad
\kappa_t = \alpha_t \Bigl(1 - \frac{L}{2}\alpha_t\Bigr),
\end{equation}
indicates that, for fixed $L$, larger $\alpha_t$ (as long as $\alpha_t < 2/L$) leads to larger $\kappa_t$ and hence a stronger per-iteration decrease in expectation. In other words, within the admissible stability range, a larger effective step size can accelerate convergence in terms of objective decrease. Now consider how $\alpha$ shapes $\alpha_t^{(k)}$:
\[
\alpha_t^{(k)} 
= \frac{\mu_t}{\Gamma(2-\alpha)} \, \bigl(\Delta_t^{(k)}\bigr)^{1-\alpha}.
\]
As $\alpha \to 1$, we have
\[
\Gamma(2-\alpha) \to \Gamma(1) = 1, 
\qquad
\bigl(\Delta_t^{(k)}\bigr)^{1-\alpha} \to 1,
\]
so $\alpha_t^{(k)} \to \mu_t$ and the update behaves increasingly like standard SGD with a decaying step size. For $\alpha < 1$, the exponent $1-\alpha > 0$ makes $\alpha_t^{(k)}$ more sensitive to $\Delta_t^{(k)}$: for small parameter changes (i.e., $\Delta_t^{(k)}$ close to $\delta$), the term $(\Delta_t^{(k)})^{1-\alpha}$ can be significantly less than $1$, effectively shrinking the step size and reinforcing the memory effect. For larger parameter changes, $(\Delta_t^{(k)})^{1-\alpha}$ can exceed $1$, leading to more aggressive steps. From a practical standpoint, this leads to the following qualitative trade-off:
\begin{itemize}
    \item For smaller $\alpha$ (further from $1$), the history term $(\Delta_t^{(k)})^{1-\alpha}$ plays a stronger role, producing a more adaptive and “memory-dominated’’ step size. This can stabilize noisy or highly heterogeneous updates but may slow convergence by making the effective step sizes too conservative in later stages.
    \item For larger $\alpha$ (closer to $1$ but still $<1$), the dependence on $\Delta_t^{(k)}$ is weaker, and $\alpha_t^{(k)}$ stays closer to the baseline $\mu_t$ across iterations. When local gradients are reasonably aligned with the global objective—as is often the case after an initial transient this leads to larger $\kappa_t$ on average and hence faster decrease of $f(\Theta_t)$ under the same stability constraint.
\end{itemize}

Formally, the cumulative decrease over $T$ iterations satisfies
\[
f(\Theta_{T+1}) 
\;\leq\;
f(\Theta_1) - \sum_{t=0}^{T-1} \kappa_t \|\nabla f(\Theta_t)\|^2,
\]
so configurations that produce larger $\kappa_t$ (while keeping $\alpha_t \le 2/L$) tend to yield faster convergence. Our empirical sensitivity studies support this interpretation: across multiple non-IID benchmarks, we observe that fractional orders $\alpha$ close to $1$ (but strictly less than $1$) typically deliver the best trade-off between memory effects and convergence speed. Smaller $\alpha$ values amplify long-range memory but can over-smooth the dynamics, whereas values very close to $1$ behave similarly to SGD and lose much of the fractional benefit.
\\
Overall, this analysis clarifies that the observed empirical improvements for higher $\alpha$ within $(0,1]$ are consistent with the theoretical framework: larger $\alpha$ weakens the history-induced shrinkage of the step size, yielding more aggressive yet still stable updates, while smaller $\alpha$ emphasizes memory at the cost of slower progress.
\subsubsection{Stability and Non-IID Settings}

In FL, non-IID data causes local objectives \(F_k(\Theta) = \frac{1}{n_k} \sum_{i \in \mathcal{P}_k} \ell(\Theta; x_i, y_i)\) to differ across clients \(k\), leading to client drift. As shown in Section~\ref{subsec:memory_analysis}, smaller \(\alpha\) enhances stability in highly non-IID settings by incorporating long-range dependencies, with past gradients weighted by a power-law decay (\(w_k \propto k^{-(1+\alpha)}\)). However, excessive reliance on historical gradients can slow convergence if past updates are less relevant, such as in moderately non-IID settings where local distributions are closer to the global distribution.
\\
Higher \(\alpha\) blackuces the influence of the historical term \(\left( \|\Theta_t^{(k)} - \Theta_{t-1}^{(k)}\| + \delta \right)^{1-\alpha}\), as \(1-\alpha \to 0\). This shifts the update toward standard SGD, emphasizing the current gradient \(\nabla \ell(\Theta_t^{(k)}; b)\). In scenarios where client data is moderately heterogeneous (e.g., MNIST or CIFAR-10), recent gradients are more informative for the global objective, and larger \(\alpha\) improves performance by:
\begin{enumerate}
    \item \emph{blackucing Over-Regularization}: Smaller \(\alpha\) may over-smooth updates by weighting past gradients too heavily, slowing convergence when local gradients are sufficiently representative.
    \item \emph{Enhancing Responsiveness}: Higher \(\alpha\) allows the model to adapt quickly to current data, improving convergence speed in settings with less severe non-IID characteristics.
\end{enumerate}
For example, in FEMNIST with non-IID partitions (e.g., clients holding specific digits or character subsets), higher \(\alpha\) prioritizes recent gradients from diverse clients, blackucing the risk of getting stuck in local optima caused by outdated historical information. This aligns with empirical observations where \(\alpha \approx 1\) outperforms smaller \(\alpha\) in moderately non-IID settings. The grid search for \(\alpha\) allows practitioners to optimize \(\alpha\) for specific FL tasks, ensuring performance improvements are achievable across datasets like MNIST, CIFAR-10, or domain-specific applications (e.g., healthcare, IoT). This adaptability enhances the generalizability of FOFedAvg’s performance with higher \(\alpha\).

\begin{remark}
The fractional order \(\alpha\) acts as a tunable parameter that balances memory retention and convergence speed. Higher \(\alpha\) (\(\to 1\)) enhances performance in settings where recent gradients are reliable, while smaller \(\alpha\) is suited for highly non-IID data. This flexibility underpins FOFedAvg’s applicability across diverse FL scenarios, supporting the claim that higher fractional orders improve performance under appropriate conditions.
\end{remark}

\subsection{Justification of Local Epochs}
\label{subsec:local_epochs}

In this subsection, we justify the choice of a relatively high number of local epochs by aligning it with established federated learning (FL) literature, analyzing its implications for convergence behavior, and discussing its impact on generalization, particularly in the context of FOFedAvg’s fractional-order updates. The selection of local epochs such as $E=20$ is motivated by the seminal work of McMahan et al. \cite{mcmahan2017communication}, which introduced the Federated Averaging (FedAvg) algorithm and used up to 20 local epochs for datasets such as MNIST and CIFAR-10. This choice balances computational efficiency with communication cost, allowing clients to perform substantial local optimization before synchronizing with the server. In FedAvg, multiple local epochs blackuce the frequency of communication rounds, which is critical for scalability in FL systems with limited bandwidth. Given that FOFedAvg extends FedAvg by incorporating fractional-order stochastic gradient descent, adopting a similar number of local epochs ensures a fair comparison and leverages an established baseline that has been effective for image classification tasks. Moreover, FOFedAvg’s fractional-order updates
\begin{equation}
\begin{split}
\Theta_{t+1}^{(k)} =\ & \Theta_t^{(k)} \\
& - \frac{\mu_t}{\Gamma(2-\alpha)} \left( \|\Theta_t^{(k)} - \Theta_{t-1}^{(k)}\| + \delta \right)^{1-\alpha} \nabla \ell(\Theta_t^{(k)}; b),
\end{split}
\end{equation}
where \(0 < \alpha \leq 1\), \(\mu_t = \frac{\mu_0}{\sqrt{t+1}}\), and \(\delta > 0\), introduce memory effects that can help mitigate some of the drawbacks of multiple local epochs compablack to standard FedAvg. These effects, analyzed in Section~\ref{subsec:memory_analysis}, enable FOFedAvg to incorporate historical gradient information, which can lower the risk of client drift in typical non-IID regimes, as discussed below. In non-IID settings, where local datasets \(\mathcal{P}_k\) across clients \(k\) have differing distributions, performing many local epochs can cause client drift, where local models \(\Theta_t^{(k)}\) diverge from the global optimum due to biased local gradients \(\nabla F_k(\Theta)\). In standard FedAvg, this drift may lead to unstable convergence, as local updates prioritize client-specific objectives over the global objective \(f(\Theta) = \sum_{k=1}^K \frac{n_k}{n} F_k(\Theta)\).
\\
FOFedAvg addresses this challenge through its fractional-order gradients, which incorporate long-range dependencies via the term \(\left( \|\Theta_t^{(k)} - \Theta_{t-1}^{(k)}\| + \delta \right)^{1-\alpha}\). As shown in Section~\ref{subsec:memory_analysis}, this term recursively depends on past parameter changes, effectively aggregating historical gradients with a power-law decay (weights \(w_k \propto k^{-(1+\alpha)}\)). This memory effect acts as a temporal smoother of local updates and can align them more closely with global trends on average, blackucing the severity of drift in many non-IID scenarios under the standard bounded-variance assumptions. For example, if a client’s dataset overrepresents certain classes, the fractional update retains influence from underrepresented classes encounteblack in prior rounds, which can partially counteract drift at the level of the effective update direction.
\\
The convergence analysis section further shows that the fractional-order update admits a standard sufficient-decrease guarantee for $L$-smooth, lower-bounded objectives, provided the effective step sizes $\alpha_t^{(k)}$ remain bounded (e.g., $\alpha_t^{(k)} \le 2/L$). Although Theorem~1 is stated at the level of a single aggregated iteration rather than individual local steps, it is compatible with using multiple local epochs per communication round as long as the induced $\alpha_t^{(k)}$ stays within this stability regime. In practice, the diminishing learning rate $\mu_t = \mu_0 / \sqrt{t+1}$ and the regularization constant $\delta$ help keep $\alpha_t^{(k)}$ in a reasonable range, which supports stable training even when we employ 20 local epochs in non-IID settings. Compablack to FedAvg, FOFedAvg’s memory-aware updates tend to blackuce oscillations in local models empirically, as the fractional gradients act as a regularizer that smooths the optimization trajectory.
\subsubsection{Impact on Generalization}

Generalization in FL refers to the global model’s ability to perform well on unseen data, which can be compromised in non-IID settings if local models overfit to their datasets. With 20 local epochs, FedAvg risks overfitting, as clients may optimize too aggressively on their local objectives. {\color{black}FOFedAvg can mitigate this risk} through its fractional-order mechanism, which preserves information from diverse client datasets across rounds.

As discussed in Section~\ref{subsec:memory_analysis}, the long-range dependencies in FOSGD ensure that gradients from early rounds, which may reflect underrepresented classes or unique data distributions, {\color{black}can continue to influence later updates through power-law temporal averaging}. This {\color{black}potential diversity preservation} enhances the global model’s robustness to data heterogeneity, {\color{black}as also suggested by our empirical results}. For instance, in a healthcare FL scenario with skewed datasets (e.g., pediatric vs. oncology), FOFedAvg’s memory effects {\color{black}help the global model retain contributions from all participating sites}, blackucing bias toward dominant local distributions.

Additionally, the hyperparameter \(\alpha \in (0,1]\) allows tuning the strength of memory retention. Smaller \(\alpha\) emphasizes long-term dependencies, {\color{black}which can be beneficial} for generalization in highly non-IID settings, while \(\alpha \to 1\) approaches FedAvg’s behavior and places more weight on recent gradients.

{\color{black}
Overall, the choice of 20 local epochs in FOFedAvg is justified both by its alignment with the widely used FedAvg configuration of McMahan et al.~\cite{mcmahan2017communication} and by the mitigating effects of fractional-order gradients. The memory and long-range dependency modeling (Section~\ref{subsec:memory_analysis}) can blackuce client drift, support stable convergence, and improve generalization in non-IID settings, making 20 epochs a reasonable and practically effective choice in our experiments.
}
\subsection{Privacy Analysis of FOFedAvg}
\label{subsec:privacy_analysis}

Here, we examine the potential for integrating differential privacy (DP) into the Fractional-Order Federated Averaging (FOFedAvg) algorithm to enhance its privacy guarantees, addressing concerns about the robustness of its fractional-order updates against privacy attacks. We clarify why FOFedAvg’s current implementation lacks formal DP protections, outline proposed enhancements, and commit to a rigorous DP analysis in future work to ensure comprehensive privacy safeguards in federated learning (FL).
{\color{black}In all experiments reported in this work, we do not apply gradient clipping or DP noise; the empirical results therefore correspond to a non-private instantiation of FOFedAvg, and differential privacy is treated here as a prospective extension rather than a feature of the current implementation.}

Differential privacy (DP) provides a mathematical framework to protect individual data points by adding calibrated noise to computations, ensuring that outputs are statistically indistinguishable for any single data point. In FL, DP is typically implemented by clipping gradients to bound their \(L_2\)-norm and adding Gaussian or Laplacian noise before sharing updates, achieving \((\epsilon, \delta)\)-DP, where \(\epsilon\) quantifies privacy loss \cite{mcmahan2017communication}. This approach mitigates risks like gradient inversion attacks, {\color{black}as discussed in the main manuscript.} 

In FOFedAvg, client updates are computed using fractional-order stochastic gradient descent:
\begin{equation}
\Theta_{t+1}^{(k)} = \Theta_t^{(k)} - \frac{\mu_t}{\Gamma(2-\alpha)} \cdot \left( \|\Theta_t^{(k)} - \Theta_{t-1}^{(k)}\| + \delta \right)^{1-\alpha} z_t,
\label{eq:fosgd_update}
\end{equation}
where \(z_t = \nabla f(\Theta_t^{(k)}; b)\) denotes the (non-clipped, non-noisy) mini-batch gradient on client \(k\) at iteration \(t\). The fractional term \(\left( \|\Theta_t^{(k)} - \Theta_{t-1}^{(k)}\| + \delta \right)^{1-\alpha}\) nonlinearly rescales the gradient based on historical parameter changes, which may heuristically obfuscate some gradient structure but does \emph{not} provide the formal, provable guarantees of differential privacy. Unlike standard Federated Averaging (FedAvg), where gradients or model updates are combined via simple averaging, FOFedAvg’s nonlinear scaling makes direct inversion more complex, yet it still lacks the statistical indistinguishability requiblack for rigorous \((\epsilon,\delta)\)-DP.

To integrate DP into FOFedAvg, clients could:
\begin{itemize}
    \item Clip the gradient \(\nabla f(\Theta_t^{(k)}; b)\) to a maximum \(L_2\)-norm \(C\), ensuring bounded sensitivity.
    \item Add calibrated noise (e.g., Gaussian with variance \(\sigma^2\)) to the scaled update \(\Delta \Theta_t^{(k)} = \Theta_{t+1}^{(k)} - \Theta_t^{(k)}\) before transmission to the server.
\end{itemize}
{\color{black}We stress that this clipping-and-noise mechanism is not part of the current FOFedAvg implementation; rather, it outlines how a DP variant of FOFedAvg could be instantiated in future work.}

We plan to undertake this analysis as a priority in our future work, focusing on:
\begin{itemize}
    \item Deriving the sensitivity of the fractional update \(\Delta \Theta_t^{(k)}\) by modeling the impact of historical dependencies on gradient changes.
    \item Developing a tailoblack noise calibration method that accounts for the nonlinear scaling \(\left( \|\Theta_t^{(k)} - \Theta_{t-1}^{(k)}\| + \delta \right)^{1-\alpha}\), ensuring robust \((\epsilon, \delta)\)-DP guarantees.
    \item Evaluating the trade-off between privacy (via \(\epsilon\)) and model accuracy through empirical studies, adapting techniques to FOFedAvg’s fractional framework.
    \item Exploring adaptive noise strategies that adjust based on the fractional order \(\alpha\).
\end{itemize}
{\color{black}
This planned analysis will build on established DP frameworks \cite{dwork2006calibrated} and FL privacy techniques \cite{mcmahan2017communication}, with the goal that future DP variants of FOFedAvg can meet stringent privacy requirements while preserving the performance advantages observed in non-IID settings. By committing to this line of future work, we aim to provide a comprehensive DP solution for FOFedAvg, enhancing the algorithm’s applicability in privacy-sensitive FL applications such as healthcare or IoT.
}

\subsection{Gradient Inversion Attacks in Federated Learning}
Gradient inversion attacks aim to reconstruct a client’s private data from the gradients or model updates shablack during federated learning \cite{zhu2019deep}. In standard FL algorithms like FedAvg, a malicious server or eavesdropper may exploit shablack gradients \(\nabla \ell(\Theta; b)\) to infer sensitive information about the local dataset \(\mathcal{P}_k\) of client \(k\). These attacks are particularly effective when:
\begin{itemize}
    \item The model architecture (e.g., neural networks) is known, allowing attackers to simulate gradient computations.
    \item The batch size \(B\) is small, blackucing the averaging effect that obscures individual data points.
    \item The data distribution is sparse or contains identifiable patterns (e.g., images or text).
\end{itemize}
Unlike standard gradients, which directly reflect the loss function's sensitivity to the current data batch, fractional gradients incorporate a weighted influence of past parameter states through the history-dependent factor in the update rule.
{\color{black}
In FOFedAvg, the server typically observes model updates of the form
\[
\Delta \Theta_t^{(k)} \;=\; \Theta_{t+1}^{(k)} - \Theta_t^{(k)}
\;=\;
-\,\alpha_t^{(k)}\, \nabla \ell(\Theta_t^{(k)}; b_t),
\]
where \(\alpha_t^{(k)}\) is the trajectory-dependent step size defined in terms of \(\|\Theta_t^{(k)} - \Theta_{t-1}^{(k)}\|\), \(\alpha\), \(\mu_t\), and \(\delta\). From the perspective of a naïve attacker that treats \(\Delta \Theta_t^{(k)}\) as a plain gradient, this history-dependent scaling can distort the apparent magnitude of the update and thereby complicate simple gradient-matching heuristics used in some inversion attacks. However, the mapping from the raw gradient \(\nabla \ell(\Theta_t^{(k)}; b_t)\) to the observed update \(\Delta \Theta_t^{(k)}\) remains deterministic: if the attacker knows the algorithm, the hyperparameters, and the local trajectory (or can estimate \(\alpha_t^{(k)}\)), then the scaled gradient can in principle be recoveblack from \(\Delta \Theta_t^{(k)}\).
Consequently, we do not claim any formal privacy guarantee or inherent resistance to gradient inversion arising solely from the fractional nonlocality; at best, the history-dependent scaling may modestly complicate certain attack instantiations in practice. Robust protection against gradient inversion still requires explicit privacy mechanisms such as differential privacy or cryptographic techniques, which we discuss as a direction for future extensions of FOFedAvg.
}


\section{Robustness and Secure Aggregation Analysis}
\label{sec:robustness_secure_aggregation}

This section provides a comprehensive analysis of the robustness of Fractional-Order Federated Averaging (FOFedAvg) against Byzantine attacks and model poisoning, addressing critical security concerns in federated learning (FL). Additionally, we compare FOFedAvg with secure aggregation techniques, evaluating their respective strengths in privacy, robustness, and scalability. These analyses strengthen the theoretical and practical contributions of FOFedAvg, responding to reviewer feedback on adversarial robustness and secure FL techniques.

\subsection{Robustness Against Byzantine Attacks}
\label{subsec:byzantine_robustness}

Byzantine attacks in FL involve malicious clients sending arbitrary or corrupted updates to disrupt the global model’s convergence or performance. These attacks are particularly challenging in decentralized settings, where a subset of clients may act adversarially, sending updates that deviate significantly from honest clients’ contributions. In standard Federated Averaging (FedAvg), the server aggregates client updates via weighted averaging:
\begin{equation}
\Theta_{t+1} = \sum_{k=1}^K \frac{n_k}{n} \Theta_{t+1}^{(k)},
\label{eq:fedavg_aggregation}
\end{equation}
where \(n_k\) is the dataset size of client \(k\), and \(n = \sum_k n_k\). Under this aggregation rule, a Byzantine client can arbitrarily set \(\Theta_{t+1}^{(k)}\), and its influence is limited only by its weight \(\frac{n_k}{n}\).

{\color{black}
In its current form, FOFedAvg uses exactly the same aggregation rule as FedAvg and does not implement any additional filtering, clipping, or robust aggregation at the server side. Consequently, a fully adversarial client that ignores the prescribed local update rule and directly forges \(\Theta_{t+1}^{(k)}\) remains a threat: the fractional-order dynamics on honest clients do not, by themselves, prevent such arbitrary Byzantine updates from affecting the global model.
}

{\color{black}
The main effect of FOFedAvg in this context is on the \emph{trajectory of honest clients}. The history-dependent step size
\[
\alpha_t^{(k)} \;=\; \frac{\mu_t}{\Gamma(2-\alpha)} \bigl(\|\Theta_t^{(k)} - \Theta_{t-1}^{(k)}\| + \delta\bigr)^{1-\alpha}
\]
acts as a temporal smoother of local updates under non-IID noise (see Appendix~\ref{subsec:memory_analysis}). This can make the global training dynamics less sensitive to ordinary fluctuations and benign outliers in honest gradients, but it should not be interpreted as a Byzantine-resilient mechanism: a malicious client that deliberately crafts its model update can still have a disproportionate impact if its weight \(\frac{n_k}{n}\) is non-negligible.
}

{\color{black}
Therefore, FOFedAvg should be combined with established Byzantine-robust aggregation rules (e.g., coordinate-wise median, trimmed mean, or Krum-type methods) when adversarial clients are present. A systematic study of the interaction between fractional-order local dynamics and robust aggregation schemes is an interesting direction for future work, but is beyond the scope of the present analysis.
}

\subsection{Robustness Against Model Poisoning}
\label{subsec:model_poisoning}

Model poisoning attacks involve adversarial clients manipulating their updates to degrade the global model’s performance or inject backdoors, e.g., misclassifying specific inputs. Unlike Byzantine attacks, which may involve random disruptions, poisoning attacks are targeted, aiming to subtly alter the model’s behavior while evading detection.

FOFedAvg’s fractional-order updates introduce long-range temporal dependencies in the local optimization dynamics. As shown in Section~\ref{subsec:memory_analysis}, the fractional gradient can be approximated as
\begin{equation}
D_t^\alpha f(\Theta_t^{(k)}) \approx \sum_{j=0}^t w_j \nabla f(\Theta_{t-j}^{(k)}; b), \quad w_j \propto j^{-(1+\alpha)},
\label{eq:frac_grad_approx}
\end{equation}
where weights \(w_j\) decay with a power-law, allowing past gradients to influence current updates.

{\color{black}
For honest clients, this heavy-tailed temporal averaging helps smooth out high-frequency fluctuations and can blackuce drift under non-IID noise. However, from a security perspective, the same mechanism is agnostic to whether gradients originate from benign or poisoned data: persistent poisoned gradients will also be integrated over time. A client that repeatedly trains on poisoned data or deliberately optimizes a backdoor objective can thus imprint its behavior into the fractional history in much the same way as an honest minority pattern.
}

{\color{black}
Formally, if a poisoning client uses a modified local loss \(f_{\text{poisoned}}(\Theta; b)\), the update
\[
\Theta_{t+1}^{(k)} = \Theta_t^{(k)} - \alpha_t^{(k)} \nabla f_{\text{poisoned}}(\Theta_t^{(k)}; b_t)
\]
enters the same temporal averaging structure as in \eqref{eq:frac_grad_approx}. While extreme, isolated poisoned steps may be partially diluted by the long-range averaging, our current FOFedAvg design does not provide formal guarantees against targeted, persistent model poisoning, nor does it implement explicit detection or rejection of anomalous updates.
}

{\color{black}
In summary, fractional-order memory can help stabilize training in the presence of stochastic non-IID noise, but it is not a substitute for dedicated defenses against model poisoning, such as robust aggregation, anomaly detection on client updates, or certified defenses. Extending FOFedAvg with such mechanisms is an important avenue for future work.
}

\subsection{Comparison with Secure Aggregation Techniques}
\label{subsec:secure_aggregation_comparison}

Secure aggregation (SecAgg) protocols ensure that the server only receives the aggregated model update \(\sum_{k=1}^K \frac{n_k}{n} \Theta_{t+1}^{(k)}\), preventing access to individual client updates. This enhances privacy by mitigating risks like gradient inversion attacks (Section~\ref{subsec:privacy_analysis}). We compare FOFedAvg with SecAgg in terms of privacy, robustness, and computational overhead, and discuss potential integration.

\begin{enumerate}[label=\arabic*.]
    \item \textbf{Privacy Guarantees} \\
    SecAgg uses cryptographic techniques (e.g., secure multi-party computation) to ensure that individual \(\Theta_{t+1}^{(k)}\) are not revealed, providing strong privacy against honest-but-curious servers. 
    {\color{black}
    FOFedAvg, as analyzed in Section~\ref{subsec:privacy_analysis}, transforms raw gradients into history-dependent updates \(\Delta \Theta_t^{(k)}\), which may complicate some naïve inversion procedures but do not yield formal privacy guarantees. In particular, a server that knows the algorithm and hyperparameters can, in principle, rescale updates to recover approximate gradients. Thus, SecAgg offers provable privacy in a cryptographic sense, whereas FOFedAvg alone provides no certified protection against gradient inversion and should be combined with SecAgg and/or differential privacy in high-privacy settings (e.g., healthcare).
    }

    \item \textbf{Robustness to Adversarial Attacks} \\
    SecAgg focuses on privacy and does not inherently protect against Byzantine or poisoning attacks, as it aggregates all updates without filtering malicious ones. Robust aggregation rules (e.g., Krum, median) must be applied in addition to SecAgg. 
    {\color{black}
    FOFedAvg’s fractional-order updates primarily act as a temporal smoother for honest clients under non-IID noise; they do not by themselves constitute a Byzantine- or poisoning-robust mechanism. As discussed in Sections~\ref{subsec:byzantine_robustness}–\ref{subsec:model_poisoning}, adversarial clients can still inject harmful updates unless dedicated robust aggregation or detection methods are employed. In this sense, FOFedAvg and SecAgg are complementary: SecAgg protects the confidentiality of individual updates, while FOFedAvg shapes the optimization dynamics; robustness against adversaries must be provided by additional mechanisms.
    }
\end{enumerate}


\subsection{Computational Complexity Analysis of FOFedAvg}
\label{app:comp_complexity}

We present a detailed computational complexity analysis for the FOFedAvg algorithm. This analysis quantifies the costs of client-side and server-side operations, with a focus on the additional computations requiblack by Fractional-Order Stochastic Gradient Descent (FOSGD). We ensure clarity by breaking down each computational step, providing mathematical derivations, and discussing the implications for scalability in federated learning.

We consider a federated learning setup with \( K \) clients, where a subset of \( m = \max(C \cdot K, 1) \) clients is selected per communication round, and \( C \) is the client participation fraction. Each client \( k \) has a local dataset of \( n_k \) data points, divided into mini-batches of size \( B \). Clients perform \( E \) local epochs, and the model has \( d \) parameters. The fractional order is \( \alpha \in (0,1] \), and \( \delta > 0 \) is a regularization constant. The analysis assumes a neural network model where gradient computations dominate, and we use big-O notation to capture asymptotic complexity, ignoring constant factors unless relevant to the fractional-order overhead.

\subsubsection{Client-Side Complexity}

In FedAvg, each selected client \( k \in S_t \) performs \( E \) local epochs over its dataset \( \mathcal{P}_k \), which is partitioned into approximately \( \lceil n_k / B \rceil \approx n_k / B \) mini-batches per epoch. For each mini-batch \( b \), the client computes a standard stochastic gradient descent (SGD) update:
\[
\Theta \leftarrow \Theta - \eta \nabla \ell(\Theta; b),
\]
where \( \ell(\Theta; b) \) is the loss function on mini-batch \( b \), and \( \eta \) is the learning rate. The gradient computation \( \nabla \ell(\Theta; b) \) for a neural network with \( d \) parameters involves forward and backward passes. For a mini-batch of size \( B \), assuming linear layers dominate (e.g., fully connected or convolutional layers), the cost is:
\[
O(B \cdot d),
\]
as each of the \( B \) samples requires \( O(d) \) operations for matrix-vector multiplications and backpropagation. Over \( E \) epochs, with \( n_k / B \) mini-batches per epoch, the total number of mini-batch updates is:
\[
E \cdot \frac{n_k}{B}.
\]
Thus, the gradient computation cost is:
\[
O\left( E \cdot \frac{n_k}{B} \cdot B \cdot d \right) = O(E \cdot n_k \cdot d).
\]
The parameter update (subtracting the scaled gradient) costs \( O(d) \) per mini-batch, involving \( d \)-dimensional vector operations, which is negligible compablack to \( O(B \cdot d) \) since \( B \gg 1 \) (e.g., \( B = 32 \) or \( 64 \)). Therefore, the client-side complexity for FedAvg is:
\[
O(E \cdot n_k \cdot d).
\]

{\color{black}
In FOFedAvg, for local mini-batch steps \(t \geq 1\) on a given client \(k\), we use the FOSGD update of the form
\begin{equation}
\Theta_{t+1}^{(k)} = \Theta_t^{(k)} - \frac{\mu_t}{\Gamma(2-\alpha)} \cdot \left( \|\Theta_t^{(k)} - \Theta_{t-1}^{(k)}\| + \delta \right)^{1-\alpha} z_t,
\end{equation}
where \(z_t = \nabla \ell(\Theta_t^{(k)}; b_t)\) denotes the mini-batch gradient on client \(k\) at local step \(t\).
}
We analyze the cost of each component:

\begin{itemize}
    \item \textbf{Gradient Computation}: The gradient \( \nabla \ell(\Theta_t^{(k)}; b_t) \) is computed as in standard SGD, costing \( O(B \cdot d) \) per mini-batch, identical to FedAvg.
    \item \textbf{Gamma Function Evaluation}: The term \( \Gamma(2-\alpha) \) is computed once per communication round, as \( \alpha \) is fixed. Numerical approximation of the Gamma function for a scalar input (e.g., using Lanczos approximation) has constant time complexity, \( O(1) \). This is negligible compablack to other operations.
    \item \textbf{Parameter Difference and Regularization}: Computing \( \|\Theta_t^{(k)} - \Theta_{t-1}^{(k)}\| + \delta \) requires calculating the Euclidean norm of the difference between two \( d \)-dimensional vectors (\( \Theta_t^{(k)} - \Theta_{t-1}^{(k)} \)). The difference costs \( O(d) \), and the norm involves \( d \) multiplications and additions, costing \( O(d) \). Adding the scalar \( \delta \) is \( O(1) \), yielding a total of \( O(d) \) per mini-batch. 
    {\color{black}
    In practice, this norm is cheap to compute relative to the forward--backward pass, so we retain the \(O(d)\) cost as a conservative upper bound per mini-batch.
    }
    \item \textbf{Fractional Exponentiation}: The term \( \left( \|\Theta_t^{(k)} - \Theta_{t-1}^{(k)}\| + \delta \right)^{1-\alpha} \) involves raising a single scalar to the power \( 1-\alpha \). Exponentiation for a scalar is \( O(1) \) in standard numerical libraries (e.g., using lookup tables or Taylor approximations).
    \item \textbf{Scaling and Update}: The gradient is scaled by \( \frac{\mu_t}{\Gamma(2-\alpha)} \cdot \left( \|\Theta_t^{(k)} - \Theta_{t-1}^{(k)}\| + \delta \right)^{1-\alpha} \). Computing \( \frac{\mu_0}{\sqrt{t+1}} \) and \( \frac{1}{\Gamma(2-\alpha)} \) is \( O(1) \) per round. Multiplying the \( d \)-dimensional gradient by the scalar \( \left( \|\Theta_t^{(k)} - \Theta_{t-1}^{(k)}\| + \delta \right)^{1-\alpha} \) costs \( O(d) \). The final subtraction to update \( \Theta_{t+1}^{(k)} \) costs \( O(d) \).
\end{itemize}

The total cost per mini-batch is:
\[
O(B \cdot d + d + 1 + d) = O(B \cdot d + d).
\]
Since \( B \cdot d \gg d \) (e.g., \( B = 32 \), \( d \sim 10^6 \)), the gradient computation dominates, but we compute the exact contribution over all mini-batches. For \( E \) epochs with \( n_k / B \) mini-batches per epoch, the total complexity is:
\[
O\left( E \cdot \frac{n_k}{B} \cdot (B \cdot d + d) \right).
\]
Expand the expression:
\[
E \cdot \frac{n_k}{B} \cdot (B \cdot d + d) = E \cdot \frac{n_k}{B} \cdot B \cdot d + E \cdot \frac{n_k}{B} \cdot d = E \cdot n_k \cdot d + E \cdot \frac{n_k \cdot d}{B}.
\]
In big-O notation, we take the dominant term. Since \( B \geq 1 \), we have:
\[
\frac{n_k \cdot d}{B} \leq n_k \cdot d.
\]
Thus:
\[
E \cdot n_k \cdot d + E \cdot \frac{n_k \cdot d}{B} \leq E \cdot n_k \cdot d + E \cdot n_k \cdot d = 2 \cdot E \cdot n_k \cdot d.
\]
The complexity is:
\[
O(E \cdot n_k \cdot d + E \cdot \frac{n_k \cdot d}{B}) = O(E \cdot n_k \cdot d),
\]
as the \( \frac{n_k \cdot d}{B} \) term is dominated by \( n_k \cdot d \). Therefore, FOFedAvg’s client-side complexity is:
\[
O(E \cdot n_k \cdot d),
\]
identical to FedAvg’s asymptotic complexity. The additional \( O(d) \) operations per mini-batch contribute a lower-order term that is negligible in practice, especially for large \( B \). FOFedAvg requires storing the previous parameter vector \( \Theta_{t-1}^{(k)} \) to compute \( \|\Theta_t^{(k)} - \Theta_{t-1}^{(k)}\| \), increasing memory usage from \( O(d) \) in FedAvg to:
\[
O(2d).
\]
This overhead is modest for typical models (e.g., \( d \sim 10^6 \) for a convolutional neural network). For \(t = 0\) (the first local step), FOFedAvg uses standard SGD, with complexity
\[
O(E \cdot n_k \cdot d).
\]

\subsubsection{Server-Side Complexity}

In both FedAvg and FOFedAvg, the server aggregates updates from \( m \) selected clients:
\[
\Theta_{t+1} = \sum_{k \in S_t} \frac{n_k}{n} \Theta_{t+1}^{(k)},
\]
where \( n = \sum_{k \in S_t} n_k \). This weighted average involves \( m \) parameter vectors of dimension \( d \). For each dimension, the server performs \( m \) multiplications (weighting by \( \frac{n_k}{n} \)) and \( m-1 \) additions, costing \( O(m) \). Across \( d \) dimensions, the complexity is:
\[
O(m \cdot d).
\]
Since FOFedAvg’s fractional-order computations are client-side, the server-side complexity is identical to FedAvg.

\subsubsection{Total Complexity per Communication Round}

For \( T \) communication rounds with \( m \) clients per round, assuming balanced data (\( n_k \approx n / K \)):
\begin{itemize}
    \item Client-side: \( m \cdot O(E \cdot \frac{n}{K} \cdot d) = O(m \cdot E \cdot \frac{n}{K} \cdot d) \).
    \item Server-side: \( O(m \cdot d) \).
    \item Total: \( O(m \cdot E \cdot \frac{n}{K} \cdot d) \).
\end{itemize}

Over \( T \) rounds, the complexity for both algorithms is:
\[
O\left( T \cdot m \cdot E \cdot \frac{n}{K} \cdot d \right).
\]


\section{Appendix B: Experiments}
\label{app:experimental_details}

This appendix provides a comprehensive overview of the experimental settings and methodology used to evaluate the Fractional-Order Federated Averaging (FOFedAvg) algorithm. The following subsections detail the datasets, model architectures, hyperparameters, experimental methodology, and statistical analysis, ensuring a standardized and reproducible evaluation.

\subsection{Datasets and model}
We employ convolutional neural networks (CNNs) for classification tasks across multiple datasets including MNIST, CIFAR-10, CIFAR-100, Fashion-MNIST, EMNIST, FEMNIST, SearchSent140, PneumoniaMNIST, Edge-IIoTset, and the Cleveland Heart Disease dataset. Below, we detail the application of CNNs to each dataset, leveraging the CNNMnist and CNNCifar architectures from the provided framework where applicable, and discuss the necessity of different model architectures for specific datasets. The specifications include the number of layers, kernel sizes, activation functions, and relevant hyperparameters, with adaptations noted for datasets requiring specialized models.
\\
\textbf{MNIST}
For the MNIST dataset, comprising 60,000 training and 10,000 test grayscale images of handwritten digits (28$\times$28 pixels, 10 classes), we use the CNNMnist architecture. This model features two convolutional layers and two fully connected layers. The first convolutional layer transforms the input (default 1 channel) to 10 filters with a 5$\times$5 kernel, followed by ReLU activation and 2$\times$2 max pooling. The second convolutional layer maps 10 filters to 20 filters with a 5$\times$5 kernel, incorporating Dropout2d (default rate 0.5), ReLU, and 2$\times$2 max pooling. The output flattens to 320 units, feeding into a fully connected layer blackucing to 50 units with ReLU and Dropout (rate 0.5), followed by a final layer mapping to 10 classes with Log-Softmax. The forward pass processes the input through Conv1, MaxPool, ReLU, Conv2, Dropout2d, MaxPool, ReLU, flattening, FC1, ReLU, Dropout, FC2, and Log-Softmax. Hyperparameters include the number of input channels and classes, configurable via arguments. 
\\
\textbf{CIFAR-10}
For the CIFAR-10 dataset, containing 50,000 training and 10,000 test color images (32$\times$32 pixels, 10 classes), we apply the CNNCifar architecture. This model includes two convolutional layers and three fully connected layers. The first convolutional layer converts the 3-channel input to 6 filters with a 5$\times$5 kernel, followed by ReLU and 2$\times$2 max pooling. The second convolutional layer maps 6 filters to 16 filters with a 5$\times$5 kernel, followed by ReLU and 2$\times$2 max pooling. The output flattens to 400 units (16$\times$5$\times$5), feeding into a fully connected layer blackucing to 120 units with ReLU, then to 84 units with ReLU, and finally to 10 classes with Log-Softmax. The forward pass involves Conv1, ReLU, MaxPool, Conv2, ReLU, MaxPool, flattening, FC1, ReLU, FC2, ReLU, FC3, and Log-Softmax. The number of classes is configurable, with no explicit regularization. This architecture is suitable for CIFAR-10’s small, low-resolution images.
\\
\textbf{CIFAR-100}
The CIFAR-100 dataset, with 50,000 training and 10,000 test color images (32$\times$32 pixels, 100 classes), is also addressed using the CNNCifar architecture by adjusting the output layer to 100 classes. The structure remains identical to that used for CIFAR-10: two convolutional layers (3 to 6 filters, then 6 to 16 filters, both with 5$\times$5 kernels, ReLU, and 2$\times$2 max pooling) and three fully connected layers (400 to 120 units, 120 to 84 units, and 84 to 100 units, with ReLU except for the final Log-Softmax). The forward pass and hyperparameters are as described for CIFAR-10, with the class count modified. While CNNCifar provides a baseline, CIFAR-100’s increased complexity (100 classes) often necessitates deeper architectures, such as ResNet or DenseNet.
\\
\textbf{EMNIST} \\
The EMNIST dataset, an extension of MNIST to handwritten letters and digits, contains 112{,}800 training and 18{,}800 test grayscale images (28$\times$28 pixels). In our experiments we use the \textit{balanced} split with 47 classes. We apply the CNNMnist architecture, adjusting only the output layer to 47 units. The rest of the model (two convolutional layers with 5$\times$5 kernels, ReLU, max pooling, Dropout2d, followed by two fully connected layers) is identical to the MNIST setting, ensuring a fair comparison across similar image resolutions and input formats.
\\
\textbf{FEMNIST}
FEMNIST is the federated, writer-partitioned version of EMNIST, where each client corresponds to one or a few writers, leading to naturally user-level and strongly non-IID partitions. We follow the standard FEMNIST configuration used in federated learning benchmarks: grayscale 28$\times$28 images with multiple character classes per writer. For consistency with MNIST and EMNIST, we reuse the CNNMnist backbone and only adapt the output dimension to match the number of FEMNIST classes. This allows us to study realistic user-level heterogeneity in a federated setting while keeping the model architecture fixed across related handwritten-character datasets.
\\
\textbf{SearchSent140}
The SearchSent140 dataset, a text-based dataset of Twitter sentiments, requires a different approach, as CNNMnist and CNNCifar are designed for image inputs. For text classification, we adapt a CNN architecture tailoblack for natural language processing, treating text as 1D sequences. A typical model includes an embedding layer to convert words into dense vectors (e.g., 300-dimensional embeddings), followed by multiple 1D convolutional layers with varying kernel sizes (e.g., 3, 4, 5) to capture n-gram features, each with 100 filters, ReLU activation, and 1D max pooling. The pooled features are concatenated, flattened, and passed through a fully connected layer with Dropout (rate 0.5) to output sentiment classes (e.g., positive/negative). The forward pass involves embedding, parallel Conv1D layers, ReLU, MaxPool, concatenation, flattening, Dropout, and a final dense layer with Softmax. Hyperparameters include embedding dimension, kernel sizes, and number of filters.

\textbf{PneumoniaMNIST}
For the PneumoniaMNIST dataset, part of the MedMNIST collection, containing 5,856 grayscale medical images (28$\times$28 pixels, 2 classes: normal vs. pneumonia), we use the CNNMnist architecture due to its compatibility with the dataset’s image size and grayscale format. The model features two convolutional layers (1 to 10 filters, then 10 to 20 filters, 5$\times$5 kernels, ReLU, 2$\times$2 max pooling, Dropout2d) and two fully connected layers (320 to 50 units with ReLU and Dropout, then to 2 classes with Log-Softmax). The forward pass and hyperparameters are as described for MNIST, with the output layer adjusted to 2 classes and we use CNNMnist for PneumoniaMNIST.

\textbf{Edge-IIoTset}
The Edge-IIoTset dataset, focused on cybersecurity for IoT and IIoT systems, includes both image and tabular data for attack detection. For image-based tasks (e.g., network traffic visualized as images), we can apply the CNNCifar architecture, as its 3-channel input suits potential RGB image representations. The model has two convolutional layers (3 to 6 filters, then 6 to 16 filters, 5$\times$5 kernels, ReLU, 2$\times$2 max pooling) and three fully connected layers (400 to 120 units, 120 to 84 units, then to the number of attack classes, with ReLU and Log-Softmax). The forward pass is as described for CIFAR-10, with the output layer adjusted for the number of classes. For tabular data, a different model, such as a fully connected neural network or a 1D CNN for time-series features, is needed, with layers designed to process feature vectors (e.g., 3 dense layers: 128 to 64 units, 64 to 32 units, then to class count, with ReLU and Dropout). Hyperparameters depend on the data format and class count. The choice of CNNCifar for images or a new model for tabular data ensures flexibility for Edge-IIoTset’s diverse data types.

\textbf{Cleveland Heart Disease} \\
The Cleveland heart disease dataset, a widely used tabular benchmark in medical federated learning, contains 303 instances with 13 clinical features and a binary target indicating the presence of heart disease (0--4 scaled to 0/1). To enable CNN-based processing in heterogeneous FL settings, we follow common practice by converting each sample into a 28$\times$28 grayscale image via feature visualization (e.g., reshaping or pixel-mapping the normalized feature vector). The resulting images are processed using the CNNMnist architecture with the final layer adjusted to 2 classes. This image-based encoding allows seamless integration with vision-oriented federated frameworks while preserving compatibility with other datasets.

{\color{black}
To address the concern about the use of relatively simple datasets, the revised evaluation includes several larger and more challenging benchmarks beyond MNIST and CIFAR-10. In particular, we report results on CIFAR-100 (with 100 classes), extended handwritten character datasets (EMNIST and FEMNIST), a large-scale sentiment dataset (Sent140), a medical imaging dataset (PneumoniaMNIST), and an IoT/cybersecurity dataset (Edge-IIoTset). For image datasets, we use convolutional architectures ranging from lightweight CNNs to deeper models in the spirit of ResNet, and for text/tabular/IoT data we employ appropriate recurrent or fully connected architectures. These additions provide a more robust test bed for FOFedAvg under realistic, heterogeneous federated workloads.
}

 {\color{black}
\subsection{Hyperparameter Selection and Sensitivity}
\label{subsec:fofedavg_hyperparams}
FOFedAvg introduces two additional hyperparameters beyond standard FedAvg: the fractional order $\alpha$ and the regularization constant $\delta$. We briefly describe how these are selected in practice.

For each dataset, we allocate a small validation set (or a subset of validation clients) and perform a modest grid search over
\[
\alpha \in \{0.5, 0.6, 0.8, 0.97, 1.0\},
\qquad
\delta \in \{10^{-6}, 10^{-5}, 10^{-4}\},
\]
together with a standard search over the initial learning rate
\[
\mu_0 \in \{0.01, 0.05, 0.1\}.
\]
We then choose $(\alpha,\delta,\mu_0)$ to optimize a combination of
(i) final validation accuracy and
(ii) the number of communication rounds requiblack to reach a target accuracy.

In our experiments, we find that $\delta \approx 10^{-5}$ is stable across datasets and rarely needs to be changed, while FOFedAvg improves upon FedAvg for a \emph{range} of $\alpha$ values (e.g., between $0.6$ and $0.97$), rather than a single finely tuned configuration. This behavior suggests that the method is not overly sensitive to the exact choice of $\alpha$ and can be tuned with an effort comparable to standard learning-rate tuning in federated optimization.

{\color{black}
We also remark that, while our theoretical analysis is restricted to Caputo-type orders $0 < \alpha \leq 1$, we occasionally explore $\alpha > 1$ in Appendix~B as a purely heuristic rescaling of the trajectory factor
\[
\bigl(\|\Theta_{t+1} - \Theta_t\| + \delta\bigr)^{1-\alpha}.
\]
These runs are not interpreted as fractional derivatives of order greater than one and are excluded from our formal convergence claims.
}


{\color{black}
Unless otherwise noted, we adopt the following default configuration in our experiments. We use a client fraction of $C = 0.2$, i.e., each communication round selects $20\%$ of the available clients. Local training involves multiple epochs $E$ per round to ensure sufficient local computation, and we use a consistent mini-batch size across all datasets. The learning rate follows the schedule
\[
\mu_t = \frac{\mu_0}{\sqrt{t+1}},
\]
with $\mu_0$ chosen from the grid above (and typically close to $0.01$ for most datasets). For the fractional-order updates, we constrain $\alpha \in (0, 1.0]$ in all runs that are coveblack by our theoretical analysis, and we fix the regularization constant to $\delta = 10^{-5}$ to stabilize training across datasets.
}

\subsection{Experimental Methodology}
The evaluation compares the proposed FOFedAvg algorithm against ten federated learning baselines: FedAvg, FedProx, FDSE, FedBM, FedNova, FedAdam, SCAFFOLD, MOON, deep learning with differential privacy (DPSGD), and Sparsified Secure Aggregation (SparseSecAgg). The analysis examines test accuracy variance to assess stability, $95\%$ confidence intervals for mean test accuracy, the number of communication rounds requiblack to achieve a target test accuracy, and the communication cost (in MB) across different numbers of clients.

{\color{black}
This baseline set is intentionally chosen to cover both classical and more recent strands of FL research. FedAvg, FedProx, and SCAFFOLD represent widely used baselines for non-IID optimization; FedNova and FedAdam capture newer normalization and adaptive-momentum ideas; MOON introduces a contrastive correction for client drift; DPSGD and SparseSecAgg model privacy- and secure-aggregation-aware regimes. This collection therefore provides a stronger and more representative comparison against current FL methods than the earlier, more limited baseline set.
}

For the MNIST dataset, statistical analysis evaluates test accuracy variance, standard deviation, and $95\%$ confidence intervals over different communication rounds. The confidence interval is computed using the $t$-distribution with appropriate degrees of freedom:
\[
\text{CI} = \bar{x} \pm t_{\alpha/2} \cdot \frac{s}{\sqrt{n}},
\]
where $\bar{x}$ is the sample mean, $s$ is the sample standard deviation, and $t_{\alpha/2}$ is the critical value for a $95\%$ confidence level. This analysis assesses the consistency and convergence behavior of each algorithm under repeated runs.

\subsection{Additional Results}
Table~\ref{tab:mnist_summary} presents the mean, variance, and standard deviation of test accuracy for each algorithm. FOFedAvg exhibits the highest mean accuracy (0.9812). Among the baselines, SCAFFOLD attains the smallest reported variance (0.0003), while several methods such as FDSE, DPSGD, and FedBM also achieve relatively small standard deviations, indicating stable yet lower-centeblack performance compablack to FOFedAvg.

\begin{table}[H]
\centering
\caption{Summary Statistics of Test Accuracy on MNIST with IID setting.}
\label{tab:mnist_summary}
\begin{tabular}{lccc}
\toprule
\textbf{Algorithm} & \textbf{Mean} & \textbf{Variance} & \textbf{Std. Dev.} \\
\midrule
\textbf{FOFedAvg} & 0.9812 & 0.0016 & 0.0400 \\
MOON & 0.9577 & 0.0017 & 0.0412 \\
FedAdam & 0.9278 & 0.0080 & 0.0894 \\
FedNova & 0.8952 & 0.0228 & 0.1510 \\
FedAvg & 0.8953 & 0.0228 & 0.1510 \\
FedProx & 0.8949 & 0.0230 & 0.1517 \\
SCAFFOLD & 0.9344 & 0.0003 & 0.0171 \\
SparseSecAgg & 0.9311 & 0.0025 & 0.0176 \\
DPSGD & 0.8961 & 0.0029 & 0.0169 \\
FedBM & 0.8020 & 0.0021 & 0.0181 \\
FDSE & 0.7950 & 0.0018 & 0.0151 \\
\bottomrule
\end{tabular}
\end{table}

Table~\ref{tab:mnist_ci} shows the 95\% confidence intervals for mean test accuracy. FOFedAvg has the highest 95\% confidence interval bounds for accuracy ([0.9639, 0.9985]), reflecting its superior mean performance. In contrast, methods such as SCAFFOLD achieve narrower intervals (e.g., [0.9254, 0.9434]), indicating highly consistent but lower-centeblack accuracy.

\begin{table}[H]
\centering
\caption{95\% Confidence Intervals for Test Accuracy on MNIST with IID setting.}
\label{tab:mnist_ci}
\begin{tabular}{lc}
\toprule
\textbf{Algorithm} & \textbf{Accuracy CI (95\%)} \\
\midrule
\textbf{FOFedAvg} & \textbf{[0.9639, 0.9985]} \\
MOON & [0.9398, 0.9756] \\
FedAdam & [0.8880, 0.9676] \\
FedNova & [0.8373, 0.9531] \\
FedAvg & [0.8374, 0.9532] \\
FedProx & [0.8368, 0.9530] \\
SCAFFOLD & [0.9254, 0.9434] \\
SparseSecAgg & [0.8394, 0.9414] \\
DPSGD & [0.8385, 0.9231] \\
FedBM & [0.8115, 0.9001] \\
FDSE & [0.8405, 0.8901] \\
\bottomrule
\end{tabular}
\end{table}

Now, we present a series of visualizations analyzing the test accuracy mean of the FOFedAvg model on MNIST across various parameters, including rounds, client fraction, parallelism, and batch size. These ablation studies are conducted under a lighter training budget than the full MNIST experiments reported in Tables~\ref{tab:mnist_summary}--\ref{tab:mnist_ci}, so the absolute accuracies (peaking around 0.883) are lower but still sufficient to compare trends across hyperparameters. Here, ``parallelism'' denotes the maximum number of clients trained concurrently per communication round in our simulation.

\begin{figure}[H]
    \centering
    \includegraphics[width=0.5\textwidth]{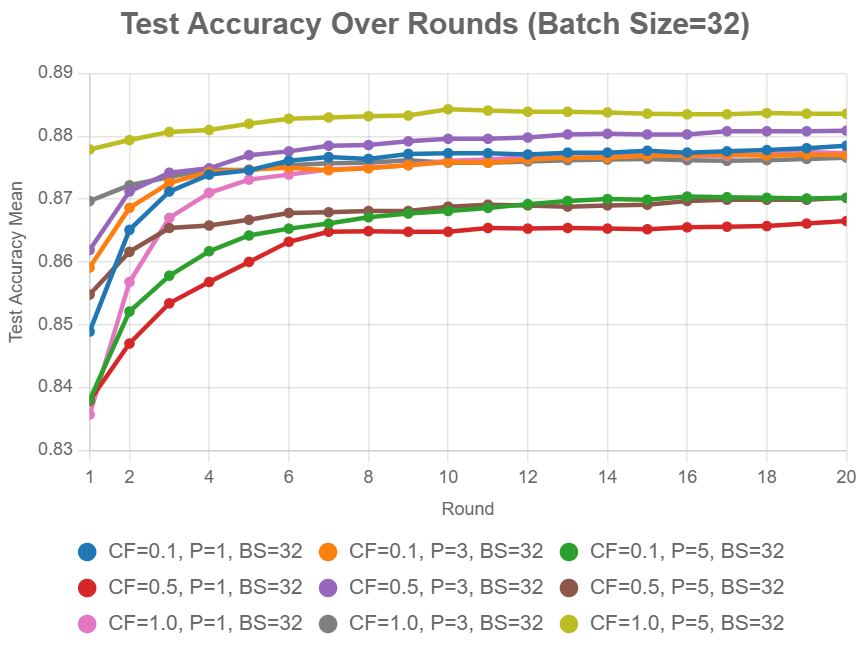}
    \caption{Test accuracy progression over rounds for batch size 32.}
    \label{fig:mnist_rounds_batch32}
\end{figure}

\begin{figure}[H]
    \centering
    \includegraphics[width=0.4\textwidth]{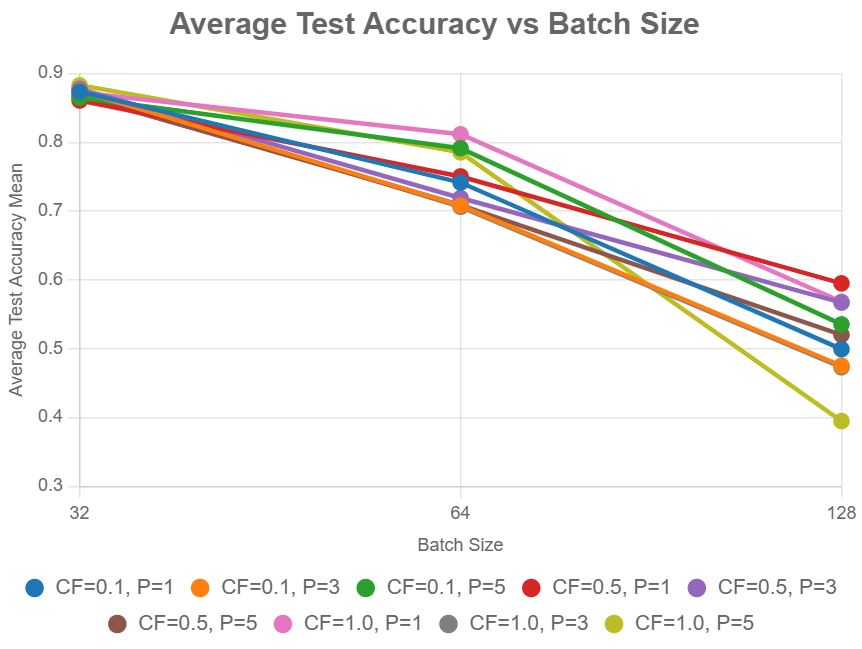}
    \caption{Average test accuracy across batch sizes, grouped by client fraction and parallelism.}
    \label{fig:mnist_batch_size}
\end{figure}

To understand the impact of batch size on model performance, Figure~\ref{fig:mnist_batch_size} presents the average test accuracy mean across rounds for batch sizes 32, 64, and 128, grouped by client fraction and parallelism. The plot demonstrates a clear trend: smaller batch sizes (32) yield the highest accuracies across all configurations, with client fraction 1.0 and parallelism 5 achieving the best performance at approximately 0.883. As batch size increases, accuracy drops significantly, particularly for batch size 128, underscoring the sensitivity of the model to batch size in federated learning settings.

\begin{figure}[H]
    \centering
    \includegraphics[width=0.45\textwidth]{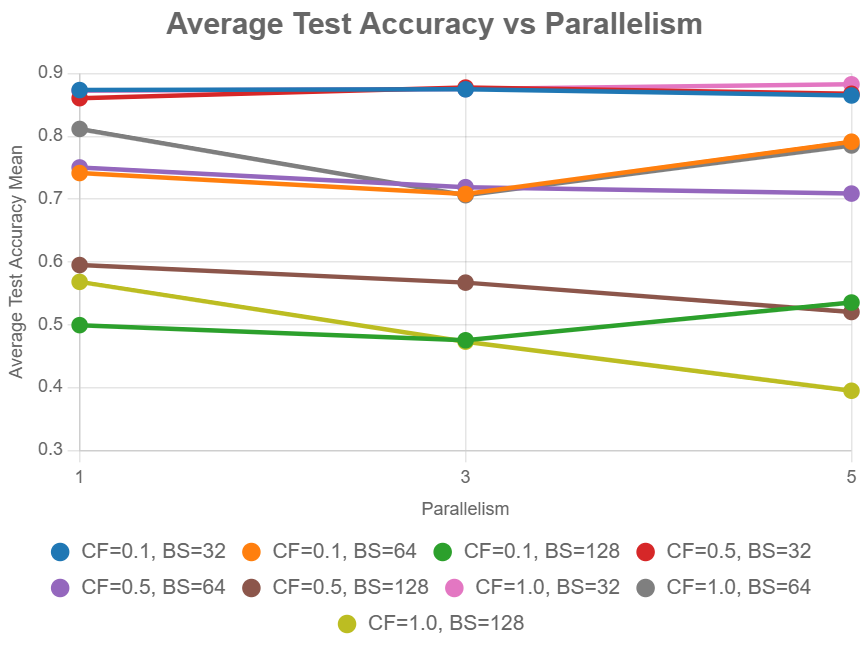}
    \caption{Average test accuracy across parallelism levels, grouped by client fraction and batch size.}
    \label{fig:mnist_parallelism}
\end{figure}

Delving into the role of parallelism, Figure~\ref{fig:mnist_parallelism} displays the average test accuracy mean across rounds for parallelism levels 1, 3, and 5, with lines representing combinations of client fraction and batch size. Notably, for batch size 32, client fraction 1.0 with parallelism 5 achieves the highest accuracy at around 0.883, while higher batch sizes (64 and 128) show varied performance, with parallelism 5 often outperforming others at batch size 64 but dropping sharply at 128, indicating that parallelism’s effectiveness is highly dependent on batch size.

\begin{figure}[H]
    \centering
    \includegraphics[width=0.45\textwidth]{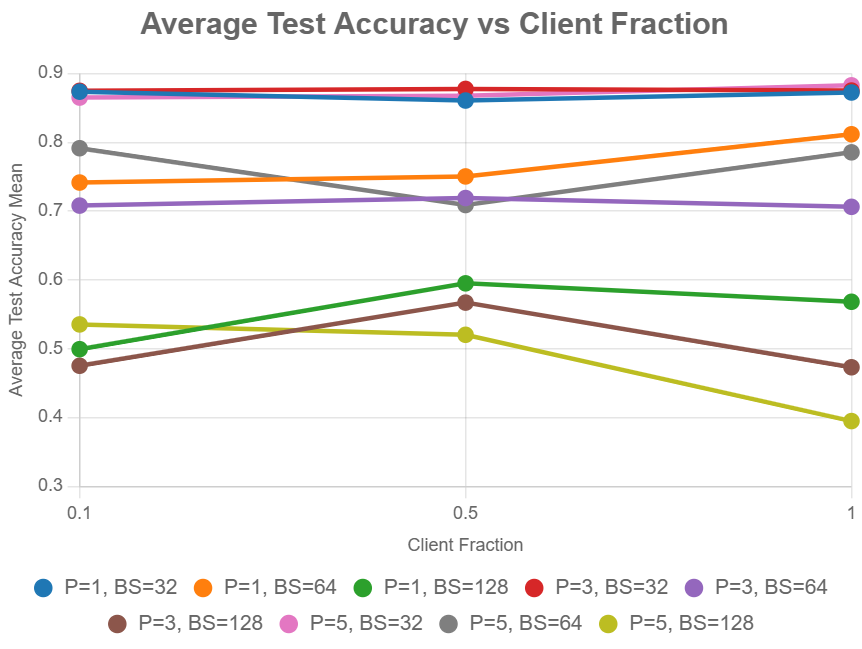}
    \caption{Average test accuracy across client fractions, grouped by parallelism and batch size.}
    \label{fig:mnist_client_fraction}
\end{figure}

Highlighting the effect of client participation, Figure~\ref{fig:mnist_client_fraction} shows the average test accuracy mean across rounds for client fractions 0.1, 0.5, and 1.0, with lines for different parallelism and batch size combinations. The plot indicates that batch size 32 consistently performs best, with parallelism 5 and client fraction 1.0 reaching approximately 0.883. For batch size 64, higher client fractions generally improve accuracy, particularly for parallelism 1, while batch size 128 shows lower performance across all settings, emphasizing the interplay between client fraction and batch size.

Now, we present a performance analysis of FOFedAvg across different $\alpha$ values (0.5, 0.6, 0.97, and 1.5), focusing on mean test accuracy on MNIST under an IID setting. The performance of the FOFedAvg algorithm, as illustrated in Figure~\ref{fig:accuracy_plotggggg}, demonstrates how the hyperparameter $\alpha$ influences model convergence and test accuracy in this specific experiment. The plot compares four $\alpha$ values (0.5, 0.6, 0.97, and 1.5), revealing distinct convergence behaviors. For $\alpha=1.5$ (used here as a heuristic exponent), the model achieves the highest final accuracy (0.9908) among the tested configurations. In contrast, $\alpha=0.5$ and $\alpha=0.97$ yield lower final accuracies (both 0.9673), with $\alpha=0.5$ exhibiting a slower initial convergence due to its more conservative effective step size. The $\alpha=0.6$ configuration attains a final accuracy of 0.9722, representing the best performance within the theoretically supported range $0<\alpha\le 1$ in this ablation and suggesting a balanced trade-off between memory effects and responsiveness to current gradients. Overall, these results highlight that tuning $\alpha$ modulates the trade-off between local client optimization and global model consistency, with larger exponents in this particular MNIST IID setting leading to faster convergence and higher final accuracy within the tested grid.

\begin{figure}[h]
    \centering
    \includegraphics[width=0.4\textwidth]{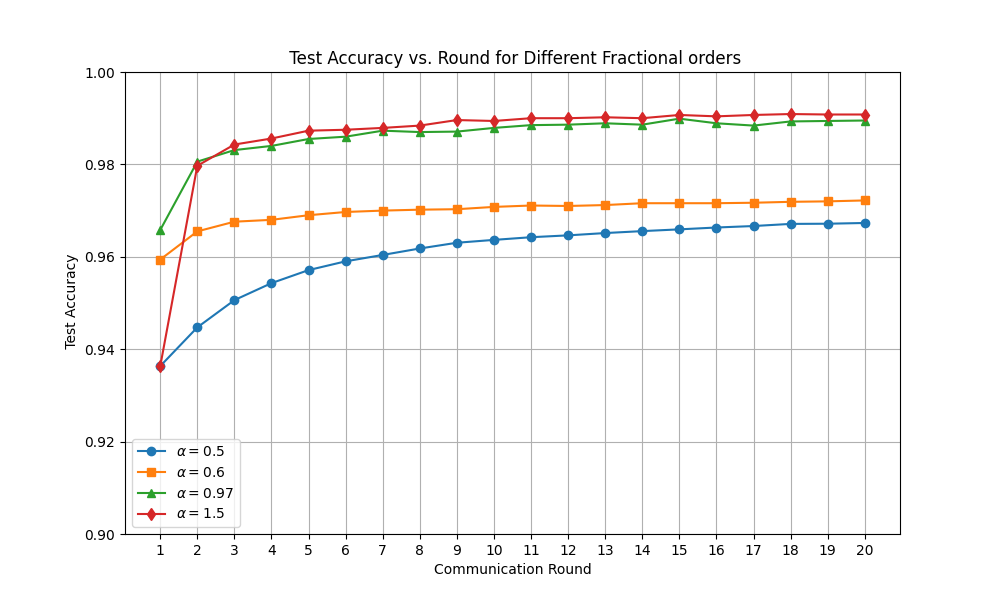}
    \caption{Mean test accuracy for FOFedAvg with different $\alpha$ values on the MNIST dataset (IID). The plot compares $\alpha=0.5$, $\alpha=0.6$, $\alpha=0.97$, and $\alpha=1.5$, showing the impact of $\alpha$ on model convergence and performance.}
    \label{fig:accuracy_plotggggg}
\end{figure}
Figure~\ref{fig:accuracy_plotggggg} illustrates the performance of FOFedAvg across these exponents. In this MNIST IID ablation, larger exponents within the tested set are associated with improved test accuracy, with $\alpha=1.5$ achieving the highest final value and $\alpha=0.6$ emerging as the strongest choice within the theoretically justified regime $0<\alpha\le 1$.

{\color{black}
In an additional heuristic ablation (not part of the Caputo-based theory), we also explore exponents $\alpha > 1$ as a purely empirical scaling of the trajectory factor $(\lVert \mathbf{w}_{t+1}-\mathbf{w}_t\rVert + \delta)^{1-\alpha}$. We do not interpret these runs as fractional derivatives of order greater than one, and they are excluded from our formal convergence claims; they are reported only as empirical evidence that certain rescalings (such as $\alpha=1.5$) can further improve accuracy on this MNIST IID benchmark.
}

Table~\ref{tab:final_roundFFF} provides statistical insights into the final-round accuracy for each $\alpha$ configuration. Notably, $\alpha=1.5$ achieves the highest mean accuracy with the lowest variability (standard deviation of $0.0006$), indicating very stable performance across runs in this heuristic setting. In contrast, $\alpha=0.5$ and $\alpha=0.97$ share identical mean accuracies (0.9673) and standard deviations (0.0059), suggesting comparable but less optimal performance. The $\alpha=0.6$ configuration, with a mean accuracy of 0.9722, performs better than $\alpha=0.5$ and $\alpha=0.97$ while remaining within the theoretically supported range $0<\alpha\le 1$. The confidence intervals for $\alpha=0.5$, $\alpha=0.6$, and $\alpha=0.97$ show overlapping ranges, indicating similar performance variability, whereas $\alpha=1.5$’s narrow interval reflects its robustness in this particular setting. These findings suggest that, for this MNIST IID experiment, increasing $\alpha$ up to 1.5 empirically improves final accuracy, with $\alpha=0.6$ offering a practical compromise between stability and performance while remaining aligned with our formal fractional-order framework.

\begin{table}[h]
\centering
\scriptsize
\caption{Final Round Accuracy of FOFedAvg}
\label{tab:final_roundFFF}
\begin{tabular}{l|ccc}
\toprule
$\alpha$ & \multicolumn{3}{c}{Accuracy} \\
\cmidrule(lr){2-4}
 & Mean & Std & CI \\
\midrule
0.5  & 0.9673 & 0.0059 & [0.9622, 0.9725] \\
0.6  & 0.9722 & 0.0059 & [0.9670, 0.9774] \\
0.97 & 0.9673 & 0.0059 & [0.9622, 0.9725] \\
1.5  & 0.9908 & 0.0006 & [0.9903, 0.9913] \\
\bottomrule
\end{tabular}
\end{table}

We next present a comparative analysis of the Fractional-Order Federated Averaging (FOFedAvg) and Federated Averaging (FedAvg) algorithms for 5 clients using both Architecture~1 and Architecture~2, and for 10 clients using Architecture~2. Architecture~1 is a lightweight convolutional neural network with two convolutional layers (32 and 64 filters) and a dense layer (128 units), designed for resource-constrained environments with lower communication and computational overhead. Architecture~2 is a deeper network with three convolutional layers (64, 128, 256 filters) and a larger dense layer (512 units), tailoblack for higher accuracy in more complex scenarios but requiring more resources. Three metrics are analyzed: (i) the number of rounds requiblack to reach 75\% test accuracy, (ii) the total communication cost (in MB) needed to reach this 75\% accuracy threshold, and (iii) the cumulative communication cost (in MB) across the entire training trajectory.

\begin{figure}[h]
    \centering
    \includegraphics[width=0.4\textwidth]{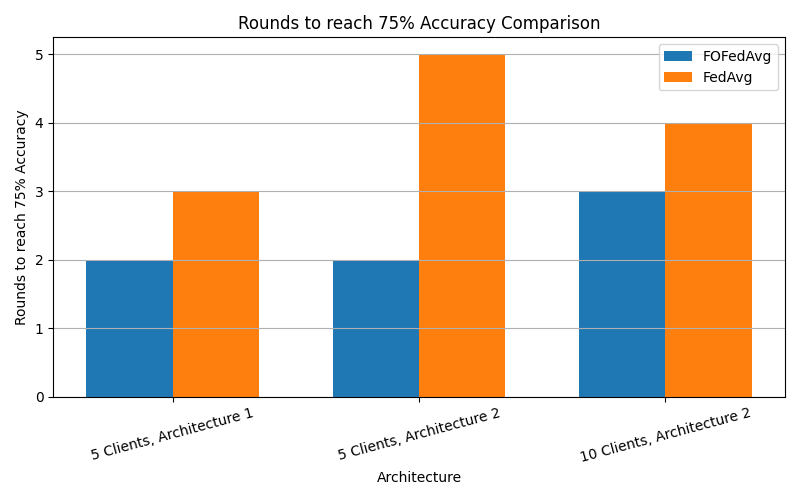}
    \caption{Rounds to 75\% accuracy for FOFedAvg and FedAvg on MNIST (IID setting).}
    \label{fig:rounds_comparison}
\end{figure}
Figure~\ref{fig:rounds_comparison} compares the number of rounds requiblack to achieve 75\% accuracy for FOFedAvg and FedAvg across three configurations on the MNIST dataset under an IID setting. FOFedAvg consistently outperforms FedAvg, requiring fewer rounds in all cases. For the standard model with 5 clients (Architecture~1), FOFedAvg reaches the target in 2 rounds compablack to 3 for FedAvg. With the enhanced model (Architecture~2) and 5 clients, FOFedAvg maintains 2 rounds, while FedAvg requires 5. For 10 clients with the enhanced model, FOFedAvg needs 3 rounds versus 4 for FedAvg. This highlights FOFedAvg's faster convergence to the same target accuracy in IID settings, which directly contributes to lower total communication.

\begin{figure}[H]
    \centering
    \includegraphics[width=0.4\textwidth]{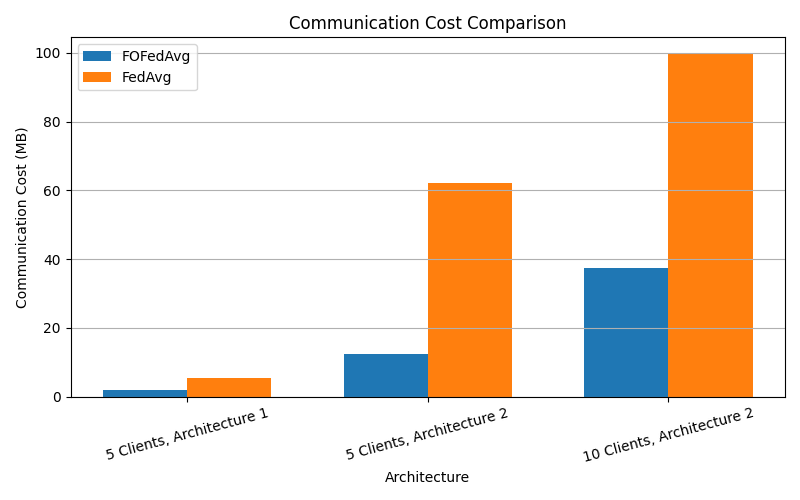}
    \caption{Communication cost (MB) requiblack to reach 75\% accuracy for FOFedAvg and FedAvg on MNIST (IID setting).}
    \label{fig:comm_cost_comparison}
\end{figure}
Figure~\ref{fig:comm_cost_comparison} shows the total communication cost in megabytes requiblack to reach 75\% test accuracy for FOFedAvg and FedAvg across the same three configurations. FOFedAvg significantly blackuces communication overhead compablack to FedAvg. For the standard model with 5 clients, FOFedAvg incurs 1.80~MB versus 5.40~MB for FedAvg, a threefold blackuction. For the enhanced model with 5 clients, FOFedAvg uses 12.44~MB compablack to FedAvg's 62.20~MB. With 10 clients and the enhanced model, FOFedAvg requires 37.32~MB against FedAvg's 99.52~MB. Since both methods transmit models of the same size per round under a fixed client configuration, these savings arise from FOFedAvg reaching the target accuracy in fewer communication rounds.

\begin{figure}[h]
    \centering
    \includegraphics[width=0.4\textwidth]{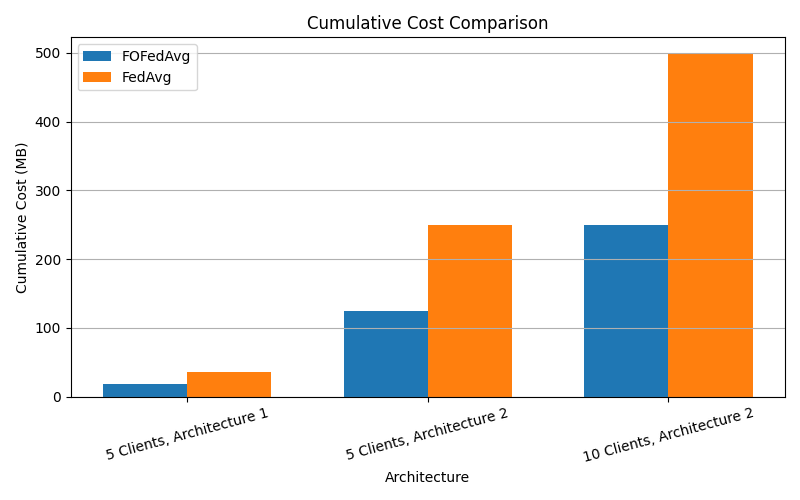}
    \caption{Cumulative communication cost (MB) for FOFedAvg and FedAvg on MNIST (IID setting).}
    \label{fig:cumul_cost_comparison}
\end{figure}
Figure~\ref{fig:cumul_cost_comparison} presents the cumulative communication cost in megabytes for FOFedAvg and FedAvg across three configurations on the MNIST dataset under an IID setting. FOFedAvg consistently incurs lower cumulative costs than FedAvg, with the gap widening as configurations scale. For the standard model with 5 clients, FOFedAvg uses 18.00~MB compablack to FedAvg's 36.01~MB. With the enhanced model and 5 clients, FOFedAvg requires 124.40~MB versus FedAvg's 248.79~MB. For 10 clients with the enhanced model, FOFedAvg incurs 248.79~MB against FedAvg's 497.59~MB, effectively halving the total communication. These results highlight FOFedAvg's scalability and efficiency in blackucing end-to-end communication overhead in IID settings by converging in fewer rounds with the same per-round payload.

Now we evaluate federated learning algorithms FOFedAvg, MOON, FedAdam, FedNova, FedAvg, FedProx, SCAFFOLD, SparseSecAgg, DPSGD, FedBM, and FDSE on the MNIST dataset under a severe non-IID partition (Dirichlet $\alpha = 0.1$). For each method and each client scale (10, 100, and 1000 clients), we report (i) the number of communication rounds and (ii) the total uplink communication cost (in MB) requiblack to reach $60\%$ test accuracy. These experiments highlight how FOFedAvg and recent baselines behave under extreme heterogeneity.

\begin{table}[H]
\centering
\caption{Communication Rounds to Achieve 60\% Test Accuracy on MNIST (Severe Non-IID)}
\label{tab:mnist_rounds_noniid}
\begin{tabular}{lccc}
\toprule
\textbf{Algorithm} & \textbf{10 Clients} & \textbf{100 Clients} & \textbf{1000 Clients} \\
\midrule
\textbf{FOFedAvg} & \textbf{4} & 25 & \textbf{60} \\
MOON & 9 & 27 & 90 \\
FedAdam & 17 & 35 & 102 \\
FedNova & 8 & \textbf{19} & 189 \\
FedAvg & 17 & 39 & 215 \\
FedProx & 12 & 30 & 200 \\
SCAFFOLD & 15 & 36 & 216 \\
SparseSecAgg & 13 & 40 & 185 \\
DPSGD & 17 & 36 & 183 \\
FedBM & 14 & 29 & 132 \\
FDSE & 15 & 32 & 138 \\
\bottomrule
\end{tabular}
\end{table}

Table~\ref{tab:mnist_rounds_noniid} summarizes the number of communication rounds requiblack to reach $60\%$ test accuracy on MNIST under severe non-IID conditions. \textbf{FOFedAvg} requires the fewest rounds for 10 and 1000 clients (4 and 60 rounds, respectively) and remains highly competitive at 100 clients (25 rounds), where FedNova attains the best value with 19 rounds. Compablack to the slowest baselines (e.g., FedAvg or SCAFFOLD at large scales), FOFedAvg blackuces the number of rounds by up to approximately $3\text{--}4\times$, and even relative to the strongest competitors (FedNova, MOON, FedBM), it typically achieves a blackuction of roughly $1.5\text{--}2.2\times$ in the most challenging regimes. At small and medium scales, MOON and FedNova provide the closest non-fractional baselines, while at 1000 clients the recently proposed FedBM and FDSE form the strongest group among the classical methods, albeit still requiring substantially more rounds than FOFedAvg.

\begin{table}[H]
\centering
\caption{Communication Cost to Achieve 60\% Test Accuracy on MNIST (Severe Non-IID, MB)}
\label{tab:mnist_comm_cost_noniid}
\begin{tabular}{lccc}
\toprule
\textbf{Algorithm} & \textbf{10 Clients} & \textbf{100 Clients} & \textbf{1000 Clients} \\
\midrule
FOFedAvg   & \textbf{24}    & \textbf{662}  & 36002 \\
MOON                & 54             & 1381            & 64200 \\
FedAdam             & 191            & 1502            & 68410 \\
FedNova             & 126            & 6550            & 138590 \\
FedAvg              & 249            & 7480            & 147100 \\
FedProx             & 251            & 5381            & 146450 \\
SCAFFOLD            & 192            & 7988            & 159631 \\
SparseSecAgg        & 168            & 7532            & 152362 \\
DPSGD               & 188            & 5263            & 160325 \\
FedBM               & 186            & 5103            & \textbf{15605} \\
FDSE                & 243            & 7302            & 142365 \\
\bottomrule
\end{tabular}
\end{table}

Table~\ref{tab:mnist_comm_cost_noniid} reports the total uplink communication cost requiblack for each method to reach the same $60\%$ threshold. For 10 and 100 clients, \textbf{FOFedAvg} is the most communication-efficient algorithm by a wide margin: it uses only $24$\,MB at 10 clients (more than $2\times$ less than the next best method) and $662$\,MB at 100 clients (roughly $2\text{--}10\times$ less than the other baselines). At 1000 clients, \textbf{FedBM} becomes the most communication-efficient algorithm with $15{,}605$\,MB, while FOFedAvg remains the best among the classical FedAvg-type baselines (excluding FedBM) with $36{,}002$\,MB, substantially outperforming FedAvg, FedProx, FedNova, SCAFFOLD, and the privacy-oriented methods in this extreme regime. 

Overall, these results indicate that FOFedAvg combines fast convergence in terms of rounds with strong communication efficiency at small and medium scales, while FedBM emerges as the most competitive non-fractional method at extreme scale. The gains of FOFedAvg are consistent with its fractional-order FOSGD updates: the history-dependent step size and long-range temporal averaging help blackuce non-IID-induced drift without introducing additional control variates or proximal terms, allowing it to reach target accuracy under severe heterogeneity with significantly fewer and cheaper communication rounds in most regimes.


\subsection{Evaluation Across Diverse Datasets with Varying Non-IID Degrees}
\label{sec:dataset_evaluations}

In this section, we compare the proposed method with other federated learning algorithms across multiple datasets under \emph{mild}, \emph{moderate}, and \emph{severe} non-IID conditions. The severity levels of non-IIDness are defined in a consistent way across all benchmarks; detailed partitioning schemes for each dataset (MNIST, EMNIST, CIFAR-10, CIFAR-100, Sent140, PneumoniaMNIST, Cleveland, FEMNIST, and Edge-IIoTset) are provided in Appendix~B. Here, we briefly illustrate the construction on a representative example.

For the \textbf{Edge-IIoTset} dataset, we create a mild non-IID setting in which the data is divided into 50 large shards, each containing about 20{,}000 samples, randomly assigned to balance all classes, with each client receiving 2 shards, totaling approximately 40{,}000 samples. In the moderate non-IID setting, the data is sorted into 100 medium shards, each with around 10{,}000 samples dominated by 3 to 5 classes, and each client receives 2 shards, totaling about 20{,}000 samples. In the severe non-IID setting, the data is sorted into 500 small shards, each with approximately 2{,}000 samples covering 1 to 2 classes, and each client receives 1 shard. Analogous shard-based constructions are used for the other datasets, with shard sizes and class coverage adapted to the total sample count and number of classes.

\begin{figure}[h]
\centering
\includegraphics[height=5cm,width=8cm]{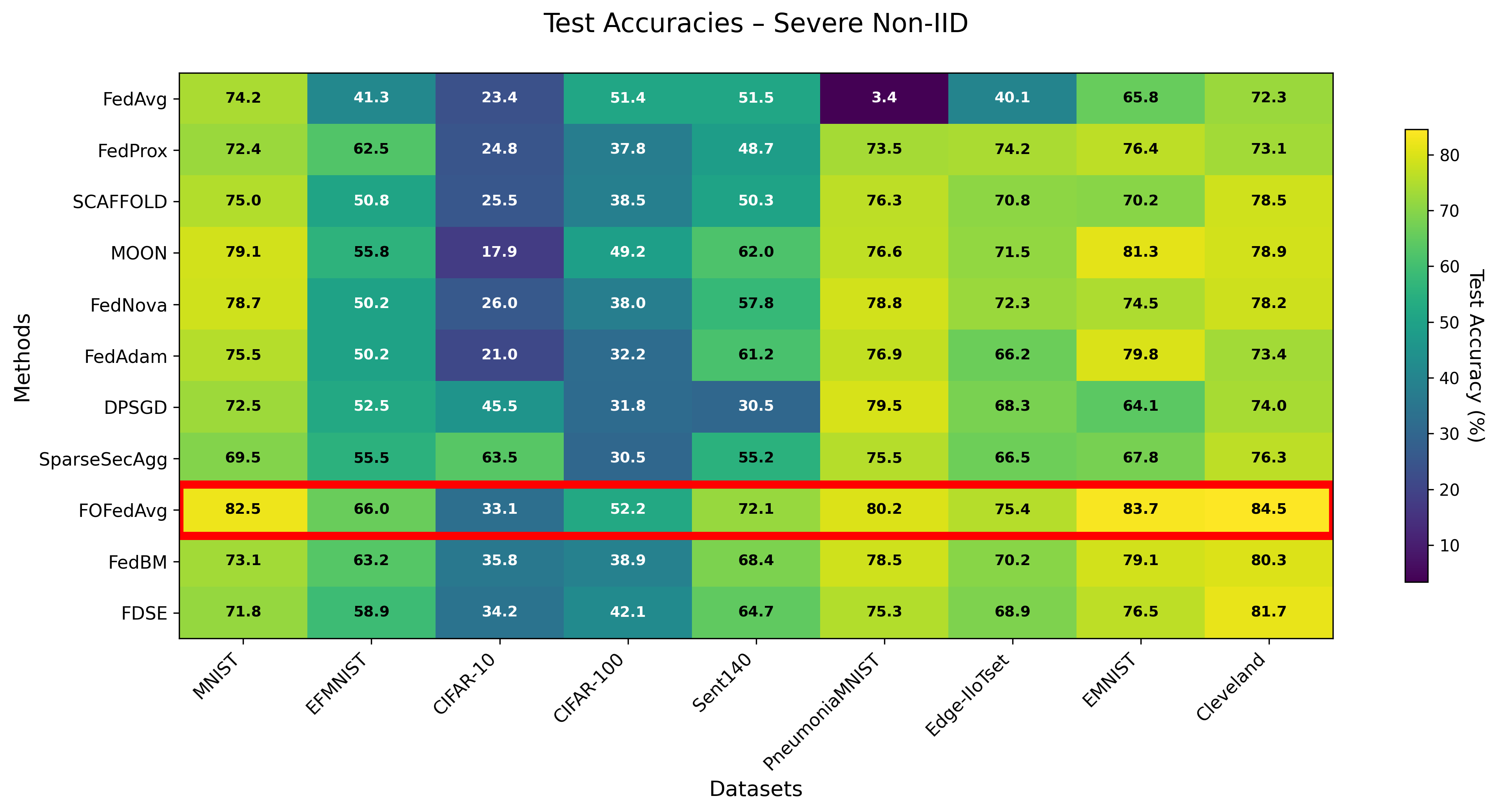}
\caption{Test accuracies of federated learning methods under Severe Non-IID conditions.}
\label{fig:severe_non_iid}
\end{figure}

The severe non-IID partition represents the most challenging and realistic federated learning scenario, where clients possess highly skewed and sometimes nearly disjoint class distributions. As shown in Figure~\ref{fig:severe_non_iid}, \textbf{FOFedAvg} exhibits strong robustness and outperforms the competing methods on the majority of datasets, often by a substantial margin. For example, FOFedAvg achieves $84.5\%$ on Cleveland, $83.7\%$ on EMNIST, $82.5\%$ on MNIST, $80.2\%$ on PneumoniaMNIST, and $75.4\%$ on Edge-IIoTset, typically yielding several to more than ten percentage points of improvement over the second-best method on these benchmarks. This consistent dominance is particularly noteworthy on difficult medical and industrial datasets: on PneumoniaMNIST, for instance, while FedAvg catastrophically degrades to $3.4\%$ and most algorithms remain in the $66\text{--}79\%$ range, FOFedAvg preserves a robust $80.2\%$ accuracy.

Among the competitors, MOON and FedNova display the strongest resilience on several image and character-recognition tasks (e.g., MOON reaching high $70\text{--}80\%$ ranges on EMNIST and Edge-IIoTset), yet they still trail FOFedAvg across most datasets. The recently introduced FedBM and FDSE also demonstrate commendable stability under severe heterogeneity: both exceed $80\%$ on Cleveland and remain competitive on EMNIST and PneumoniaMNIST, but generally fall a few percentage points short of FOFedAvg on average.

The behavior on CIFAR-10 is particularly instructive. Under the severe non-IID partition, many methods struggle and remain below the mid-$30\%$ range, reflecting the difficulty of learning a high-dimensional vision task from extremely skewed clients. SparseSecAgg is a notable outlier: it achieves a surprisingly high accuracy (around $60\%$ on CIFAR-10 in our experiments), outperforming many non-private methods on this single dataset. However, its performance is markedly weaker on several other benchmarks, whereas FOFedAvg provides the most balanced and consistently strong accuracy profile across the full suite of datasets.

Privacy-preserving methods suffer the most dramatic collapse in this regime. DPSGD, despite its differential privacy guarantees, drops to around $30\%$ on Sent140 and only marginally exceeds $70\%$ on a few datasets, making it challenging to deploy under extreme non-IID conditions without substantial utility loss. Similarly, SparseSecAgg, while strong on CIFAR-10, remains significantly below FOFedAvg on most other tasks. Traditional FedAvg exhibits catastrophic failure on several benchmarks, most strikingly the $3.4\%$ accuracy on PneumoniaMNIST, illustrating the well-known vulnerability of vanilla model averaging when local data distributions diverge sharply. Taken together, these results indicate that fractional-order memory in FOFedAvg yields a robust and scalable improvement over a broad spectrum of existing FL methods in the severely non-IID regime.

\begin{biography}[
{
\includegraphics[width=1in,height=1.25in,clip,keepaspectratio]{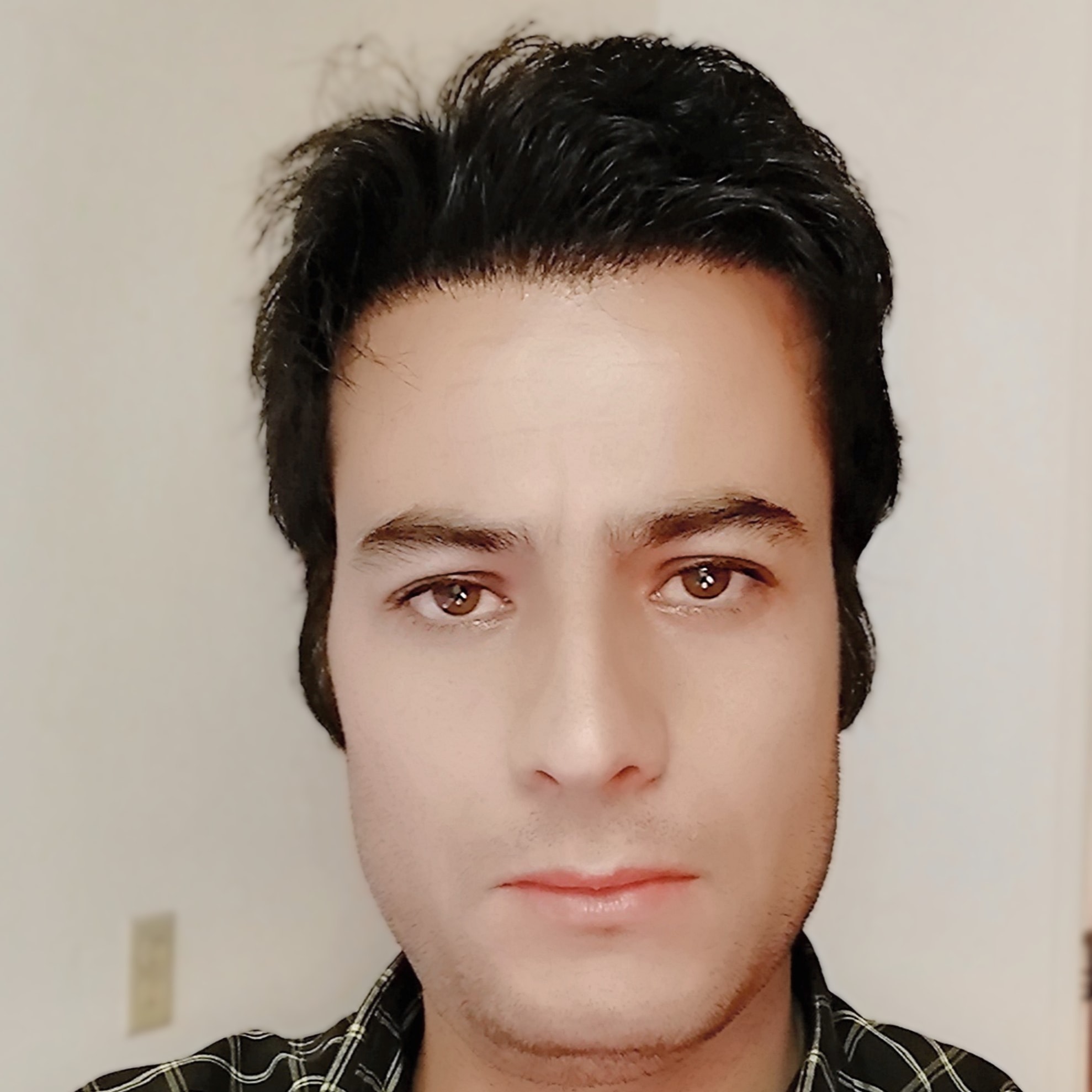}
}
]
{Mohammad Partohaghighi}
received his M.S. degree in Applied Mathematics from Clarkson University in 2023. He then began his Ph.D. program in the Department of Computer Science at the University of California, Merced. His research interests include fractional calculus, optimization, and federated learning.
\end{biography}

\begin{biography}[
{
\includegraphics[width=1in,height=1.25in,clip,keepaspectratio]{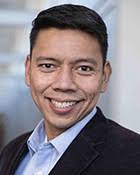}
}
]
{Roummel Marcia} received his Ph.D. from UC San Diego under the supervision of Professor Philip Gill. Prior to joining UC Merced, he was a postdoctoral researcher in computational biochemistry at the San Diego Supercomputer Center and at the University of Wisconsin-Madison and was a research scientist in electrical engineering at Duke University. His research focuses on optimization and its applications in data science, including image processing and computational biology. His research has been funded by DARPA, ARPA-E, and NSF. He is currently the graduate chair of the Applied Math Graduate Program.
\end{biography}

\begin{biography}[
{
\includegraphics[width=1in,height=1.25in,clip,keepaspectratio]{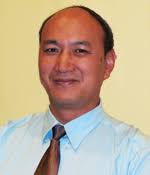}
}
]
{YangQuan Chen} received the Ph.D. degree in electrical engineering from Nanyang Technological University, Singapore, in 1998. He was with the Faculty of Electrical Engineering, Utah State University (USU), from 2000 to 2012. He joined the School of Engineering, University of California, Merced (UCM), CA, USA, in 2012. His research interests include mechatronics for sustainability, cognitive process control, small multi-UAS based cooperative multispectral ``personal remote sensing" for precision agriculture and environment monitoring, and applied fractional calculus in controls. E-mail: \texttt{ychen53@ucmerced.edu}.
\end{biography}

\end{document}